\documentclass{article}

\PassOptionsToPackage{numbers, compress}{natbib}

     \usepackage[final]{neurips_2020}

\usepackage[utf8]{inputenc} %
\usepackage[T1]{fontenc}    %
\usepackage{hyperref}       %
\usepackage{url}            %
\usepackage{booktabs}       %
\usepackage{amsfonts}       %
\usepackage{nicefrac}       %
\usepackage{microtype}      %
\usepackage{comment}

\usepackage{amsmath}%
\usepackage{amsthm}
\usepackage{amssymb}
\usepackage{bbm}
\usepackage{paralist}
\usepackage{enumitem}
\usepackage{xcolor}
\usepackage{mathtools}
\usepackage{dsfont}
\usepackage{algorithm}
\usepackage{algorithmic}

\renewcommand{\star}{*}
\newcommand{\KL}{\textnormal{KL}}
\newcommand{\kl}{\textnormal{kl}}

\newcommand{\indicator}{\mathds{1}}
\newtheorem{theorem}{Theorem}

\newtheorem{lemma}{Lemma}

\newcommand{\EXP}{\mathbb{E}}
\renewcommand{\Pr}{\mathbb{P}}
\newcommand{\set}[1]{\mathcal{#1}}
\newcommand{\ep}{\hfill $\Box$}

\title{Regret in Online Recommendation Systems}

\author{
	Kaito Ariu\\
	KTH\\
	Stockholm, Sweden\\
	 \texttt{ariu@kth.se}\\
	 \And
	 Narae Ryu\\
	 KAIST\\
	 Daejeon, South Korea\\
	 \texttt{nrryu@kaist.ac.kr}\\
	 \AND
	 Se-Young Yun\\
KAIST\\
Daejeon, South Korea\\
\texttt{yunseyoung@kaist.ac.kr}\\	
\And
	Alexandre Proutière \\
KTH\\
Stockholm, Sweden\\
\texttt{alepro@kth.se}\\
}

\begin{document}

\maketitle

\begin{abstract}
This paper proposes a theoretical analysis of recommendation systems in an online setting, where items are sequentially recommended to users over time. In each round, a user, randomly picked from a population of $m$ users, requests a recommendation. The decision-maker observes the user and selects an item from a catalogue of $n$ items. Importantly, an item cannot be recommended twice to the same user. The probabilities that a user likes each item are unknown. The performance of the recommendation algorithm is captured through its regret, considering as a reference an Oracle algorithm aware of these probabilities. We investigate various structural assumptions on these probabilities: we derive for each structure regret lower bounds, and devise algorithms achieving these limits. Interestingly, our analysis reveals the relative weights of the different components of regret: the component due to the constraint of not presenting the same item twice to the same user, that due to learning the chances users like items, and finally that arising when learning the underlying structure. 

\end{abstract}

\tableofcontents
\newpage

\section{Introduction}\label{sec:intro}

Recommendation systems \cite{resnick97} have over the last two decades triggered important research efforts (see, e.g., \cite{gentile2014online,gentile2017context,li2016collaborative,bresler2018regret} for recent works and references therein), mainly focused towards the design and analysis of algorithms with improved efficiency. These algorithms are, to some extent, all based on the principle of {\it collaborative filtering}: similar items should yield similar user responses, and similar users have similar probabilities of liking or disliking a given item. In turn, efficient recommendation algorithms need to learn and exploit the underlying structure tying the responses of the users to the various items together. 

Most recommendation systems operate in an online setting, where items are sequentially recommended to users over time. We investigate recommendation algorithms in this setting. More precisely, we consider a system of $n$ items and $m$ users, where $m\ge n$ (as this is typically the case in practice). In each round, the algorithm needs to recommend an item to a {\it known} user, picked randomly among the $m$ users. The response of the user is noisy: the user likes the recommended item with an a priori {\it unknown} probability depending on the (item, user) pair. In practice, it does not make sense to recommend an item twice to the same user (why should we recommend an item to a user who already considered or even bought the item?). We restrict our attention to algorithms that do not recommend an item twice to the same user, a constraint referred to as the {\it no-repetition} constraint. The objective is to devise algorithms maximizing the expected number of successful recommendations over a time horizon of $T$ rounds. 

We investigate different system structures. Specifically, we first consider the case of clustered items and statistically identical users -- the probability that a user likes an item depends on the item cluster only. We then study the case of unclustered items and statistically identical users -- the probability that a user likes an item depends on the item only. The third investigated structure exhibits clustered items and clustered users -- the probability that a user likes an item depends on the item and user clusters only. In all cases, the structure (e.g., the clusters) is initially unknown and has to be learnt to some extent. This paper aims at answering the question: How can the structure be optimally learnt and exploited?

To this aim, we study the regret of online recommendation algorithms, defined as the difference between their expected number of successful recommendations to that obtained under an Oracle algorithm aware of the structure and of the success rates of each (item, user) pair. We are interested in regimes where $n, m,$ and $T$ grow large simultaneously, and $T=o(mn)$ (see \textsection \ref{sec:models} for details). For the aforementioned structures, we first derive non-asymptotic and problem-specific regret lower bounds satisfied by any algorithm.\\
(i) For clustered items and statistically identical users, as $T$ (and hence $m$) grows large, the minimal regret scales as $K \max\{ {\log(m)\over \log({m\log(m)\over T})}, {T\over \Delta m}\}$, where $K$ is the number of item clusters, and $\Delta$ denotes the minimum difference between the success rates of items from the optimal and sub-optimal clusters.\\
(ii) For unclustered items and statistically identical users, the minimal {\it satisficing} regret\footnote{For this unstructured scenario, we will justify why considering the satisficing regret is needed.} scales as $\max\{ {\log(m)\over \log({m\log(m)\over T})}, {T\over m\varepsilon}\}$, where $\varepsilon$ denotes the threshold defining the satisficing regret (recommending an item in the top $\varepsilon$ percents of the items is assumed to generate no regret). \\
(iii) For clustered items and users, the minimal regret scales as ${m\over \Delta}$ or as ${m\over \Delta}\log(T/m)$, depending on the values of the success rate probabilities.\\
We also devise algorithms that provably achieve these limits (up to logarithmic factors), and whose regret exhibits the right scaling in $\Delta$ or $\varepsilon$. We illustrate the performance of our algorithms through experiments presented in the appendix.    

Our analysis reveals the relative weights of the different components of regret. For example, we can explicitly identify the regret induced by the no-repetition constraint (this constraint imposes us to select unrecommended items and induces an important learning price). We may also characterize the regret generated by the fact that the item or user clusters are initially unknown. Specifically, fully exploiting the item clusters induces a regret scaling as $K{T\over \Delta m}$. Whereas exploiting user clusters has a much higher regret cost scaling as least as ${m\over \Delta}$. 

In our setting, deriving regret lower bounds and devising optimal algorithms cannot be tackled using existing techniques from the abundant bandit literature. This is mainly due to the no-repetition constraint, to the hidden structure, and to the specificities introduced by the random arrivals of users. Getting tight lower bounds is particularly challenging because of the non-asymptotic nature of the problem (items cannot be recommended infinitely often, and new items have to be assessed continuously). To derive these bounds, we introduce novel techniques that could be useful in other online optimization problems. The design and analysis of efficient algorithms also present many challenges. Indeed, such algorithms must include both clustering and bandit techniques, that should be jointly tuned. 

Due to space constraints, we present the pseudo-codes of our algorithms, all proofs, numerical experiments, as well as some insightful discussions in the appendix.

\section{Related Work}

The design of recommendation systems has been framed into structured bandit problems in the past. Most of the work there consider a linear reward structure (in the spirit of the matrix factorization approach), see e.g. \cite{gentile2014online}, \cite{gentile2017context}, \cite{li2019improved}, \cite{li2018online}, \cite{li2016collaborative}, \cite{gopalan2016low}. These papers ignore the no-repetition constraint (a usual assumption there is that when a user arrives, a set of fresh items can be recommended). In \cite{mary2015bandits}, the authors try to include this constraint but do not present any analytical result.  Furthermore, notice that the structures we impose in our models are different than that considered in the low-rank matrix factorization approach.

Our work also relates to the literature on clustered bandits. Again the no-repetition constraint is not modeled. In addition, most often, only the user clusters \cite{bui2012clustered}, \cite{maillard2014latent} or only the item clusters are considered \cite{kwon2017sparse}, \cite{jedor2019categorized}. Low-rank bandits extend clustered bandits by modeling the (item, user) success rates as a low-rank matrix, see \cite{jun2019bilinear}, \cite{mueller2019low}, still without accounting for the no-repetition constraint, and without a complete analysis (no precise regret lower bounds).  
 
One may think of other types of bandits to model recommendation systems. However, none of them captures the essential features of our problem. For example, if we think of contextual bandits (see, e.g., \cite{hao2019adaptive} and references therein), where the context would be the user, it is hard to model the fact that when the same context appears several times, the set of available arms (here items) changes depending on the previous arms selected for this context. Budgeted and sleeping bandits \cite{combes2015bandits}, \cite{kleinberg2008regret} model scenarios where the set of available arms changes over time, but in our problem, this set changes in a very specific way not covered by these papers. In addition, studies on budgeted and sleeping bandits do not account for any structure. 

The closest related work can be found in \cite{bresler2014latent} and \cite{heckel2017sample}. There, the authors explicitly model the no-repetition constraint but consider user clusters only, and do not provide regret lower bounds. \cite{bresler2018regret} extends the analysis to account for item clusters as well but studies a model where users in the same cluster deterministically give the same answers to items in the same cluster.

\section{Models and Preliminaries}\label{sec:models}

We consider a system consisting of a set ${\cal I}=[n]:=\{1,\ldots,n\}$ of items and a set ${\cal U}=[m]$ of users. In each round, a user, chosen uniformly at random from ${\cal U}$, needs a recommendation. The decision-maker observes the user id and selects an item to be presented to the user. Importantly an item cannot be recommended twice to a user. The user immediately rates the recommended item +1 if she likes it or 0 otherwise. This rating is observed by the decision-maker, which helps subsequent item selections. 

Formally, in round $t$, the user $u_t\sim \textrm{unif}({\cal U})$ requires a recommendation. If item $i$ is recommended, the user $u_t=u$ likes the item with probability $\rho_{iu}$. We introduce the binary r.v. $X_{iu}$ to indicate whether the user likes the item, $X_{iu}\sim \textrm{Ber}(\rho_{iu})$. Let $\pi$ denote a sequential item selection strategy or algorithm. Under $\pi$, the item $i^\pi_t$ is presented to the $t$-th user. The choice $i^\pi_t$ depends on the past observations and on the identity of the $t$-th user, namely, $i_t^\pi$ is ${\cal F}_{t-1}^\pi$-measurable, with ${\cal F}_{t-1}^\pi = \sigma(u_t, (u_s,i_s^\pi, X_{i_s^\pi u_s}), s\le t-1)$ ($\sigma(Z)$ denotes the $\sigma$-algebra generated by the r.v. $Z$). Denote by $\Pi$ the set of such possible algorithms. The reward of an algorithm $\pi$ is defined as the expected number of positive ratings received over $T$ rounds: $\mathbb{E}[\sum_{t=1}^T \rho_{i_t^\pi u_t}]$. We aim at devising an algorithm with maximum reward. 

We are mostly interested in scenarios where $(m,n,T)$ grow large under the constraints (i) $m\ge n$ (this is typically the case in recommendation systems), (ii) $T=o(mn)$, and (iii) $\log(m)=o(n)$. Condition (ii) complies with the no-repetition constraint and allows some freedom in the item selection process. (iii) is w.l.o.g. as explained in \cite{bresler2014latent}, and is just imposed to simplify our definitions of regret (refer to Appendix \ref{app:regretdef} for a detailed discussion).

\subsection{Problem structures and regrets}\label{subsec:models_regrets}

We investigate three types of systems depending on the structural assumptions made on the success rates $\rho=(\rho_{iu})_{i\in {\cal I}, u\in {\cal U}}$.

{\bf Model A. Clustered items and statistically identical users.} In this case, $\rho_{iu}$ depends on the item $i$ only. Items are classified into $K$ clusters ${\cal I}_1,\ldots {\cal I}_K$. When the algorithm recommends an item $i$ for the first time, $i$ is assigned to cluster ${\cal I}_k$ with probability $\alpha_k$, independently of the cluster assignments of the other items. When $i\in {\cal I}_k$, then $\rho_i=p_k$. We assume that both $\alpha=(\alpha_k)_{k\in [K]}$ and $p=(p_k)_{k\in [K]}$ do not depend on $(n,m,T)$, but are initially unknown. W.l.o.g. assume that $p_1> p_2\ge p_3\ge \ldots\ge p_K$. To define the regret of an algorithm $\pi\in \Pi$, we compare its reward to that of an Oracle algorithm aware of the item clusters and of the parameters $p$. The latter would {\it mostly} recommend items from cluster ${\cal I}_1$. Due to the randomness in the user arrivals and the cluster sizes, recommending items not in ${\cal I}_1$ may be necessary. However, we define regret as if recommending items from ${\cal I}_1$ was always possible. Using our assumptions $T=o(mn)$ and $\log(m)=o(n)$, we can show that the difference between our regret and the true regret (accounting for the possible need to recommend items outside ${\cal I}_1$) is always negligible. Refer to Appendix \ref{app:regretdef} for a formal justification. In summary, the regret of $\pi\in \Pi$ is defined as:  
$
R^\pi(T) = Tp_1 -\sum_{t=1}^T \mathbb{E}\left[\sum_{k=1}^K\indicator_{\{ i^\pi_t\in {\cal I}_k\} }p_k\right].
$
 
{\bf Model B. Unclustered items and statistically identical users.} Again here, $\rho_{iu}$ depends on the item $i$ only. when a new item $i$ is recommended for the first time, its success rate $\rho_i$ is drawn according to some distribution $\zeta$ over $[0,1]$, independently of the success rates of the other items. $\zeta$ is arbitrary and initially unknown, but for simplicity assumed to be absolutely continuous w.r.t. Lebesgue measure. To represent $\zeta$, we also use its inverse distribution function: for any $x\in [0,1]$, $\mu_x:=\inf\{\gamma\in [0,1]: \mathbb{P}[\rho_i\le \gamma]\ge x\}$. We say that an item $i$ is within the $\varepsilon$-best items if $\rho_i\ge \mu_{1-\varepsilon}$. We adopt the following notion of regret: for a given $\varepsilon>0$,
$
R_\varepsilon^\pi (T) =  \sum_{t=1}^T \mathbb{E}\left[ \max \{ 0, \mu_{1-\varepsilon} - \rho_{i^\pi_t} \} \right]. 
$
Hence, we assume that recommending items within the $\varepsilon$-best items does not generate any regret. We also assume, as in Model A, that an Oracle policy can always recommend such items (refer to Appendix \ref{app:regretdef}). This notion of {\it satisficing} regret \cite{russo-sat} has been used in the bandit literature to study problems with a very large number of arms (we have a large number of items). For such problems, identifying the best arm is very unlikely, and relaxing the regret definition is a necessity. Satisficing regret is all the more relevant in our problem that even if one would be able to identify the best item, we cannot recommend it (play it) more than $m$ times (due to the no-repetition constraint), and we are actually forced to recommend sub-optimal items. A similar notion of regret is used in \cite{bresler2014latent} to study recommendation systems in a setting similar to our Model B.   

{\bf Model C. Clustered items and clustered users.} We consider the case where both items and users are clustered. Specifically, users are classified into $L$ clusters ${\cal U}_1,\ldots,{\cal U}_L$, and when a user arrives to the system the first time, she is assigned to cluster ${\cal U}_\ell$ with probability $\beta_\ell$, independently of the other users. There are $K$ item clusters ${\cal I}_1,\ldots {\cal I}_K$. When the algorithm recommends an item $i$ for the first time, it is assigned to cluster ${\cal I}_k$ with probability $\alpha_k$ as in Model A. Now $\rho_{iu}=p_{k\ell}$ when $i\in {\cal I}_k$ and $u\in {\cal U}_\ell$. Again, we assume that $p=(p_{k\ell})_{k,\ell}$, $\alpha=(\alpha_k)_{k\in [K]}$ and $\beta=(\beta_\ell)_{\ell \in [L]}$ do not depend on $(n,m,T)$. For any $\ell$, let $k^\star_\ell = \arg\max_kp_{k\ell}$ be the best item cluster for users in ${\cal U}_\ell$. We assume that $k^\star_\ell$ is unique. In this scenario, we assume that an Oracle algorithm, aware of the item and user clusters and of the parameters $p$, would only recommend items from cluster $k^\star_\ell$ to a user in ${\cal U}_\ell$ (refer to Appendix \ref{app:regretdef}).  The regret of an algorithm $\pi\in \Pi$ is hence defined as:
$
R^\pi(T) = T\sum_{\ell}\beta_\ell p_{k^\star_\ell\ell} -\sum_{t=1}^T \mathbb{E}\left[\sum_{k,\ell}\indicator_{\{ u_t \in {\cal U}_\ell, i^\pi_t\in {\cal I}_k\} }p_{k\ell}\right].
$

\subsection{Preliminaries -- User arrival process}

The user arrival process is out of the decision maker's control and strongly impacts the performance of the recommendation algorithms. To analyze the regret of our algorithms, we will leverage the following results. Let $N_u(T)$ denote the number of requests of user $u$ up to round $T$. From the literature on "Balls and Bins process", see e.g. \cite{Raab98}, we know that if $\overline{n}:= \mathbb{E}[\max_{u\in {\cal U}}N_u(T)]$, then 
$$
\overline{n} = 
\left\{
\arraycolsep=1.2pt\def\arraystretch{1.6}
\begin{array}{ll}
{\log(m)\over \log({m\log(m)\over T})}(1+o(1)) & \hbox{if }\ \ T=o(m\log(m)),\\
\log(m)(d_c+o(1)) & \hbox{if }\ \ T=cm\log(m),\\
{T\over m}(1+o(1)) & \hbox{if }\ \ T=\omega(m\log(m)),
\end{array}
\right.
$$
where $d_c$ is a constant depending on $c$ only.
We also establish the following concentration result controlling the tail of the distribution of $N_u(T)$ (refer to Appendix \ref{app:preliminaries}):
\begin{lemma}
Define 
$\overline{N} = \frac{4\log (m)}{\log(\frac{m\log (m)}{T}+e)}+\frac{e^2 T}{m}.$
Then, $\forall u\in {\cal U}$, $\mathbb{E}[\max\{0,N_u (T) - \overline{N}\}]\leq \frac{1}{(e-1)m}$.
\label{lem:couponcollector}
\end{lemma}

The quantities $\overline{n}$ and $\overline{N}$ play an important role in our regret analysis. %

\section{Regret Lower Bounds}\label{sec:lower}

In this section, we derive regret lower bounds for the three envisioned structures. Interestingly, we are able to quantify the minimal regret induced by the specific features of the problem: (i) the no-repetition constraint, (ii) the unknown success probabilities, (iii) the unknown item clusters, (iv) the unknown user clusters. The proofs of the lower bounds are presented in Appendices \ref{app:thlowA}-\ref{app:thlowB}-\ref{app:thlowC}.

\subsection{Clustered items and statistically identical users}\label{subsec:lowA}

We denote by $\Delta_k=p_1-p_k$ the gap between the success rates of items from the best cluster and of items from cluster ${\cal I}_k$, and introduce the function:
$
\phi(k,m,p) = {1- e^{-m \gamma(p_1,p_{k})}\over 8(1- e^{-\gamma(p_1,p_{k})})},
$
where $\gamma(p,q)=\kl(p,q)+\kl(q,p)$ and $\kl(p,q) = p \log\frac{p}{q} + (1 -p) \log \frac{1-p}{1-q}$. Using the fact that $\kl(p,q)\le {(p-q)^2/ q(1-q)}$, we can easily show that as $m$ grows large, $\phi(k,m,p)$ scales as $\eta/(16\Delta_k^2)$ when $\Delta_k$ is small, where $\eta:=\min_kp_k(1-p_k)$. 

We derive problem-specific regret lower bounds, and as in the classical stochastic bandit literature, we introduce the notion of {\it uniformly} good algorithm. $\pi$ is uniformly good if its expected regret $R^\pi(T)$ is $O(\max\{ \sqrt{T}, {\log(m)\over \log({m\log(m)\over T}+e)}\})$ for all possible system parameters $(p,\alpha)$ when $T,m,n$ grow large with $T=o(nm)$ and $m\ge n$. As shown in the next section, uniformly good algorithms exist.

\begin{theorem}\label{th:lowA}
Let $\pi\in \Pi$ be an arbitrary algorithm. The regret of $\pi$ satisfies: for all $T\ge 1$ such that $m\ge c/\Delta_2^2$ (for some constant $c$ large enough),
$
R^\pi(T) \ge \max\{ R_{\mathrm{nr}}(T), R_{\mathrm{ic}}(T) \},
$
where $R_{\mathrm{nr}}(T)$ and $R_{\mathrm{ic}}(T)$, the regrets due to the no-repetition constraint and to the unknown item clusters, respectively, are defined by $R_{\mathrm{nr}}(T) := \overline{n}  \sum_{k\neq 1}\alpha_k\Delta_k$ and $R_{\mathrm{ic}}(T) := {T\over m} \sum_{k\neq 1}\alpha_k \phi(k,m,p)\Delta_k$.\\
Assume that $\pi$ is uniformly good, then we have\footnote{We write $a\gtrsim b$ if $\lim\inf_{T\to\infty}a/b\ge 1$.}:
$
R^\pi(T)\gtrsim  R_{\mathrm{sp}}(T):=\log(T) \sum_{k\neq 1} {\Delta_k\over 2 \kl (p_k,p_1)},
$
where $R_{\mathrm{sp}}(T)$ refers to the regret due to the unknown success probabilities.
\end{theorem}

From the above theorem, analyzing the way $R_{\mathrm{nr}}(T)$, $R_{\mathrm{ic}}(T)$, and $R_{\mathrm{sp}}(T)$ scale, we can deduce that:\\
(i) When $T=o(m\log(m))$, the regret arises mainly due to either the no-repetition constraint or the need to learn the success probabilities, and it scales at least as $\max\{ {\log(m)\over \log({m\log(m)\over T})}, \log(T)\}$. \\
(ii) When $T=c m\log(m)$, the three components of the regret lower bound scales in the same way, and the regret scales at least as $\log(T)$. \\
(iii) When $T=\omega(m\log(m))$, the regret arises mainly due to 
either the no-repetition constraint or the need to learn the item 
clusters, and it scales at least as ${T\over m}$.

\subsection{Unclustered items and statistically identical users}\label{subsec:lowB}

In this scenario, the regret is induced by the no-repetition constraint, and by the fact the success rate of an item when it is first selected and the distribution $\zeta$ are unknown. These two sources of regret lead to the terms $R_{\mathrm{nr}}(T)$ and $R_{\mathrm{i}}(T)$, respectively, in our regret lower bound.

\begin{theorem}\label{th:lowB}
Assume that the density of $\zeta$ satisfies, for some $C>0$, $\zeta(\mu) \le C$ for all $\mu\in [0,1]$. Let $\pi\in \Pi$ be an arbitrary algorithm. Then its satisficing regret satisfies: for all $T\ge 1$ such that $m\ge c/\varepsilon^2$ (for some constant $c\ge 1$ large enough),
$
R_\varepsilon^\pi(T) \ge \max\{ R_{\mathrm{nr}}(T), R_{\mathrm{i}}(T) \},
$
where $R_{\mathrm{nr}}(T):=\overline{n} \int_0^{\mu_{1-\varepsilon}} (\mu_{1-\varepsilon}-\mu)\zeta(\mu)d\mu$ and $R_{\mathrm{i}}(T):={T\over m} { \frac{(1-\varepsilon)^2}{2C}\left( 1 - \frac{\varepsilon C}{1 - \varepsilon}\right)^2 \over \min\{ 1, (1+C)\varepsilon \}  + 1/m}$. 
\end{theorem}

\subsection{Clustered items and clustered users}\label{subsec:lowC}

To state regret lower bounds in this scenario, we introduce the following notations. For any $\ell\in [L]$, let $\Delta_{k\ell}=p_{k^\star_\ell\ell}-p_{k\ell}$ be the gap between the success rates of items from the best cluster ${\cal I}_{k^\star_\ell}$ and of items from cluster ${\cal I}_k$.
 We also denote by $ \mathcal{R}_\ell=\left\{r\in [L]: {k_\ell^\star} \neq {k_r^\star}\right\}.$ 
 We further introduce the functions:
\begin{align*}
\phi(k,\ell,m,p) = \frac{1- e^{-m \gamma(p_{k^\star_\ell\ell},p_{k\ell})}}{ 8\left( 1- e^{-\gamma(p_{k^\star_\ell\ell},p_{k\ell})}\right)} 
\quad 
\text{and} 
\quad  
\psi(\ell,k, T, m, p)=\frac{1- e^{- {T\over m} \gamma(p_{k_\ell^\star \ell},p_{k \ell})}}{ 8 \left(1- e^{- \gamma(p_{k_\ell^\star \ell},p_{k \ell})}\right)}.
\end{align*}
Compared to the case of clustered items and statistically identical users, this scenario requires the algorithm to actually learn the user clusters. To discuss how this induces additional regret, assume that the success probabilities $p$ are known. Define ${\cal L}_{\perp}=\{ (\ell,\ell')\in [L]^2: p_{k_\ell^\star \ell}\neq p_{k_{\ell'}^\star\ell'}\}$, the set of pairs of user clusters whose best item clusters differ. If ${\cal L}_{\perp}\neq \emptyset$, then there isn't a single optimal item cluster for all users, and when a user $u$ first arrives, we need to learn its cluster. If $p$ is known, this classification generates at least a constant regret (per user) -- corresponding to the term $R_{\mathrm{uc}}(T)$ in the theorem below. For specific values of $p$, we show that this classification can even generate a regret scaling as $\log(T/m)$ (per user). This happens when ${\cal L}^{\perp}(\ell)=\{ \ell'\neq \ell: k_\ell^\star\neq k_{\ell'}^\star, p_{k_\ell^\star\ell}=p_{k_\ell^\star\ell'}\}$ is not empty -- refer to Appendix \ref{app:thlowC} for examples. In this case, we cannot distinguish users from ${\cal U}_\ell$ and ${\cal U}_{\ell'}$ by just presenting items from ${\cal I}_{k_\ell^\star}$ (the greedy choice for users in ${\cal U}_\ell$). The corresponding regret term in the theorem below is $R_{\mathrm{uc}}'(T)$. To formalize this last regret component, we define uniformly good algorithms as follows. An algorithm is uniformly good if for any user $u$, $R_u^\pi(N)=o(N^\alpha)$ as $N$ grows large for all $\alpha>0$, where $R_u^\pi(N)$ denotes the accumulated expected regret under $\pi$ for user $u$ when the latter has arrived $N$ times.  

\begin{theorem}\label{th:lowC}
Let $\pi\in \Pi$ be an arbitrary algorithm. Then its regret satisfies: for all $T\ge 2m$ such that $m\ge c/\min_{k,\ell}\Delta_{k\ell}^2$ (for some constant $c$ large enough),
$
R^\pi(T) \ge \max\{ R_{\mathrm{nr}}(T), R_{\mathrm{ic}}(T),R_{\mathrm{uc}}(T) \},
$
where $R_{\mathrm{nr}}(T)$, $R_{\mathrm{ic}}(T)$, and $R_{\mathrm{uc}}(T)$ are regrets due to the no-repetition constraint, to the unknown item clusters, and to the unknown user clusters respectively, defined by:
$$
\left\{
\arraycolsep=1.2pt\def\arraystretch{1.4}
\begin{array}{l}
R_{\mathrm{nr}}(T) := \overline{n} \sum_\ell\beta_\ell \sum_{k\neq k^\star_\ell}\alpha_k\Delta_{k\ell},\\
R_{\mathrm{ic}}(T) := {T\over m} \sum_\ell\beta_\ell\sum_{k\neq k^\star_\ell}\alpha_k \phi(k,\ell,m,p)\Delta_{k\ell},\\
R_{\mathrm{uc}}(T):= m \sum_{\ell \in [L]} \beta_\ell \frac{\sum_{k \in \set{R}_\ell} \Delta_{k \ell} \psi(\ell, k, T, m, p)}{K}.%
\end{array}
\right.
$$
In addition, when $T = \omega(m)$, if $\pi$ is uniformly good, $R^\pi(T)\gtrsim R_{\mathrm{uc}}'(T):=c(\beta,p)m\log(T/m)$ where $c(\beta,p)=\inf_{n\in {\cal F}} \sum_{\ell} \beta_\ell \sum_{k \neq k_\ell^\star} \Delta_{k \ell} n_{k \ell}$ with \\ ${\cal F}=\{n \ge 0 : \forall \ell, \; \forall \ell' \in \set{L}^\perp(\ell), \sum_{k \neq k_\ell^\star}\kl(p_{k \ell}, p_{k \ell'}) n_{k \ell}  \ge 1\}$. 
\end{theorem}

Note that we do not include in the lower bound the term $R_{\mathrm{sp}}(T)$ corresponding to the regret induced by the lack of knowledge of the success probabilities. Indeed, it would scale as $\log(T)$, and this regret would be negligible compared to $R_{\mathrm{uc}}(T)$ (remember that $T=o(m^2)$), should ${\cal L}_\perp\neq\emptyset$. Under the latter condition, the main component of regret is for any time horizon is due to the unknown user clusters. When $\set{L}_\perp \neq \emptyset$, the regret scales at least as $m$ if for all $\ell$, ${\cal L}^{\perp}(\ell)=\emptyset$, and $m\log(T/m)$ otherwise.

\section{Algorithms}\label{sec:algo}

This section presents algorithms for our three structures and an analysis of their regret. The detailed pseudo-codes of our algorithms and numerical experiments are presented in Appendix \ref{app:algo}. The proofs of the regret upper bounds are postponed to Appendices \ref{app:upperA}-\ref{app:upperB}-\ref{app:upperC}.

\subsection{Clustered items and statistically identical users}

To achieve a regret scaling as in our lower bounds, the structure needs to be exploited. Even without accounting for the no-repetition constraint, the KL-UCB algorithm would, for example, yield a regret scaling as ${n\over \Delta_2}\log(T)$. Now we could first sample $T/m$ items and run KL-UCB on this restricted set of items -- this would yield a regret scaling as ${T\over m\Delta_2}\log(T)$, without accounting for the no-repetition constraint. Our proposed algorithm, Explore-Cluster-and-Test (ECT), achieves a better regret scaling and complies with the no-repetition constraint. Refer to Appendix \ref{app:algo} for numerical experiments illustrating the superiority of ECT. 

{\bf The Explore-Cluster-and-Test algorithm.} ECT proceeds in the following phases:

(a) {\bf Exploration phase.} This first phase consists in gathering samples for a subset ${\cal S}$ of randomly selected items so that the success probabilities and the clusters of these items are learnt accurately. Specifically, we pick $|{\cal S}| =\lfloor \log(T)^2\rfloor$ items, and for each of these items, gather roughly $\log(T)$ samples. 

(b) {\bf Clustering phase.} We leverage the information gathered in the exploration phase to derive an estimate $\hat{\rho_i}$ of the success probability $\rho_i$ for item $i\in {\cal S}$. These estimates are used to cluster items, using an appropriate version of the K-means algorithm. In turn, we extract from this phase, accurate estimates $\hat{p}_1$ and $\hat{p}_2$ of the success rates of items in the two best item clusters, and a set ${\cal V}\subset {\cal S}$ of items believed to be in the best cluster: ${\cal V}:=\{i\in {\cal S}: \hat\rho_i > (\hat{p}_1+ \hat{p}_2)/2\}$. 

(c) {\bf Test phase.} The test phase corresponds to an exploitation phase. Whenever this is possible (the no-repetition constraint is not violated), items from ${\cal V}$ are recommended. When an item outside ${\cal V}$ has to be selected due to the no-repetition constraint, we randomly sample and recommend an item outside ${\cal V}$. This item is appended to ${\cal V}$. To ensure that any item $i$ in the (evolving) set ${\cal V}$ is from the best cluster with high confidence, we keep updating its empirical success rate $\hat{\rho}_i$, and periodically test whether $\hat{\rho}_i$ is close enough from $\hat{p}_1$. If this is not the case, $i$ is removed from ${\cal V}$.     
 
In all phases, ECT is designed to comply with the no-repetition constraint: for example, in the exploration phase, when the user arrives, if we cannot recommend an item from ${\cal S}$ due to the constraint, we randomly select an item not violating the constraint. In the analysis of ECT regret, we upper bound the regret generated in rounds where a random item selection is imposed. Observe that ECT does not depend on any parameter (except for the choice of the number of items initially explored in the first phase). 

\begin{theorem}
We have:
$R^{\mathrm{ECT}}(T) = \set{O}\Bigg(\frac{2\overline{N}}{\alpha_1} \sum_{k=2}^{K}  
\frac{\alpha_k(p_1-p_k)}{({p}_1-{p}_2)^2} + (\log T)^3\Bigg)$.
\label{thm:algorithmA}
\end{theorem}
The regret lower bound of Theorem \ref{th:lowA} states that for any algorithm $\pi$, $R^\pi(T)=\Omega(\overline{N})$, and if $\pi$ is uniformly good $R^\pi(T)=\Omega(\max\{\overline{N},\log(T)\} )$. Thus, in view of the above theorem, ECT is order-optimal if $\overline{N}=\Omega((\log T)^3)$, and order-optimal up to an $(\log T)^2$ factor otherwise. Furthermore, note that when $R_{\mathrm{ic}}(T)=\Omega({T\over \Delta_2 m})$ is the leading term in our regret lower bound, ECT regret has also the right scaling in $\Delta_2$:  $R^{\mathrm{ECT}}(T)= \set{O}({T\over \Delta_2 m})$.

\subsection{Unclustered items and statistically identical users}

When items are not clustered, we propose ET (Explore-and-Test), an algorithm that consists of two phases: an exploration phase that aims at estimating the threshold level $\mu_{1-\varepsilon}$, and a test phase where we apply to each item a sequential test to determine whether the item if above the threshold. 

{\bf The Explore-and-Test algorithm.} The ET algorithm proceeds as follows.

(a) {\bf Exploration phase.} In this phase, we randomly select of set ${\cal S}$ consisting of $\lfloor \frac{8^2}{\varepsilon^2} \log T\rfloor$ items and recommend each selected item to $\lfloor 4^2 \log T\rfloor$ users. For each item $i\in {\cal S}$, we compute its empirical success rate $\hat{\rho}_i$. We then estimate ${\mu}_{1-\frac{\varepsilon}{2}}$ by $\hat{\mu}_{1-\frac{\varepsilon}{2}}$ defined so that:
$
\frac{\varepsilon}{2} |\mathcal{S}|  = \left|\{ i\in \mathcal{S} : \hat{\rho}_i \ge \hat{\mu}_{1-\frac{\varepsilon}{2}} \} \right|.
$
We also initialize the set $\mathcal{V}$ of candidate items to exploit as $\mathcal{V} = \{ i\in \mathcal{S} : \hat{\rho}_i \ge \hat{\mu}_{1-\frac{\varepsilon}{2}} \}$.

(b) {\bf Test phase.} In this phase, we recommend items in ${\cal V}$, and update the set ${\cal V}$. Specifically, when a user $u$ arrives, we recommend the item $i\in {\cal V}$ that has been recommended the least recently among items that would not break the no-repetition constraint. If no such items exist in ${\cal V}$, we randomly recommend an item outside ${\cal V}$ and add it to ${\cal V}$. \\
Now to ensure that items in ${\cal V}$ are above the threshold, we perform the following sequential test, which is reminiscent of sequential tests used in optimal algorithms for infinite bandit problems \cite{bonald2013}. For each item, the test is applied when the item has been recommended for the $\lfloor 2^{\ell} \log\log_2( 2^e m^2) \rfloor$ times for any positive integer $\ell$. For the $\ell$-th test, we denote by $\bar{\rho}^{(\ell)}$ the real number such that $\kl (\bar{\rho}^{(\ell)}, \hat{\mu}_{1-\frac{\varepsilon}{2}}) = 2^{-\ell}$. If $\bar{\rho}^{(\ell)} \le \hat{\mu}_{1-\frac{\varepsilon}{2}}$, the item is removed from ${\cal V}$.

\begin{theorem} Assume that the density of $\zeta$ satisfies $
\zeta(\mu)\le C$ for all $\mu\in [0,1]$.\\ 
For any $\varepsilon  \ge C 
\sqrt{\frac{\pi}{2\log T}}$, we have:
$
R^{\mathrm{ET}}_\varepsilon (T) = \set{O}\left( \overline{N}\frac{ \log(1/\varepsilon) \log\log(m)}{\varepsilon }   + \frac{(\log T)^2}{\varepsilon^2}  \right). 
$
\label{thm:modelBupper}
\end{theorem}
{\vskip -0.3cm}
In view of Theorem \ref{th:lowB}, the regret of any algorithm scales at least as $\Omega({\overline{N}\over \varepsilon})$. Hence, the above theorem states that ET is order-optimal at least when $\overline{N}=\Omega((\log T)^2)$.

\subsection{Clustered items and clustered users}

The main challenge in devising an algorithm in this setting stems from the fact that we do not control the user arrival process. In turn, clustering users with low regret is delicate. We present Explore-Cluster with Upper Confidence Sets (EC-UCS), an algorithm that essentially exhibits the same regret scaling as our lower bound. The idea behind the design of EC-UCS is as follows. We estimate the success rates $(p_{k\ell})_{k,\ell}$ using small subsets of items and users. Then based on these estimates, each user is optimistically associated with a UCS, {\it Upper Confidence Set}, a set of clusters the user may likely belong to. The UCS of a user then shrinks as the number of requests made by this user increases (just as the UCB index of an arm in bandit problems gets closer to its average reward). The design of our estimation procedure and of the various UCS is made so as to get an order-optimal algorithm. In what follows, we assume that $m^2\ge T(\log T)^3$ and $T \ge m\log(T)$. 

{\bf The Explore-Cluster-with-Upper Confidence Sets algorithm.} 

(a) {\bf Exploration and item clustering phase.} The algorithm starts by collecting data to infer the item clusters. It randomly selects a set ${\cal S}$ consisting of $\min\{ n, \lfloor\frac{m}{(\log T)^2}\rfloor \}$ items. For the $10m$ first user arrivals, it recommends items from $\mathcal{S}$ uniformly at random. These $10m$ recommendations and the corresponding user responses are recorded in the dataset $\mathcal{D}$. From the dataset $\mathcal{D}$, the item clusters are extracted using a spectral algorithm (see Algorithm \ref{alg:SIC_B} in the appendix). This algorithm is taken from \cite{yun2014streaming}, and considers the {\it indirect edges} between items created by users. Specifically, when a user appears more than twice in ${\cal D}$, she creates an indirect edge between the items recommended to her for which she provided the same answer (1 or 0). Items with indirect edges are more likely to belong to the same cluster. The output of this phase is a partition of ${\cal S}$ into item clusters $\hat{I}_1, \dots,\hat{I}_K$. We can show that with an exploration budget of $10m$, w.h.p. at least $m/2$ indirect edges are created and that in turn, the spectral algorithm does not make any clustering errors w.p. at least $1-{1\over T}$. 

(b) {\bf Exploration and user clustering phase.} To the $(10 + \log (T))m$ next user arrivals, EC-UCS clusters a subset of users using a Nearest-Neighbor algorithm. The algorithm selects a subset ${\cal U}^\star$ of users to cluster, and recommendations to the remaining users will be made depending some distance to the inferred clusters in ${\cal U}^\star$. Users from all clusters must be present in ${\cal U}^\star$. To this aim, EC-UCS first randomly selects a subset ${\cal U}_0$ of $\lfloor m/\log(T)\rfloor$ users from which it extracts the set ${\cal U}^\star$ of $\lfloor \log(T)^2\rfloor$ users who have been observed the most. The extraction and the clustering of ${\cal U}^\star$ is made several times until the $\lfloor (10+\log(T))m\rfloor$-th user arrives so as to update and improve the user clusters. From these clusters, we deduce estimates $\hat{p}_{k\ell}$ of the success probabilities.  

(c) {\bf Recommendations based on Optimistic Assignments.} After the $10 m$-th arrivals, recommendations are made based on the estimated $\hat{p}_{k\ell}$'s. For user $u_t\notin {\cal U}_0$, the item selection further depends on the $\hat{\rho}_{ku_t}$'s, the empirical success rates of user $u_t$ for items in the various clusters. A greedy recommendation for $u_t$ would consist in assigning $u_t$ to cluster $\ell$ minimizing $\| \hat{p}_{\cdot \ell}-\hat{\rho}_{\cdot u_t}\|$ over $\ell$, and then in picking an item from cluster $\hat{I}_k$ with maximal $\hat{p}_{k\ell}$. Such a greedy recommendation would not work as when $u_t$ has not been observed many times, the cluster she belongs to remains uncertain. To address this issue, we apply the Optimism in Front of Uncertainty principle often used in bandit algorithms to foster exploration. Specifically, we build a set ${\cal L}(u_t)$ of clusters $u_t$ is likely to belong to. ${\cal L}(u_t)$ is referred to as the Upper Confidence Set of $u_t$. As we get more observations of $u_t$, this set shrinks. Specifically, we let $x_{k\ell}=\max\{ | \hat{p}_{k\ell} - \hat{\rho}_{k u_t}|-\epsilon,0\}$, for some well defined $\epsilon>0$ (essentially scaling as $\sqrt{\log\log(T)/\log(T)}$, see Appendix \ref{app:algo} for details), and define ${\cal L}(u_t)= \{\ell\in [L]: \sum_kx_{k\ell}^2n_{ku_t}< 2K\log(n_{u_t})\}$ ($n_{u_t}$ is the number of time $u_t$ has  arrived, and $n_{ku_t}$ is the number of times $u_t$ has been recommended an item from cluster $\hat{I}_k$). After optimistically composing the set ${\cal L}(u_t)$, $u_t$ is assigned to cluster $\ell$ chosen uniformly at random in ${\cal L}(u_t)$, and recommended an item from cluster $\hat{I}_k$ with maximal $\hat{p}_{k\ell}$. 

\begin{theorem}
For any $\ell$, let $\sigma_\ell$ be the permutation of $[K]$ such that $p_{\sigma_\ell(1)\ell}>p_{\sigma_\ell(2)\ell}\ge \dots \ge p_{\sigma_\ell(K)\ell}$. Let $\mathcal{R}_\ell=\left\{r\in [L]: {k_\ell^\star} \neq {k_r^\star}\right\}$, $\mathcal{S}_{\ell r} =\{k \in [K] : p_{k\ell} \neq p_{kr}\}$, $y_{\ell r} = \min_{k\in \mathcal{S}_{\ell r}}|p_{k \ell}-p_{kr}|$, $\delta=\min_\ell (p_{\sigma_\ell(1)\ell}-p_{\sigma_\ell(2)\ell})$, and $\phi(x) \coloneqq {x}/{\log\left(1/x\right)}$. Then, we have:
\begin{align*}
&R^{\mathrm{EC-UCS}}(T) = \set{O} \left( m \sum_{\ell} \beta_\ell (p_{\sigma_\ell(1)\ell}- p_{\sigma_\ell(K)\ell}) \left(\max \left(\frac{K^{3}\log K}{\phi(\min(y_{\ell r},\delta)^2)}, \frac{\sqrt{K}}{\min_\ell \beta_\ell}\right) \right.\right. \cr
&\qquad\qquad\qquad\qquad + \left.\left. \sum_{r\in\mathcal{R}_\ell\setminus\mathcal{L}^{\perp}(\ell)} \frac{K^2\log K}{ \phi(|{p}_{k_\ell^* r}-{p}_{k_\ell^* \ell}|^2)}+ \sum_{k\in \mathcal{S}_{\ell r}} \sum_{r\in\mathcal{L}^{\perp}(\ell)} \frac{K\log \overline{N}}{|\mathcal{S}_{\ell r}||{p}_{k\ell}-{p}_{kr}|^2}\right)\right).
\end{align*}
\label{thm:newalgC}
\end{theorem}
{\vskip -0.5cm}
EC-UCS blends clustering and bandit algorithms, and its regret analysis is rather intricate. The above theorem states that remarkably, the regret of the EC-UCS algorithm macthes our lower bound order-wise. In particular, the algorithm manages to get a regret (i) scaling as $m$ whenever it is possible, i.e., when $\mathcal{L}^{\perp}(\ell)=\emptyset$ for all $\ell$, (ii) scaling as $m\log(\overline{N})$ otherwise. 

In Appendix \ref{app:algoC}, we present ECB, a much simpler algorithm than EC-UCS, but whose regret upper bound, derived in Appendix \ref{app:upperECB}, always scales as $m \log(\overline{N})$.

\section{Conclusion}

This paper proposes and analyzes several models for online recommendation systems. These models capture both the fact that items cannot repeatedly be recommended to the same users and some underlying user and item structure. We provide regret lower bounds and algorithms approaching these limits for all models. 
Many interesting and challenging questions remain open. We may, for example, investigate other structural assumptions for the success probabilities (e.g. soft clusters), and adapt our algorithms. 
We may also try to extend our analysis to the very popular linear reward structure, but accounting for no-repetition constraint.

\section*{Broader Impact}

This work, although mostly theoretical, may provide guidelines and insights towards an improved design of recommendation systems. The benefits of such improved design could be to increase user experience with these systems, and to help companies to improve their sales strategies through differentiated recommendations. The massive use of recommendation systems and its potential side effects have recently triggered a lot of interest. We must remain aware of and investigate such effects. These include: opinion polarization, a potential negative impact on users' behavior and their willingness to pay, privacy issues.

\section*{Acknowledgements}
K. Ariu was supported by the Nakajima Foundation Scholarship. S. Yun and N. Ryu were supported by Institute of Information \& communications Technology Planning \& Evaluation (IITP) grant funded by the Korea government(MSIT)(No.2019-0-00075, Artificial Intelligence Graduate School Program(KAIST)). A. Proutiere's research is supported by the Wallenberg AI, Autonomous Systems and Software Program (WASP) funded by the Knut and Alice Wallenberg Foundation.

\medskip

\small
\bibliographystyle{plainnat}
\bibliography{ref,References,ReferencesK}

\normalsize
\clearpage

\section{Table of Notations}
\begin{table}[htbp]
	\begin{center}%
		\small 
		\begin{tabular}{c c p{10cm} }
			\toprule
			\multicolumn{3}{l}{\bf Notations common to all models}\\
			\hline
			$n$ &   & Number of items\\
			$m$ &  &  Number of users \\
			$\set{I}$ & & Set of items\\
			$\set{U}$  & & Set of users\\
			$u_t$& & User requesting recommendation at round $t$\\
			$T$ & & Time horizon\\
			$X_{i u}$ & &  Binary random variable to indicate whether user $u$ likes the item $i$\\
			$\rho = (\rho_{i u })_{i \in \set{I}, u \in \set{U}}$ & & Probability that the user $u$ likes the item $i$\\
			$\pi$ & & Algorithm for sequential item selection\\
			$\Pi$ & & Set of all algorithms for sequential item selection\\
			$i_t^\pi $ & & Item selected at round $t$ under $\pi$ \\
			$ \set{F}^\pi_{t-1}$ & & $\sigma$-algebra generated by $(u_t, (u_s, i_s^\pi, X_{i_s^\pi u_s}), s \le t-1)$\\
			$\overline{n}$& & Term $\EXP[\max_{u \in \set{U}} N_u(T)]$\\
			$\overline{N}$& & Term $ \frac{4 \log(m)}{ \log \left(\frac{m \log (m)}{T} + e\right) } + \frac{e^2T}{m}$ 
			\\ \multicolumn{3}{c}{}\\
			\hline
			\multicolumn{3}{l}{\bf Generic notations}\\
			\hline
			$\hat{a}$& & Estimated value of $a$ \\
			$ \sigma(A)$ & & $\sigma$-algebra generated by $A$\\
			$\kl(p, q)$& &  Kullback–Leibler divergence from Bernoulli random variable with parameter $p$ to that with parameter $q$\\
			$ \gtrsim$ & & We write $a\gtrsim b$ if $\lim\inf_{T\to\infty}a/b\ge 1$
			\\ \bottomrule
		\end{tabular}
		\normalsize
	\end{center}
	\label{tab:TableOfNotationsGeneral}
	\caption{Table of notations common to all models}
\end{table}

\begin{table}[htbp]
	\begin{center}%
		\small 
		\begin{tabular}{c c p{10cm} }
			\toprule
			\multicolumn{3}{l}{\bf Model A: Clustered items and statistically identical users}\\
			\hline
			$\set{I}_k$ & & Set of items in the item cluster $k$\\
			$\alpha = ( \alpha_k)_{k \in [K]}$ & & Probability that an item is assigned to the item cluster $k$\\
			$K$& & Number of item clusters\\
			$\Delta$& & Minimum difference between the success rates of items in optimal cluster and of items in sub-optimal cluster\\
			$p = (p_k)_{k \in [K]}$ & &  Probability that the user likes the item $i \in \set{I}_k$\\
			 $R^\pi(T)$ & & Regret of an algorithm $\pi$ \\
			$ \Delta_k$ & & Term $p_1 - p_k$\\
			$\phi(k, m, p)$ & & Term $ \frac{1 - e^{-m \gamma(p_1, p_k)}}{ 8(1 - e^{-\gamma(p_1, p_k)})}$  
			\\
			$\gamma(p, q)$ & & Term $ \kl(p,q) + \kl(q, p)$ \\
			$ \eta$& & Term $\min_k p_k (1 - p_k)$\\
			$\set{S}$ & & Set of initially sampled items \\
			$\set{V}$ & & Set of items believed to be in the best cluster
			\\ \bottomrule
		\end{tabular}
		\normalsize
	\end{center}
	\label{tab:TableOfNotationsA}
	\caption{Table of notations: Model A}
\end{table}

\begin{table}[htbp]
	\begin{center}%
		\small 
		\begin{tabular}{c c p{10cm} }
			\toprule
			\multicolumn{3}{l}{\bf Model B: Unclustered items and statistically identical users}\\
			\hline
			$\zeta$ & & Distribution over $[0,1]$\\
			$\mu_x$ & & Term $ \inf\{\gamma \in [0, 1] : \Pr[\rho_i \le \gamma] \ge x\}$\\
			$\varepsilon$ & & Constant  that specifies the $\varepsilon$-best items\\
			$R^\pi_\varepsilon(T)$ & & Satisficing regret of algorithm $\pi$ with a given $\varepsilon>0$\\
			$C$ & & Constant that regularizes the distribution $\zeta(\mu)$\\
			$\set{S}$ & & Set of initially sampled items \\
			$\set{V}$ & & Set of items believed to be in the best cluster
			\\ \bottomrule
		\end{tabular}
		\normalsize
	\end{center}
	\label{tab:TableOfNotationsB}
	\caption{Table of notations: Model B}
\end{table}

\begin{table}[htbp]
	\begin{center}%
		\small 
		\begin{tabular}{c c p{10cm} }
			\toprule
			\multicolumn{3}{l}{\bf Model C: Clustered items and clustered users}\\
			\hline
			$\set{I}_k$ & & Set of items in the item cluster $k$\\
			$\alpha = ( \alpha_k)_{k \in [K]}$ & & Probability that an item is assigned to the item cluster $k$\\
			$\set{U}_\ell$& & Set of users in the user cluster $\ell$\\
			$\beta = ( \beta_\ell)_{\ell \in [L]}$ & & Probability that a user is assigned to the user cluster $\ell$\\
			$K$& & Number of item clusters\\
			$L$ & &  Number of user clusters\\
			$p = (p_{k \ell})_{k \in [K], \ell \in [L]}$& & Probability that the user $u$ likes the item $i$ such that $i \in \set{I}_k$ and $u \in \set{U}_\ell$\\
			$ k_\ell^\star$ & & Term $ \arg \max_k p_{k \ell}$\\ 
			$ \Delta_{k \ell}$ & & Term $p_{k_\ell^\star \ell} - p_{k \ell}$\\
			$\Delta$& & Minimum difference between the success rates of items in optimal cluster and of items in sub-optimal cluster\\
			$\delta_\ell$ & & Term $ \min_{k:\Delta_{k\ell}>0} \Delta_{k\ell}$\\
			$\phi(k,\ell,m,p) $ & & Term $ \frac{1- e^{-m \gamma(p_{k^\star_\ell\ell},p_{k\ell})}}{ 8\left( 1- e^{-\gamma(p_{k^\star_\ell\ell},p_{k\ell})}\right)}$\\
			$\psi(\ell, k, T, m, p)$ & & Term $ \frac{1 - e^{ - \frac{T}{m}\gamma(p_{k_\ell^\star}, p_k) }}{8\left( 1 - e^{ - \gamma(p_{k_\ell^\star}, p_k) }\right)}$\\
			${\cal L}_{\perp}$& & Set $\{ (\ell,\ell')\in [L]^2: p_{k_\ell^\star \ell}\neq p_{k_{\ell'}^\star\ell'}\}$\\
			${\cal L}^{\perp}(\ell)$& & Set $\{ \ell'\neq \ell: k_\ell^\star\neq k_{\ell'}^\star, p_{k_\ell^\star\ell}=p_{k_\ell^\star\ell'}\}$\\
			$R_u^\pi(N)$ & & Accumulated expected regret under $\pi$ for user $u$ when the user has arrived $N$ times\\
			$R^\pi(T)$ & & Regret of an algorithm $\pi$\\
			$\set{S}$ & & Set of initially sampled items\\
			$\set{U}_0$& & Set of initially sampled users\\
			$\set{U}^\star$ & & Set of $(\log T)^2$  users in $\set{U}_0$ who have been arrived the most\\
			$\set{L}(u_t)$& & Upper Condifence Set of the user $u_t$\\
			 $\sigma_\ell$ 	& &Permutation of $[K]$ such that $p_{\sigma_\ell(1)\ell}>p_{\sigma_\ell(2)\ell}\ge \dots \ge p_{\sigma_\ell(K)\ell}$ \\
			$\mathcal{R}_\ell$ & & Set $\left\{r\in [L]: {k_\ell^\star} \neq {k_r^\star}\right\}$\\
			 $\mathcal{S}_{\ell r}$ & & Set $ \{k \in [K] : p_{k\ell} \neq p_{kr}\}$ \\
			$y_{\ell r}$& &  Term $\min_{k\in \mathcal{S}_{\ell r}}|p_{k \ell}-p_{kr}|$\\
			 $\delta$  & & Term $\min_\ell (p_{\sigma_\ell(1)\ell}-p_{\sigma_\ell(2)\ell})$\\
			 $ \epsilon$& & Term $ K\sqrt{\frac{8Km}{t} \log\frac{t}{m}}$ (Updated only when the user clustering is executed)
			\\ \bottomrule
		\end{tabular}
		\normalsize
	\end{center}
	\label{tab:TableOfNotationsC}
	\caption{Table of notations: Model C}
\end{table}

\clearpage

\appendix

\newpage
\section{Algorithms and experiments}\label{app:algo}

In this section, we present the detailed pseudo-codes of our algorithms. We also illustrate the performance of these algorithms numerically.

\subsection{Clustered items and statistically identical users}

\begin{algorithm}[htb]
	\caption{Explore-Cluster-and-Test}
	\begin{algorithmic}
		\STATE {\bf Input:} $T, K$
		\STATE {\bf 1. Exploration}
		\STATE Sample a set $\mathcal{S}$ of $\lfloor(\log T)^2\rfloor$ items (uniformly at random)		
		\STATE Recommend each item in $\set{S}$ to $\lfloor \log T\rfloor$ users
		\STATE (when this is not possible due to the no-repetition constraint) recommend a random feasible item.
		\STATE $T_0\gets$ round where the exploration phase ends
		\STATE{\bf 2. Clustering}
		\STATE $\hat{\rho}_i\gets $ the empirical average of $i$ for all $i \in \mathcal{S}$
		\STATE $Q_i \gets \{j\in \set{S}: |\hat{\rho}_i-\hat{\rho}_j| \leq (\log T)^{-\frac{1}{4}} \}$ for all $i \in \set{S}$
		\STATE $M \gets \emptyset$
		\FOR{$k=1$ {\bfseries to} $K$}
		\STATE $i_k \gets \arg \max_{j \in \mathcal{S}} |Q_j \setminus \cup_{\ell =1}^{k-1} Q_{i_\ell}|$  
		\IF {$ |Q_{i_k}| < \log T$}
		\STATE {\bf break}
		\ENDIF
		\STATE $M \gets M \cup \{i_k\}$
		\ENDFOR
		\STATE $i_1 \gets \arg \max_{i \in M} \hat{\rho}_i\ $ and $\ \hat{p}_1 \gets \hat{\rho}_{i_1}$
		\STATE $i_2 \gets \arg \max_{i \in M\setminus \{i_1\}} \hat{\rho}_i\ $ and $\ \hat{p}_2 \gets \hat{\rho}_{i_2}$
		\STATE{\bf 3. Test}
		\STATE $\Delta_0 \gets \hat{p}_1 - \hat{p}_2$
		\STATE $\mathcal{V}, \mathcal{V}_0 \gets \{i\in \mathcal{S}:\hat{\rho}_i > (\hat{p}_1 + \hat{p}_2) /2\}$
		\FOR {$t=(T_0+1)$ {\bfseries to} $T$}
		\STATE Recommend item from $\set{V}$ with the highest empirical average if possible, otherwise randomly recommend item $i$ from $\set{I} \setminus \set{V}_0$ and add $i$ to $\set{V}$ and $\set{V}_0$
		\IF {the number of times $i$ has been recommended is a multiple of $\lfloor \frac{2\log 3}{\Delta_0^2} \rfloor$ and $\hat{\rho}_{i} < (\hat{p}_1 + \hat{p}_2) /2$}
		\STATE Remove $i$ from $\set{V}$
		\ENDIF
		\ENDFOR
		
	\end{algorithmic}
	
	\label{alg:Model_A}
\end{algorithm}

{\bf Numerical experiments.} We illustrate the performance of ECT in the following scenario: $K=2$ item clusters, $n=3000$ items, $m=5000$ users, $p_1 = 0.7$, $p_2 = 0.2$, $\alpha_1=\alpha_2 = 0.5$. 

We compare the performance of ECT to two naive algorithms:

(i) B-KLUCB \cite{combes2015bandits}: This algorithm was proposed for budgeted bandits. Here the budget per arm is $m$. The algorithm ranks the arms (the items) according to their KL-UCB indexes, and selects the available item (accounting for the no-repetition constraint) with the highest index.    

(ii) B-KLUCB with sampling: The algorithm first samples $\lfloor (T/m) \log T\rfloor$ items randomly, and play B-KLUCB only for these items. When none of these items can be played, the algorithm plays a randomly selected item (as in ECT).

\begin{figure}[h]
	\centering
	\includegraphics[width=0.8\columnwidth]{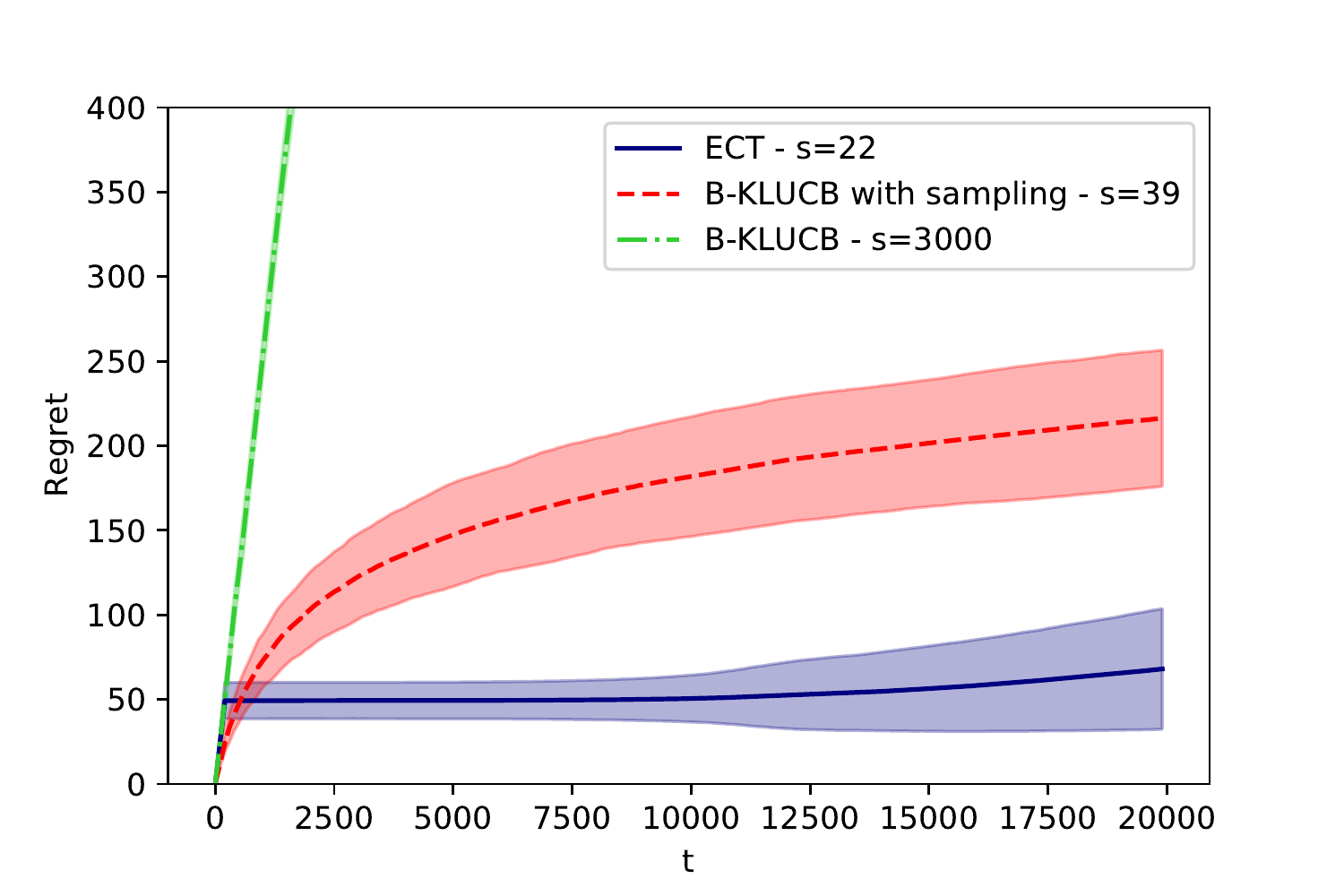} 
	\caption{Regret vs time averaged over 200 instances for ECT and B-KLUCB with sampling, and over 20 instances for B-KLUCB. $T = 20000$, $n=3000$, $m=5000$, $p_1 = 0.7$, $p_2 = 0.2$, $\alpha_1=\alpha_2=0.5$. Shaded areas correspond to one standard deviations. }
	\label{fig:modelA}
\end{figure}

Figure~\ref{fig:modelA} plots the regret vs time for the 3 algorithms, for a time horizon $T=20000$. The regret is averaged over 200 runs for ECT and B-KLUCB with sampling, and 20 runs for B-KL-UCB (we do not need more runs since there is no randomness induced by the initial item sampling procedure). For ECT, the number of items initially sampled is $|\set{S}| = \lfloor 0.225 \log (T)^2\rfloor $, whereas for B-KLUCB, it is 39 (this number is optimized so as to get the best performance -- refer to Figure \ref{fig:modelA3} for a sensitivity analysis of the regret depending on the number of items initially sampled).  

\begin{figure}[h]
	\centering
	\includegraphics[width=0.8\columnwidth]{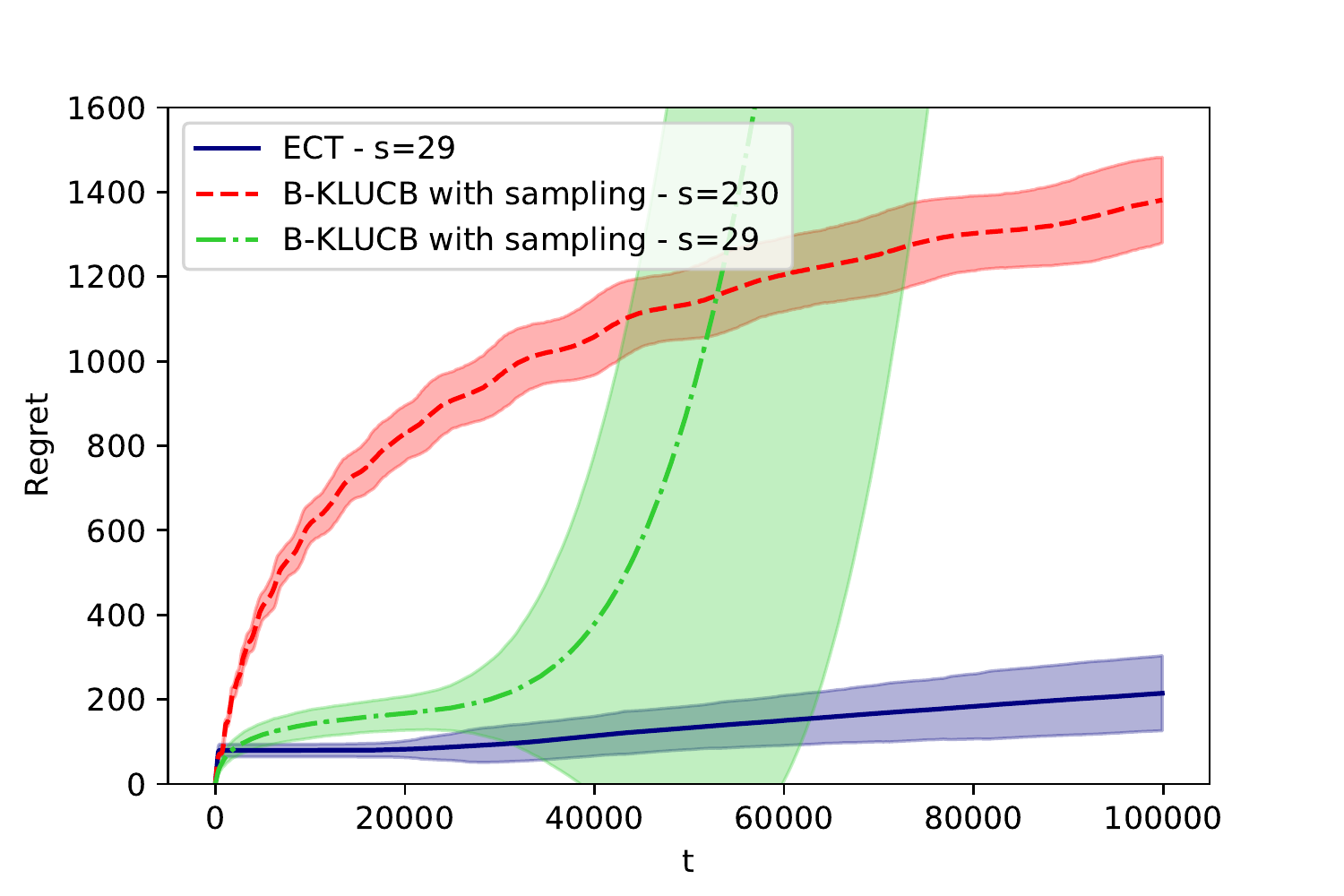} 
	\caption{Regret vs time averaged over 100 instances for each algorithm. $T = 100000$, $n=3000$, $m=5000$, $p_1 = 0.7$, $p_2 = 0.2$, $\alpha_1=\alpha_2=0.5$. Shaded areas correspond to one standard deviations.  }
	\label{fig:modelA2}
\end{figure}

Figure~\ref{fig:modelA2} compares the regret of ECT to that obtained under B-KLUCB with sampling when $T=100000$. ECT initially samples 29 items. For B-KLUCB with sampling, we have tested two different numbers of items initially sampled, namely 29 and 230. After round 20000, ECT starts playing items that have not being used in the exploration phase. To keep regret low, ECT hence relies on sequential tests. The regret curve of ECT shows that these tests perform very well. This contrasts with B-KLUCB with sampling: when the number of initially sampled items is 29, as for ECT, after round 20000, new items must be selected, and B-KLUCB performs very poorly (the regret rapidly grows). 

Finally, we assess the sensitivity of ECT and B-KLUCB with sampling w.r.t. the number of initially sampled items. Figure~\ref{fig:modelA3} plots the regret after $T=20000$ rounds depending on this number. Again we average over 200 runs. ECT is not very sensitive to the number of sampled items; B-KLUCB is, on the other hand, very sensitive. For ECT, this provides further evidence that the sequential tests applied to items not used in the exploration phase are very efficient. 

\begin{figure}[t]
	\centering
	\includegraphics[width=0.8\columnwidth]{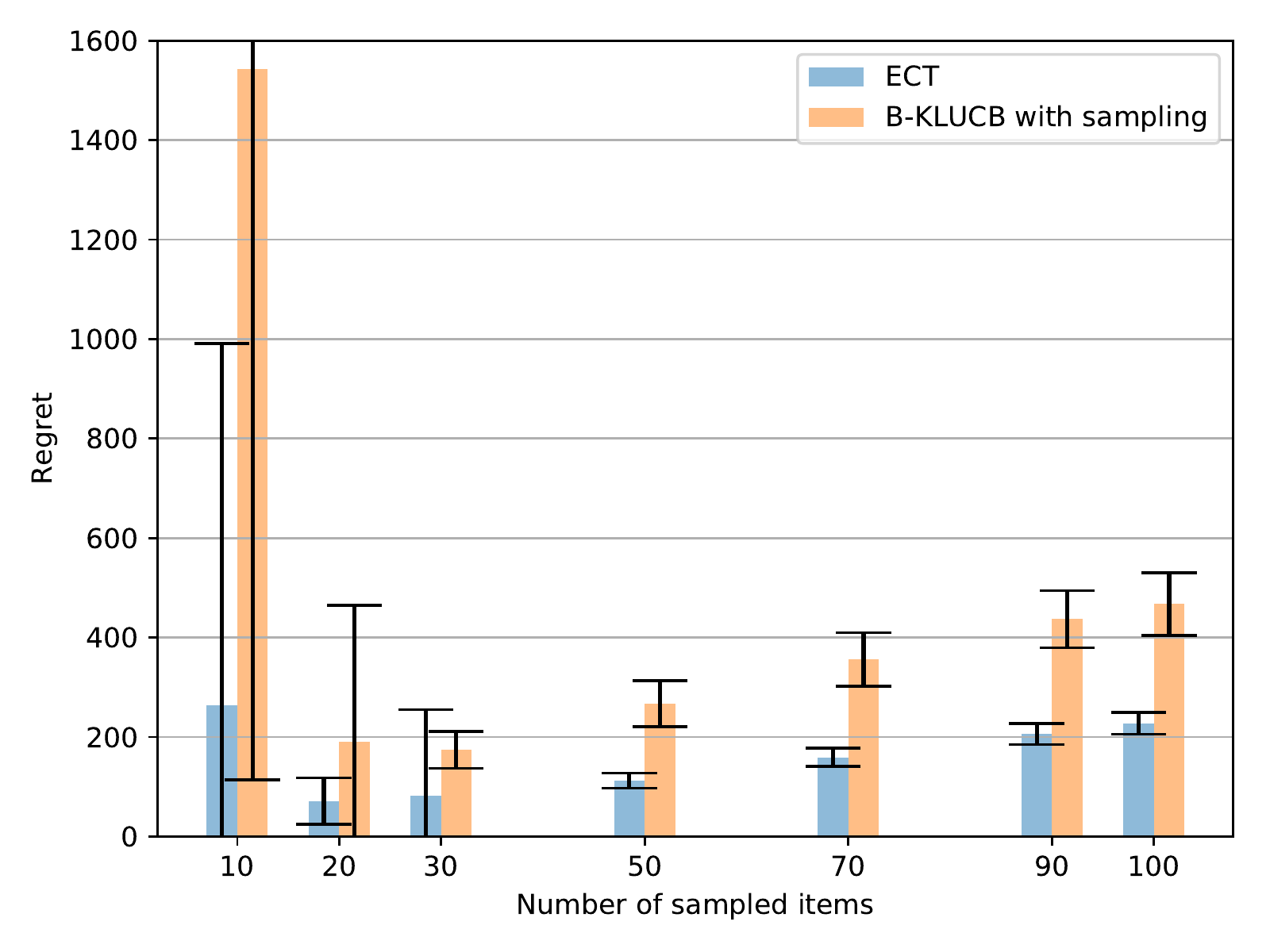} 
	\caption{Regret at $T=20000$ vs number of items initially sampled $|\set{S}|$ for ECT and B-KLUCB with sampling, averaged over $200$ instances. $T = 20000$, $n=3000$, $m=5000$, $p_1 = 0.7$, $p_2 = 0.2$, $\alpha_1=\alpha_2=0.5$. One standard deviations are shown in the error bars. }
	\label{fig:modelA3}
\end{figure}
\clearpage

\newpage

\subsection{Unclustered items and statistically identical users}

\begin{algorithm}[H]
	\caption{Explore-and-Test} \label{alg:ET}
	\begin{algorithmic}
		\STATE {\bf Input:} $T, K$
		\STATE {\bf 1. Exploration}
		\STATE Sample a set $\mathcal{S}$ of $\lfloor \frac{8^2}{\varepsilon^2} \log T\rfloor$ items (uniformly at random)		
		\STATE Recommend each item in $\set{S}$ to $\lfloor 4^2 \log T\rfloor$ users
		\STATE (when this is not possible due to the no-repetition constraint) recommend a random feasible item
		\STATE $\hat{\rho}_i\gets $ the empirical average of $i$ for all $i \in \mathcal{S}$
		\STATE $T_0\gets$ round where the exploration phase ends
		
		\STATE {\bf 2. Test}
		\STATE Compute $\hat{\mu}_{1-\frac{\varepsilon}{2}}$ s.t.  
		$\frac{\varepsilon}{2}|\mathcal{S}|  = \left|\{ i\in \mathcal{S} : \hat{\rho}_i \ge \hat{\mu}_{1-\frac{\varepsilon}{2}} \} \right|$
		\STATE $\mathcal{V} \gets \{ i\in \mathcal{S} : \hat{\rho}_i \ge \hat{\mu}_{1-\frac{\varepsilon}{2}} \}$
		\STATE Reset the reward observation history
		\FOR {$t=(T_0+1)$ {\bfseries to} $T$}
		\STATE Recommend an item $i$ that was recommended the least recently among items in $\set{V}$. If items in ${\cal V}$ cannot be selected, recommend an item $i$, randomly selected from the set of unrecommended items, and add $i$ to $\set{V}$
		\IF {the number of times $i$ has been recommended is exactly $\lfloor 2^{\ell} \log\log_2( 2^e m^2) \rfloor $ for some positive integer $\ell$}
		\STATE Compute $\bar{\rho}^{(\ell)} (\le \hat{\mu}_{1-\frac{\varepsilon}{2}})$ s.t. $\kl (\bar{\rho}^{(\ell)}, \hat{\mu}_{1-\frac{\varepsilon}{2}}) = 2^{-\ell}$ 
		\IF{$  \hat{\rho}_i \le \bar{\rho}^{(\ell)} $}
		\STATE  $\mathcal{V} \gets \mathcal{V} \setminus \{i\}$
		\ENDIF 
		\ENDIF 
		\ENDFOR
	\end{algorithmic}
\end{algorithm}

{\bf Numerical experiments.} Consider a system with $n = 700$ items, and  $m = 1300$ users. The time horizon is $T=100000$. Assume that the distribution $\zeta$ is uniform over the interval $[0.1,0.9]$, and let us target items within the 30\% best items, i.e., $\varepsilon=0.3$. In Figure \ref{fig:modelB}, we compare the satisficing regret averaged over 100 runs of the ET algorithm with $|\set{S}| = 65$ items used in the exploration phase, to that achieved under B-KLUCB with sampling. Since under any algorithm, one needs to use at least $T/m$ items, the number of items sampled under B-KLUCB with sampling is chosen as $\min((T/m) \log T, n)$, which is equal to $n=700$ in our setting. Figure \ref{fig:modelB} illustrates the efficiency of the sequential tests used under ET.

Next, we assess the sensitivity of ET and B-KLUCB w.r.t. the number of initially sampled items. Figure~\ref{fig:modelB2} compares the satisficing regret after $T=100000$ depending on this number. The values are averaged over 20 runs.  ET seems robust to the number of sampled items. B-KLUCB is, however, very sensitive to the number of sampled items and shows larger regret than that of ET.  This result presents further evidence that the sequential testing procedures used in ET are efficient.

\newpage
\begin{figure}[h]
	\centering
	\includegraphics[width=0.8\columnwidth]{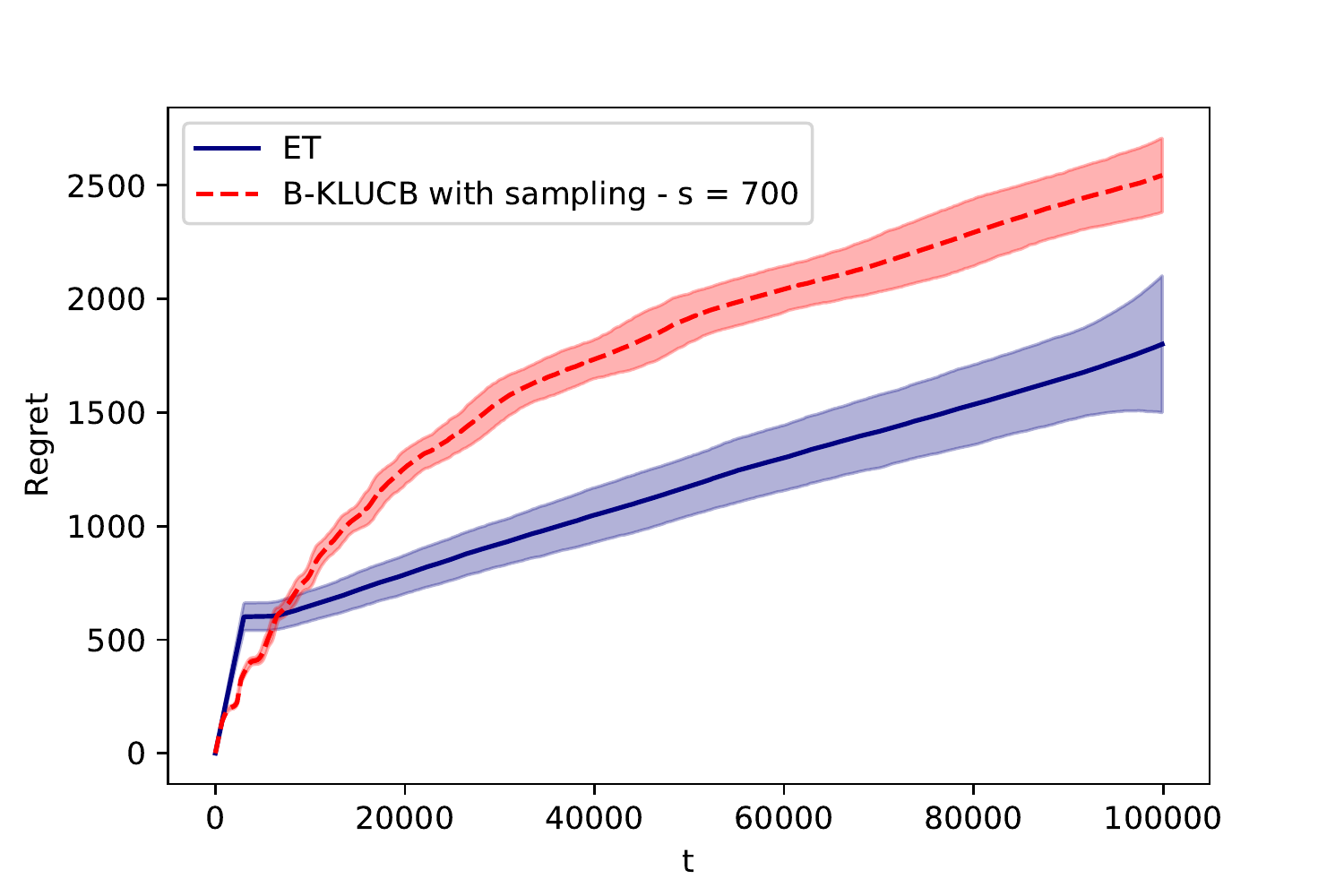} 
	\caption{Satisficing regret vs time averaged over 100 runs of ET and B-KLUCB with sampling.
		$n = 700,  m = 1300, T = 100000,  \varepsilon = 0.3$. $\zeta$ is uniform on $[0.1,0.9]$. Shaded areas correspond to one standard deviations.  }
	\label{fig:modelB}
\end{figure}

\begin{figure}[H]
	\centering
	\includegraphics[width=0.8\columnwidth]{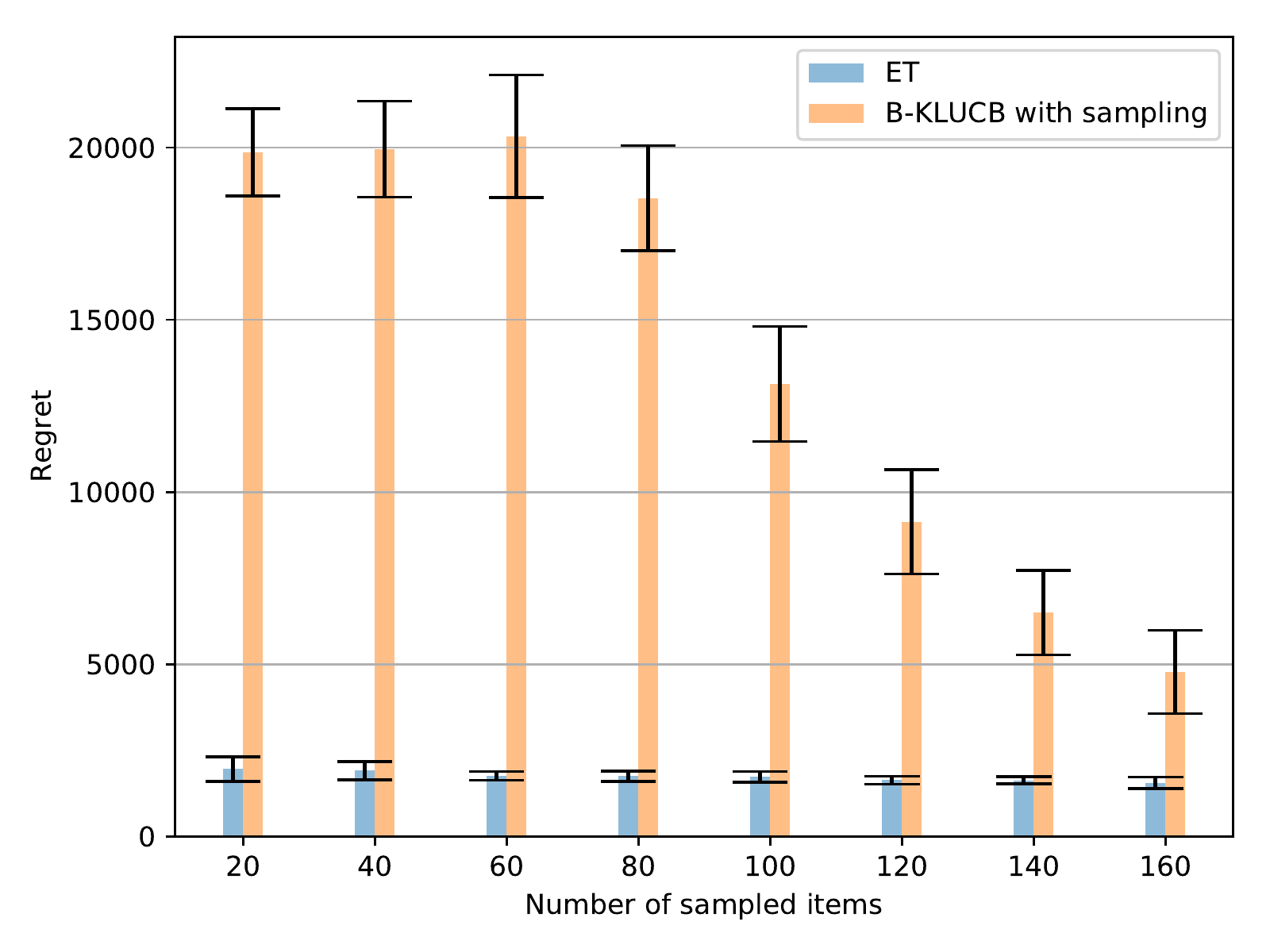} 
	\caption{Satisficing regret at $T=100000$ vs number of items initially sampled $|\set{S}|$ for ET and B-KLUCB with sampling, averaged over 20 instances. $n = 700,  m = 1300, T = 100000,  \varepsilon = 0.3$. $\zeta$ is uniform on $[0.1,0.9]$. One standard deviations are shown in the error bars. }
	\label{fig:modelB2}
\end{figure}

\newpage

\subsection{Clustered items and users}\label{app:algoC}

Here we start by providing the description of our order-optimal algorithm, EC-UCS. We then present ECB (Explore-Cluster-Bandit), a much simpler algorithm but with lower performance guarantees.  
\addtocontents{toc}{\protect\setcounter{tocdepth}{1}}
\subsubsection{The EC-UCS algorithm}
\addtocontents{toc}{\protect\setcounter{tocdepth}{2}}

 We present the pseudo-code of the EC-UCS algorithm in Algorithm~\ref{alg:ECUCS}. The algorithm calls spectral clustering algorithms whose pseudo-codes are provided in Algorithms \ref{alg:SIC_B}-\ref{alg:SP_plus_B}-\ref{alg:SD}.

\begin{algorithm}[H]
	\caption{Explore-and-Cluster-with-Upper-Confidence-Sets (EC-UCS)}\label{alg:ECUCS}
	\begin{algorithmic}
		\STATE Input: $T,K,L$
		\STATE {\bf 1. Exploration for Item Clustering }
		\STATE $n_0 \gets \min\{ n, \lfloor {m}/{(\log T)^2 } \rfloor \} $
		\STATE Sample a set $\set{S}$ of $n_0 $ items.
		\FOR {$1 \le t \le 10m$}
		\STATE Recommend items from $\set{S}$ randomly (when this is not possible due to the no-repetition constraint, recommend a random feasible item). Record the user responses in the dataset $\set{D}$.
		\ENDFOR
		\STATE {\bf 2. Item Clustering }
		\STATE Run Algorithm~\ref{alg:SIC_B} with input $\set{S}, \set{D},  K$ and output $\hat{I}_1, \ldots, \hat{I}_K$.
		\STATE {\bf 3. Exploitation}
		\STATE $\set{U}_0 \gets$ a set of randomly chosen $\lfloor\frac{m}{\log T}\rfloor$ users
		\STATE $\hat{\rho}_{u} = (\hat{\rho}_{uk})_{k \in [K]}\gets $ the empirical average of $u$ for each item cluster $ \hat{I}_k$ for all $u\in\mathcal{U}$
		\FOR {$10m < t \le T$}
		\IF {$t\le \lfloor(10+\log T)m\rfloor$ and $t=(9+2^i)m+1$ for some non-negative integer $i$}
		\STATE $\set{U}^{*}\gets$ a set of $\lfloor (\log T)^2 \rfloor$ users in $\set{U}_0$ who have been observed the most
		\STATE $\epsilon \gets  K\sqrt{\frac{8Km}{t}\log\frac{t}{m}}$
		\STATE $Q_u \gets \{v\in \set{U}^{*}: \|\hat{\rho}_u-\hat{\rho}_v\| \leq  \epsilon\}$ for all $u \in \set{U}^{*}$
		\FOR{$\ell=1$ {\bfseries to} $L$}
		\STATE $u_\ell \gets \arg \max_{v \in {U^{*}}} |Q_v \setminus \cup_{r =1}^{\ell-1} Q_{u_r}|$  and  $\hat{p}_\ell \gets \hat{\rho}_{u_\ell}$
		\STATE $L_0 \gets \ell$
		\IF{$\cup_{r =1}^{\ell} Q_{i_r}=\set{U}^*$}
		\STATE {\bf break}
		\ENDIF
		\ENDFOR
		\ENDIF
		\IF {$u_t \in \set{U}_0$ and $t\le \lfloor(10+\log T)m\rfloor$}
		\STATE Recommend an item in a round-robin fashion from $\hat{I}_1, \ldots, \hat{I}_K$ for each user
		\ELSE
		\STATE $x_{k \ell} \gets \max\{|\hat{p}_{k \ell}-\hat{\rho}_{k u_t}| - \epsilon, 0\}$
		\STATE $\mathcal{L}(u_t) \gets \{\ell\in [L_0]: \sum_{k=1}^K n_{k u_t}x_{k \ell}^2 < 2K \log n_{u_t}\}$
		\IF{$|\mathcal{L}(u_t)| \ge 1$}
		\STATE Recommend an item uniformly at random from $(\hat{I}_{k_\ell^\star})_{\ell \in \set{L}(u_t)}$ 
		\ELSE
		\STATE Recommend an item uniformly at random from $\hat{I}_1, \ldots, \hat{I}_K$
		\ENDIF
		\ENDIF
		\STATE Update $\hat{\rho}_{u_t}$
		\ENDFOR
	\end{algorithmic}
\end{algorithm}

\newpage

\begin{algorithm}[H]
	\caption{Spectral Item Clustering by indirect edges (inspired by Algorithm~1 in \cite{yun2014streaming})}\label{alg:SIC_B}
	\begin{algorithmic}
		\STATE Input: $\mathcal{S}$, $\mathcal{D}$,  $K$
		\STATE $A \leftarrow {\boldsymbol 0} \in \mathbb{R}^{\mathcal{S}\times \mathcal{S}}$, $s \leftarrow 0$
		\FOR{$u\in \set{U}$}
		\IF{$u$ has received recommendations at least two times}
		\STATE $s\leftarrow s+1$
		\IF{user $u$ gives positive responses to both $i$ and $j$ that are the first two recommended items to user $u$}
		\STATE $A_{ij} \leftarrow A_{ij}+1$, $A_{ji} \leftarrow A_{ji}+1$
		\ENDIF
		\ENDIF
		\ENDFOR
		\STATE Run Algorithm~\ref{alg:SP_plus_B} with input $A$, $s$, $K$ and output $\hat{I}_1, \ldots, \hat{I}_K$.
		\STATE Output: $\hat{I}_1, \ldots, \hat{I}_K$
	\end{algorithmic}
\end{algorithm}

\iffalse
\begin{algorithm}[H]
	\caption{Spectral Item Clustering by indirect edges (for ECG, inspired by Algorithm~1 in \cite{yun2014streaming})}\label{alg:SIC}
	\begin{algorithmic}
		\STATE Input: $\mathcal{S}$, $\mathcal{D}$, $T$
		\STATE $A \leftarrow {\boldsymbol 0} \in \mathbb{R}^{\mathcal{S}\times \mathcal{S}}$
		\FOR{$u\in \set{U}$}
		\IF{$u$ has received recommendations at least two times}
		\IF{user $u$ gives positive responses to both $i$ and $j$ that are the first two recommended items to user $u$}
		\STATE $A_{ij} \leftarrow A_{ij}+1$, $A_{ji} \leftarrow A_{ji}+1$
		\ENDIF
		\ENDIF
		\ENDFOR
		\STATE Run Algorithm~\ref{alg:SP_plus} with input $A$, $T$ and output $\hat{I}_1, \hat{I}_2$.
		\STATE Output: $\hat{I}_1, \hat{I}_2$
	\end{algorithmic}
\end{algorithm}
\fi

\begin{algorithm}[h]
	\caption{Spectral Partitioning+  (an improved version of Algorithm~2 in \cite{ok2017collaborative}}\label{alg:SP_plus_B}
	\begin{algorithmic}
		\STATE Input: Observation matrix $A$, $s$, $K$
		\STATE {\bf 1. Spectral Decomposition}
		\STATE Run Algorithm~\ref{alg:SD}, with input $A$, $K$ and output $({S}_k)_{k=1,\ldots,K}$.
		\STATE {\bf 2. Improvement}
		\STATE $\hat{p}(i, j ) \gets \frac{\sum_{v \in {S}_i} \sum_{v' \in {S}_j} A_{v, v'}}{|{S}_i| s}$ for all $1\leq i, j \leq K$
		\STATE ${S}_k^{(0)} \gets {S}_k$ for all $1 \leq k \leq K$
		\FOR {$t = 1$ \textrm{\bf to} $\lfloor \log n_0 \rfloor$}
		\STATE ${S}_k^{(t)} \gets \emptyset$ for all $1\leq k \leq K$
		\FOR {$v \in \mathcal{S}$}
		\STATE $i^* \gets \arg \max_{1\leq i \leq K} $
		\STATE $\left\{\sum_{k=0}^{K} (\sum_{w \in {S}_{k}^{(t-1)}} A_{v w}) \log \hat{p}(i,k)\right\}$ 
		\STATE \quad where $ \sum_{w \in {S}_{0}^{(t-1)}} A_{v w}: = s - \sum_{k=1}^K \sum_{w \in {S}_{k}^{(t-1)}} A_{v w}$
		\STATE and $\hat{p}(i, 0):= 1 - \sum_{k=1}^{K} \hat{p}(i,k)$ (ties are broken uniformly at random)
		\STATE ${S}_{i^*}^{(t)} \gets {S}_{i^*}^{(t)} \cup \{v\}$
		\ENDFOR
		\ENDFOR 
		\STATE $\hat{I}_k \gets {S}_k^{(\log n)}$ for all $1 \leq k \leq K$
		\STATE Output: $\hat{I}_1, \ldots, \hat{I}_K$
	\end{algorithmic}
\end{algorithm}

\newpage

\begin{algorithm}[H]
	\caption{Spectral Decomposition  (Algorithm~3 in \cite{ok2017collaborative})}\label{alg:SD}
	\begin{algorithmic}
		\STATE Input: Observation matrix $A$, $K$
		\STATE $\hat{A} \gets $ rank-$K$ approximation of $A$
		\STATE $\tilde{p} \gets \frac{\sum_{v,w \in \set{I} } A_{vw}}{|\set{S}| (|\set{S}| - 1)}$
		\FOR {$x = 1$ to $\lfloor \log n_0 \rfloor$}
		\STATE $Q_v^{(x)} \gets \left\{ w \in \set{S} : \|\hat{A}_w - \hat{A}_v\|^2 \leq  x \frac{\tilde{p}}{100}\right\}$ for all $v \in \set{S}$
		\STATE $T_l^{(x)} \gets \emptyset$ for all $k \in [K]$
		\FOR {$k=1$ \textrm{\bf to} $K$}
		\STATE $v_k^\star \gets \arg \max_{v \in \set{I}} \left|Q_v^{(x)} \setminus \cup_{i=1}^{k-1} T_i^{(x)} \right|$
		\STATE $T_k^{(x)} \gets Q_{v^\star_k}^{(x)} \setminus \cup_{i=1}^{k-1} T_i^{(x)} $
		\STATE $\xi_k^{(x)} \gets \sum_{v \in T_k^{(x)}} \frac{\hat{A}_v}{|T_k^{(x)}|}$
		\ENDFOR
		\FOR {$ v \in \set{S} \setminus \cup_{k=1}^K T_k^{(x)}$}
		\STATE $k^\star \gets \arg \min_{1\leq k\leq K} \|\hat{A}_v - \xi_k^{(x)}\|^2$
		\STATE $T_{k^\star}^{(x)} \gets T_{k^\star}^{(x)}  \cup \{v\}$ 
		\ENDFOR
		\STATE $r_x \gets \sum_{k=1}^{K} \sum_{v \in T_{k}^{(x)}} \|\hat{A}_v - \xi_k^{(x)}\|^2$
		\ENDFOR
		\STATE $x^\star \gets \arg \min_x r_x$
		\STATE ${S}_k \gets T_k^{(x^\star)}$ for all $k \in [K]$
		\STATE Output: ${S}_1, \ldots, {S}_K$
	\end{algorithmic}
\end{algorithm}

\addtocontents{toc}{\protect\setcounter{tocdepth}{1}}
\subsubsection{The ECB algorithm}
\addtocontents{toc}{\protect\setcounter{tocdepth}{2}}
 
 The ECB algorithm presented in Algorithm \ref{alg:ECB}.  ECB achieves a regret scaling as $O(m\log(\overline{N}))$ for all $p$ (ECB treats each user independently, and does not transfer the information gathered across users). The algorithm proceeds as follows.

(a)-(b) {\bf Exploration  and clustering phases.} (b) These phases are identical to those of EC-UCS. The algorithm outputs item cluster estimates $\hat{I}_1, \dots,\hat{I}_K$. We can show that with an exploration budget of $10m$, the spectral algorithm does not make any clustering errors w.p. at least $1-{1\over T}$. 

(c) {\bf Bandit phase.} The last phase consists in just applying $m$ (one for each user) UCB1 algorithms ~\cite{auer2002finite} with the set of arms $1,\ldots, K$. There, selecting arm $k$ means recommending an item from ${\hat{I}}_k$, accounting for the no-repetition constraint (which is possible w.h.p. since $m^2\ge T(\log T)^3$). 

ECB calls the clustering algorithm presented in Algorithm~\ref{alg:SIC_B}, that first constructs an item adjacency matrix (using indirect edges from users), and then applies the spectral clustering algorithm, Algorithm~\ref{alg:SP_plus_B}, to output $K$ item clusters. Note that Algorithm~\ref{alg:SP_plus_B} further calls the spectral decomposition algorithm, shown in Algorithm~\ref{alg:SD}. 

We have the following performance guarantee on ECB (the proof is presented in Appendix~\ref{subsec:ECBproof}):
\begin{theorem}\label{th:Cgen} When $m^2\ge T(\log T)^3$, the regret of ECB satisfies:
	$$
	R^\mathrm{ECB} (T) = \set{O}\left(  \sum_{\ell=1}^L \beta_\ell m\sum_{k\neq k^\star_\ell} \frac{  \log ( \overline{N} ) }{p_{k^\star_\ell \ell}-p_{k\ell}} \right).
	$$
\end{theorem}

\newpage

\begin{algorithm}[H]
	\caption{Explore-Cluster-Bandit}\label{alg:ECB}
	\begin{algorithmic}
		\STATE Input: $T,K$
		\STATE {\bf 1. Exploration for Item Clustering}
		\STATE $n_0 \gets \min\{ n, \lfloor {m}/{(\log T)^2 } \rfloor \} $
		\STATE Sample a set $\set{S}$ of $n_0 $ items.
		\FOR {$1 \le t \le 10m$}
		\STATE Recommend items from $\set{S}$ randomly (when this is not possible due to the no-repetition constraint, recommend a random feasible item). Record the user responses in the datatset $\set{D}$.
		\ENDFOR
		\STATE {\bf 2. Item Clustering}
		\STATE Run Algorithm~\ref{alg:SIC_B} with input $\set{S}, \set{D},  K$ and output $\hat{I}_1, \ldots, \hat{I}_K$.
		\STATE {\bf 3. Run UCB1} \cite{auer2002finite} {\bf for each user}
		\FOR {$10m < t \le T$}
			\STATE $k^\star \gets \arg \max_{k \in [K]} \hat{\rho}_{k}^{u_t} + \sqrt{\frac{2 \ln (t)}{N^{u_t}_k(t)}}$, where $ \hat{\rho}_{k}^{u_t}$ is the empirical average reward of user $u_t$ for items in $\hat{I}_k$ and $N^{u_t}_k(t)$ is the total number of samples of user  $u_t$ for items in $\hat{I}_k$.
			\STATE Recommend an item from $\hat{I}_{k^\star}$ randomly (when this is not possible due to the no-repetition constraint) recommend a random feasible item.
		\ENDFOR
	\end{algorithmic}
\end{algorithm}

\addtocontents{toc}{\protect\setcounter{tocdepth}{1}}
\subsubsection{Numerical Experiment} 
\addtocontents{toc}{\protect\setcounter{tocdepth}{2}}
Consider a system with $n = 2000$ items and $m = 5000 $ users.
The time horizon is $T = 800000$. The statistical parameters of $p_{k \ell}$ are given in  Table~\ref{tb:modelC_experiment}. $ \alpha_1 = \alpha_2 = \frac{1}{2}$ and $ \beta_1 = \beta_2 = \frac{1}{2}$. With this parameter setting, for each $\ell =1, 2$, $\set{L}^\perp(\ell) = \emptyset$. From Theorems~\ref{thm:newalgC} and \ref{th:Cgen}, we know that the regret of EC-UCS is $R^{\mathrm{EC-UCS}}(T) = \set{O}(m)$ whereas that of ECB is $ R^{\mathrm{ECB}}(T) = \set{O}(m \log (\overline{N}))$. Hence, we expect that EC-UCS to outperform ECB. Figure~\ref{fig:modelC1} shows the regret evolution over time of EC-UCS algorithm and ECB algorithm after the item clustering phase. The curves are averaged over 10 instances. The rate at which the regret of EC-UCS increases is rapidly decreasing. This is not the case for that of the regret of ECB.

Next, we assess the regret of the two algorithms after $T$ rounds as a function of $T$. We consider a system with $n = 3000$ items and $m = 5000 $ users. $ \alpha_1 = \alpha_2 = \frac{1}{2}$ and $ \beta_1 = \beta_2 = \frac{1}{2}$. The statistical parameters of $p_{k \ell}$ are the same as in the previous system. Figure~\ref{fig:modelC2} presents the results. Here, the regrets are averages over 20 runs. The regret of ECB clearly increases with  $T$, while the regret of EC-UCS does not seem to be sensitive to $T$. Overall, our results confirm our theoretical results, at least on simple examples. 

\begin{table}[H]
	\begin{center}
		\begin{tabular}{|c|c|c|}
			\hline
			& $k=1$ & $k=2$ \\
			\hline
			$\ell =1$ & $0.2$ & $0.8$ \\
			\hline
			$\ell = 2$ &  $0.8$ & $0.2$ \\
			\hline
		\end{tabular}
	\end{center}
	\caption{The values of $(p_{k \ell})$.}
	\label{tb:modelC_experiment}
\end{table}

\newpage
\begin{figure}[H]
	\centering
	\includegraphics[width=0.8\columnwidth]{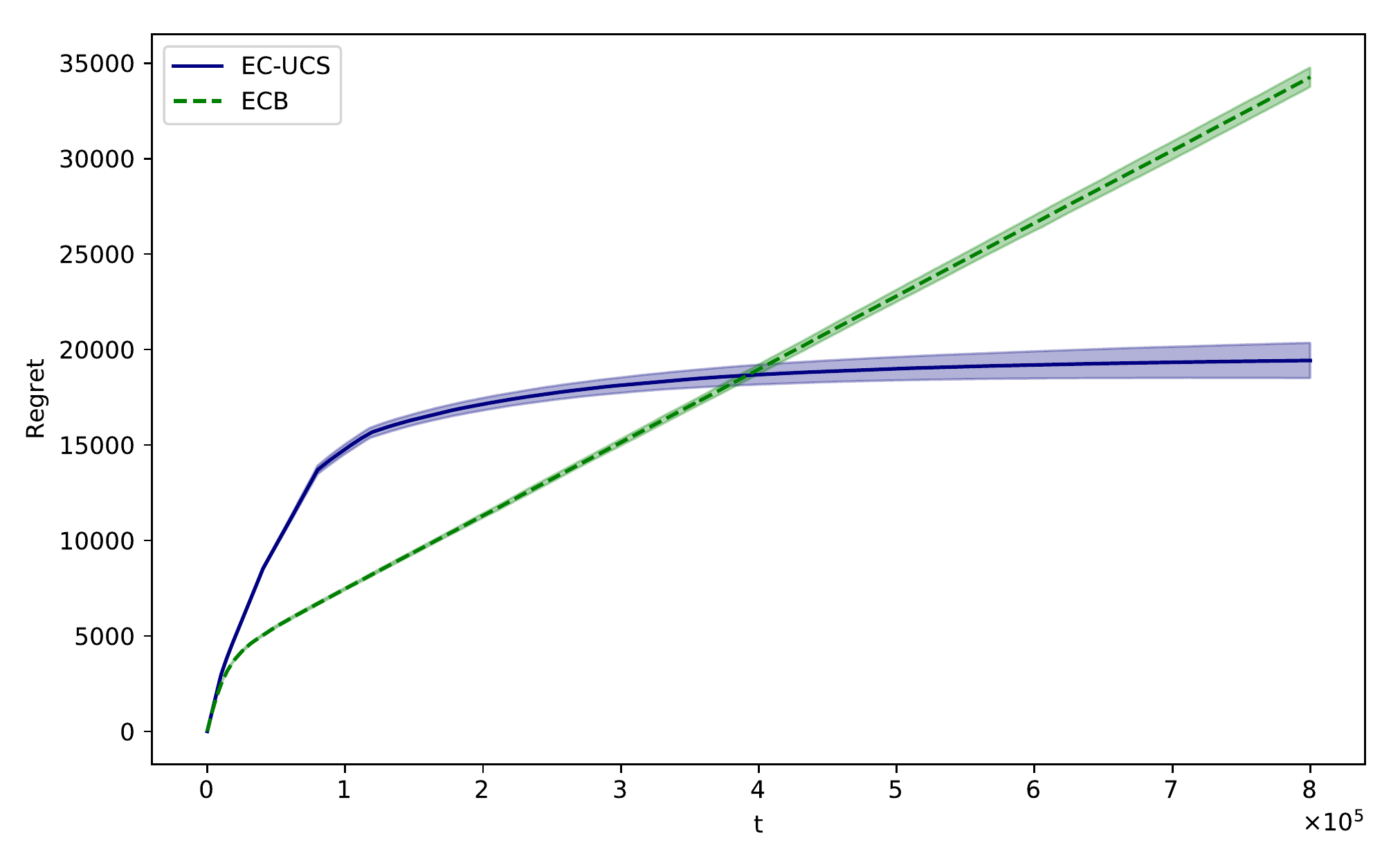} 
	\caption{Regret vs time after the item clustering phase, averaged over 10 runs of EC-UCS and ECB. $T=800000$,  $n = 2000$,  $m=5000$, $\alpha_1 = \alpha_2 = 0.5$, $ \beta_1 = \beta_2 = 0.5$. One standard deviations are shown as the shaded areas.}
	\label{fig:modelC1}
\end{figure}

\begin{figure}[H]
	\centering
	\includegraphics[width=0.8\columnwidth]{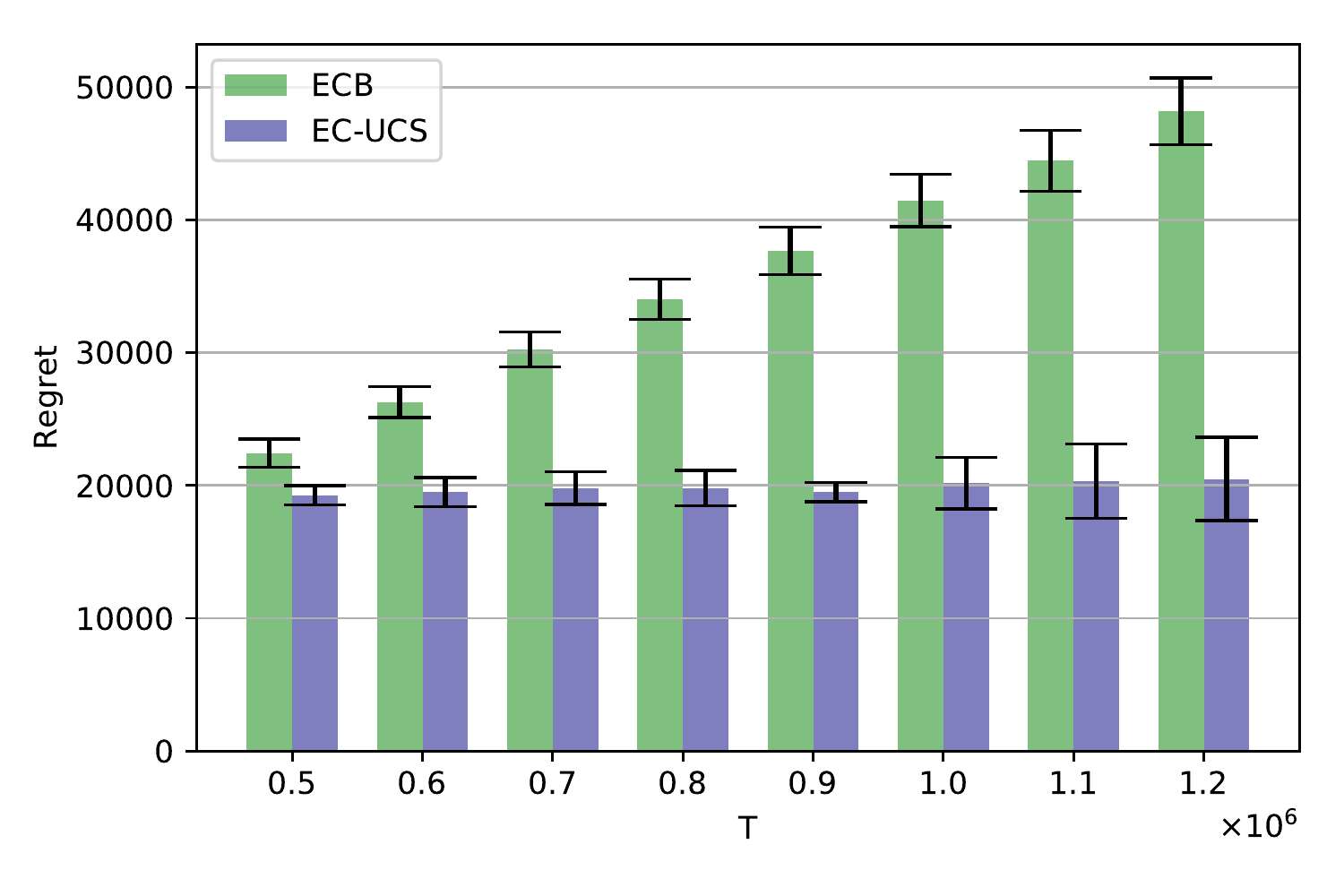} 
	\caption{Regret at round $T$ vs budget $T$ after the item clustering phase, averaged over 20 runs of EC-UCS and ECB.  $n = 3000$,  $m=5000$, $\alpha_1 = \alpha_2 = 0.5$, $ \beta_1 = \beta_2 = 0.5$. One standard deviations are shown as the shaded areas.}
	\label{fig:modelC2}
\end{figure}

\addtocontents{toc}{\protect\setcounter{tocdepth}{1}}
\subsection{Experimental set-up}
\addtocontents{toc}{\protect\setcounter{tocdepth}{2}}

The simulations were performed on a desktop computer with Intel Core i7-8700B 3.2 GHz CPU and 32 GB RAM.

\clearpage

\newpage
\section{Preliminaries: Properties of the user arrival process}\label{app:preliminaries}
This section presents several preliminary results on the user arrival process, extensively used throughout the proofs of the main theorems. Here we also provide the proof of Lemma~\ref{lem:couponcollector}.

\begin{lemma}[Chernoff-Hoeffding theorem] Let $X_1,\dots,X_n$ be i.i.d. Bernoulli random variables with mean $\nu$. Then, for any $\delta >0$,
	\begin{align*}
		\mathbb{P} \left( \frac{1}{n} \sum_{i=1}^n X_i  \ge \nu + \delta \right) \le & \exp \left( - n \kl (\nu + \delta , \nu) \right) \cr
		\mathbb{P} \left( \frac{1}{n} \sum_{i=1}^n X_i  \le \nu - \delta \right) \le & \exp \left( - n \kl (\nu - \delta , \nu) \right) 
	\end{align*}
	\label{lem:cher}
\end{lemma}

\begin{lemma}[Pinsker's inequality \cite{tsybakov2008introduction}] For any $0 \le p,q  \le 1$, $2(p-q)^2 \le \kl (p,q).$
\end{lemma}

\begin{lemma}
For any $0 \le p,q  \le 1$,	$\kl (p,q) \ge p \log \frac{p}{q} + (q-p) $. \label{lem:kllower}
\end{lemma}
{\bf Proof.} This follows from $(1-p) \log \frac{1-p}{1-q} \ge q-p$. \ep

{\bf Proof of Lemma \ref{lem:couponcollector}.} This lemma is quoted below for convenience:

{\bf Lemma~\ref*{lem:couponcollector} (restated).}
	{\it For every user $u\in \mathcal{U}$,  we have}
	\begin{align*}
		\EXP\Big[\max\Big\{0,N_u (T) - \overline{N}\Big\}\Big]\leq \frac{1}{m(e-1)}
	\end{align*}
{\it where}
$$\overline{N} = \frac{4\log m}{\log\left(\frac{m\log m}{T}+e\right)}+\frac{e^2 T}{m}.$$

\begin{proof}
	Since $u$ arrives with probability $\frac{1}{m}$ in each round, the probability that $u$ arrives for more than $\overline{N}+x$ times in the $T$ first round is,
	\begin{align*}
		\Pr\left(N_u (T)\geq \overline{N} +x\right)
		&\stackrel{(a)}{\leq} \exp\left(-T \kl \left( \frac{\overline{N} +x}{T}  ,\frac{1}{m}\right)\right)\cr
		&\stackrel{(b)}\leq \exp\left(-(\overline{N}+x) \log \left(\frac{m(\overline{N}+x)}{T} \right) -\frac{T}{m} +(\overline{N}+x)   \right)\cr
		&\stackrel{(c)}\leq \exp\left(-\frac{\overline{N}+x}{2} \log \left(\frac{m(\overline{N}+x)}{T} \right)  \right) \cr
		&\stackrel{(d)}\leq \exp\left(-\frac{\overline{N}}{2} \log \left(\frac{m\overline{N}}{T} \right) \right) \exp(-x)
	\end{align*}
	where (a) follows from Lemma~\ref{lem:cher}, (b) from Lemma~\ref{lem:kllower}, (c) is obtained from the fact that $\frac{m\overline{N}}{T} \ge e^2$, and (d) holds since $x\ge 0$ and $\log \left(\frac{m\overline{N}}{T} \right) \ge 2$. We deduce that:
	\begin{align}
		\EXP\Big[\max\Big\{0,N_u (T) - \overline{N}\Big\}\Big] 
		&= \sum_{x=1}^{\infty} \Pr\left(N_u (T) \geq \overline{N}+x\right)\cr
		&\leq  \sum_{x=1}^{\infty}\exp\left(-\frac{\overline{N}}{2} \log \left(\frac{m\overline{N}}{T} \right) \right) \exp(-x) \cr
		&= \frac{\exp\left(-\frac{\overline{N}}{2} \log \left(\frac{m\overline{N}}{T} \right) \right)}{e-1} . \label{eq:ENu1}
	\end{align}
	
	To conclude the proof, we compute an upper bound of (\ref{eq:ENu1}) for two cases: $\frac{T}{m} \ge \frac{\log m}{ e}$ and $\frac{T}{m} \le \frac{\log m}{ e}$.
	
	If $\frac{T}{m} \ge \frac{\log m}{ e}$, then $\frac{\overline{N}}{2} > \log m$ and $\log \left(\frac{m\overline{N}}{T} \right) \ge 2$. Thus, 
	\begin{equation}
		\exp\left(-\frac{\overline{N}}{2} \log \left(\frac{m\overline{N}}{T} \right) \right) < \frac{1}{m^2}. \label{eq:case1}
	\end{equation}
	
	We now consider the case where $\frac{T}{m} \le \frac{\log m}{ e}$.
	When we define $f(x) = \log \left(\frac{4x}{ \log(x+e)}\right) - \frac{\log(x+e)}{2} $, one can easily check that  $f(e) \ge 0$ and $f'(x) \ge 0$ for all $x \ge e$. Therefore, since $\frac{m\log m}{T} \ge e$, we can deduce that
	$$\frac{2\log m}{\log\left(\frac{m\log m}{T}+e\right)} \log \left(\frac{4m\log m}{T\log \left(\frac{m\log m}{T}+e\right)}\right) \ge \log m $$
	which directly implies that
	\begin{equation}
		\exp\left(-\frac{\overline{N}}{2} \log \left(\frac{m\overline{N}}{T} \right) \right) < \frac{1}{m}. \label{eq:case2}
	\end{equation}

Lemma~\ref{lem:couponcollector} is obtained by combining \eqref{eq:ENu1}, \eqref{eq:case1} and \eqref{eq:case2}.
\end{proof}

The following Lemma characterizes the lower tail of the number of user arrivals.
\begin{lemma}\label{lem:lower_tail}
	For every user $u \in \set{U}$, we have 
	\begin{align*}
			\Pr\left(N_u(T) \le \frac{T}{2m}\right) \le \exp \left(  - \frac{T}{2m}\left( 1 - \log(2) \right)\right).
	\end{align*}
\end{lemma}
{\bf Proof.}	
\begin{align*}
		 \Pr\left(N_u(T) \le \frac{T}{2m}\right) 
		& \stackrel{(a)}{\le} \exp\left( -T \kl \left(\frac{1}{2 m}, \frac{1}{m}\right)\right)
		\\
		& \stackrel{(b)}{\le} \exp\left( - T\left( \frac{1}{2 m} \log(1/2) + \left(\frac{1}{m} - \frac{1}{2m}\right) \right)\right)
		\\
		& = \exp\left( - \frac{T}{2m}\left( 1 - \log(2) \right)\right),
	\end{align*}
	where $(a)$ is from Lemma~\ref{lem:cher} and $(b)$ is from Lemma~\ref{lem:kllower}. \ep

The next lemma is instrumental in the performance analysis of the EC-UCS and ECB algorithms (for systems with clustered items and users).  

\begin{lemma} With probability $1-\frac{1}{T}$, at least $\frac{m}{2}$ users arrive at least two times among the first $10m$ arrivals. \label{lem:m2}
\end{lemma}
{\bf Proof.} We denote by $N_u(10m)$ the number of times user $u$ arrives in the $10m$ first arrivals. For any set $A\subset {\cal U}$, let $N(A,10m)$ denote the total number of arrivals of users in $A$ among the first $10m$ arrivals.

We write the probability that less than $m/2$ users arrive twice in the $10m$ first arrivals as:
\begin{align}
\mathbb{P}[\sum_u \indicator_{\{ N_u(10m)\ge 2\} }< m/2] 
& = \mathbb{P}[\sum_u\indicator_{\{ N_u(10m)< 2\} }\ge m/2] \nonumber\\
&\le \mathbb{P}[\exists A: |A|={m\over 2}, \forall u\in A, N_u(10m)< 2] \nonumber\\
&\le \mathbb{P}[\exists A: |A|={m\over 2}, N(A,10m)\le {m\over 2}] \nonumber\\
&\le {m \choose m/2} \exp\left( - 20m (\frac{1}{2}-\frac{1}{10} )^2 \right) \nonumber \\
&\le (2e)^{m/2} \exp\left( - \frac{9}{5}m  \right)  \nonumber \\
&\le e^{- 0.8 m} \le \frac{1}{T^2}. \label{eq:firstm}
\end{align}
\ep

\newpage
\section{Justifying the regret definitions}\label{app:regretdef}

In this section, we justify our definitions of regret for Models A, B and C. In these models, we define regret as if an Oracle policy would always be able to select for any user $u$ an item from the best cluster for Models A and C for this user, or an $\varepsilon$-best item for Model B, even under the no-repetition constraint. In fact, the definition of the true regret should account for the no-repetition constraint and in turn for the fact that an Oracle policy may be obliged to select items that do not belong to the best cluster for the user $u$, because user $u$ may arrive too often before the time horizon $T$ or because the size of this cluster is too small (remember that when an item, selected for the first time, is randomly assigned to a cluster). We prove here that the difference between our notion of regret and the true regret is actually negligible when compared to any of the terms involved in our regret lower bounds, namely $T/m$ and $\bar{n}$.  

To establish this claim, recall the assumptions made on $(n,m,T)$. $T=o(nm)$, $m\ge n$, and $\log(m)=o(n)$ (and as a consequence, for any $\beta>0$ independent of $(n,m,T)$, $m=o(e^{\beta n})$).

For illustrative purposes, we prove our claim for Model A (the same result holds for Models B and C). Let $Z_u$ denote the (random) number of items in the best cluster for user $u$. It is easy to show that the difference between our notion of regret $R^\pi(T)$ and the true regret $R_{true}^\pi(T)$ satisfies:
$$
| R^\pi(T) - R_{true}^\pi(T)|\le \sum_{u\in {\cal U}} \mathbb{E}[\max\{ N_u(T) - Z_u,0\}].
$$
In Model A, $Z_u$ is the size of cluster ${\cal I}_1$. The average size of ${\cal I}_1$ is $\mathbb{E}[Z_u]=\alpha_1 n$. Let $\epsilon>0$ be such that $\epsilon < \alpha_1$ and ${\epsilon nm\over T}\ge e^2$. Using the same arguments as those leading to Lemma \ref{lem:couponcollector}, we obtain the following concentration result:
\begin{align*}
		\EXP\Big[\max\Big\{0,N_u (T) - \epsilon n\Big\}\Big]\leq \frac{1}{(e-1)}\exp(-{\epsilon n\over 2}\log({\epsilon nm\over T})).
\end{align*}
In addition, we also have as a direct application of Chernoff-Hoeffding inequality presented in Lemma \ref{lem:cher}:
$$
\mathbb{P}[Z_u\le \epsilon n]\le \exp(-n\kl (\epsilon, \alpha_1)).
$$
From the two above inequalities, we deduce that:
\begin{align*}
\mathbb{E}[\max\{ N_u(T) - Z_u,0\}] &\le {T\over m}\mathbb{P}[Z_u\le \epsilon n] + \EXP\Big[\max\Big\{0,N_u (T) - \epsilon n\Big\}\Big]\\
&\le {T\over m}\exp(-n\kl (\epsilon, \alpha_1)) + \exp(-{\epsilon n\over 2}\log({\epsilon nm\over T})).
\end{align*}
We conclude that:
$$
| R^\pi(T) - R_{true}^\pi(T)|\le \underbrace{T\exp(-n\kl (\epsilon, \alpha_1))}_{= A(n,m,T)} + \underbrace{m\exp(-{\epsilon n\over 2}\log({\epsilon nm\over T}))}_{=B(n,m,T)}.
$$
Next we verify that $\max\{ A(n,m,T), B(n,m,T)\}= o(\min\{ {T\over m},\bar{n}\})$. This will be enough to justify our definition of regret, since our regret lower bounds are all larger than $\min\{ {T\over m},\bar{n}\}$.

(i) Let us check that $\max\{ A(n,m,T), B(n,m,T)\}= o({T\over m})$. Indeed, for $A(n,m,T)$, we have $m=o(e^{n\kl (\epsilon, \alpha_1)})$; as for $B(n,m,T)$, we have
$m^2/T = o(e^{\epsilon n})$ which results from $m = o(e^{\epsilon n/2})$.

(ii) Let us check that $\max\{ A(n,m,T), B(n,m,T)\}= o(\bar{n})$. We consider the three regimes for $\bar{n}$:
\begin{itemize}
\item[1.] When $T=\omega(m\log(m))$. Then $\bar{n}=T/m(1+o(1))$ and we conclude as in (i).
\item[2.] When $T=\Theta(m\log(m))$. To simplify, we just prove the statement for $T=c m\log(m)$. Then $\bar{n}=\log(m)(d_c+o(1))$ and $A(n,m,T)= m\log(m) \exp(-n\kl (\epsilon, \alpha_1))$. We hence conclude that  $A(n,m,T)=\bar{n}$ since $m=o(e^{n\kl (\epsilon, \alpha_1)})$. Now $B(n,m,T)=m\exp(-{\epsilon n\over 2}\log({\epsilon n\over c\log(m)}))$, and  $B(n,m,T)=o(\bar{n})$ is a consequence of $m=o(e^{\epsilon n\over 2})$. 
\item[3.] When $T=o(m\log(m))$. Then $\bar{n}={\log(m)\over \log(m\log(m)/T)}(1+o(1))$. $A(n,m,T)=o(\bar{n})$ is equivalent to $X=o(e^{n\kl(\epsilon,\alpha_1)})$ where $X={T\over \log(m)}\log({m\log(m)\over T})$. Now $X=o(m)$ since $T=o(m\log(m))$, and thus $A(n,m,T)=o(\bar{n})$. Finally, $B(n,m,T)=o(\bar{n})$ is equivalent to $Y=o(\exp({\epsilon n\over 2}\log(\epsilon nm/T)))$ where $Y={m\log(m\log(m)/T)\over \log(m)}$. Since $T=o(nm)$, this would be implied by $Y=o(\exp({\epsilon n\over 2}))$. However, since $T\ge 1$, $Y\le {m\log(m\log(m))\over \log(m)}=m(1+o(1))$. We conclude that  $B(n,m,T)=o(\bar{n})$ since $m=o(e^{\epsilon n\over 2})$.
\end{itemize}

\newpage
\section{Fundamental limits for Model A: Proof of Theorem \ref{th:lowA}}\label{app:thlowA}

\underline{Proof of $R^\pi(T) \ge R_{\mathrm{ic}}(T)$:} Let $\pi\in \Pi$. Assume that the success probabilities $p_k$'s are known. Further simplify the problem by relaxing the no-repetition constraint: instead, we just impose that any item cannot be selected more than $m$ times. The algorithm $\pi$ has an expected regret larger than an optimal algorithm for the problem where $p_k$'s are known, and the no-repetition constraint is relaxed as explained above. Denote by $\tau$ this optimal algorithm. Next we establish a regret lower bound for $\tau$. To this aim, we first consider that the cluster ids of the items are drawn before the first round. That way, all the $n$ items belong to a cluster even before the first round. Let $J_1,\ldots,J_K$ denote the sizes of the various clusters of items. Hence, we work conditioning on the cluster ids of the items. Now, define $\mathbb{E}^\tau[N_k]$ as the expected number of times an item of cluster ${\cal I}_k$ is selected under $\tau$  (of course, $N_k\le m$) until round $T$. More precisely, denote by $(n_1,n_2,\ldots,n_{J_k})$ the random variables representing the numbers of times the first, the second, the third, etc. items in ${\cal I}_k$ are selected under $\tau$. Then by definition: $\mathbb{E}^\tau[N_k] = \mathbb{E}^\tau[ \sum_{i=1}^{J_k}n_i]/J_k.$ We prove:
\begin{lemma}\label{lem:2arms} We have for any $m\ge c/\Delta_2^2$ and for any $k\neq 1$:      
	$
	{\phi(k,m,p)\over m}\le {\mathbb{E}^\tau[N_k]\over \mathbb{E}^\tau[N_{1}]} \le 1.
	$
\end{lemma}

The constant $c$ in the above lemma is the same as that in Theorem \ref{th:lowA}. It is chosen such that if $m\ge c/\Delta_2^2$, then, for any $k$, $\phi(k,m,p)\le m$ and thus, the first inequality of Lemma \ref{lem:2arms} makes sense. Remember that $\phi(k,m,p)$ scales as $1/\Delta_k^2$ (refer to the remark above Theorem \ref{th:lowA}), and hence such a choice for $c$ is possible. 

Lemma \ref{lem:2arms} is proved at the end of this section. Assume for now that it holds. We complete the proof by deriving a lower bound of the optimal algorithm $\tau$, $R^\tau(T)$. When the item clusters are fixed, the expected conditional regret of $\tau$ is:
$
R^\tau(T|{\cal I}_1,\ldots,{\cal I}_K)=\sum_{k\neq 1}J_k \mathbb{E}^\tau[N_k]\Delta_k.
$
Hence we have:
\begin{align*}
R^\tau(T|{\cal I}_1,\ldots,{\cal I}_K) & = T {\sum_{k\neq 1}J_k \mathbb{E}^\tau[N_k]\Delta_k\over \sum_{k}J_k \mathbb{E}^\tau[N_k]} \stackrel{(a)}{\ge} {T\over m} {\sum_{k\neq 1} J_k \mathbb{E}^\tau[N_1]\phi(k,m,p)\Delta_k \over \sum_{k} J_k \mathbb{E}^\tau[N_1]}\\
& = {T\over m} \sum_{k\neq 1}{J_k\over n} \phi(k,m,p)\Delta_k,
\end {align*}
where (a) stems from Lemma \ref{lem:2arms} (in the numerator, we use $\mathbb{E}^\tau[N_k]\ge \mathbb{E}^\tau[N_1]{\phi(k,m,p)\over m}$, and in the denominator $\mathbb{E}^\tau[N_k]\le \mathbb{E}^\tau[N_1]$). Taking the expectation of the above inequality (noting that $\mathbb{E}[J_k/ n]=\alpha_k$), we conclude that:
$
R^\pi(T)\ge R^\tau(T) \ge {T\over m} \sum_{k\neq 1}{\alpha_k} \phi(k,m,p)\Delta_k. 
$

\underline{Proof of $R^\pi(T) \ge R_{\mathrm{nr}}(T)$:}  Observe that under the no-repetition constraint, the number of items that $\pi$ will select is greater than $\max_uN_u(T)$. Now when an item is selected for the first time, by assumption, it belongs to the sub-optimal cluster ${\cal I}_k$ with probability $\alpha_k$, in which case this initial selection induces an expected regret $\Delta_k=p_1-p_k$. Hence $R^\pi(T) \ge \mathbb{E}[\max_uN_u(T)]  \sum_{k\neq 1}\alpha_k\Delta_k=\overline{n}\sum_{k\neq 1}\alpha_k\Delta_k$. \ep

\underline{Proof of $R^\pi(T) \gtrsim R_{\mathrm{sp}}(T)$:} 
To prove this asymptotic lower bound, we consider a simpler problem: the algorithm knows the item clusters $(\set{I}_k)_{k \in [K]}$. Then, the problem reduces to a $K$-armed Bernoulli bandit problem with unknown parameters $(p_k)_{k \in [K]}$. 

The proof then proceeds using a classical change-of-measure argument as in \cite{lai1985asymptotically}.  We present this argument for completeness. Assume that the algorithm $\pi$ is uniformly good. Pick any $k\in [K]$ s.t. $k \neq 1$. 
Let $p=(p_k)_{k \in [K]}$ be original parameters  and let $(p'_{k})_{k \in [K]}$ be perturbed parameters where $\forall k' \neq k, \; p'_{k'} = p_{k'}$ and $p'_k = p_1 + \varepsilon$ with some constant $\varepsilon>0$.  We denote $N_k(T) \coloneqq \sum_{t=1}^T \indicator \left\{i_t^\pi \in \set{I}_k\right\}$. We use $\EXP_p[\cdot]$ (or $\EXP[\cdot]$) and $\EXP_{p'}[\cdot]$ to denote the expectation under the original model and under the perturbed model, respectively.

Let $L_T$  be the log-likelihood defined as:
\begin{align*}
	L_T \coloneqq \sum_{t =1}^T \indicator \left\{ i_t^\pi \in \set{I}_k\right\} \left(\indicator\left\{X_{i_t^\pi u_t} = 1\right\} \log \left(\frac{p_k}{p_1+\varepsilon}\right) + \indicator\left\{X_{i_t^\pi u_t} = 0\right\} \log \left(\frac{1 - p_k}{1  - (p_1+\varepsilon)}\right)\right).
\end{align*}

Taking the expectation under $p$, we have 
\begin{align*}
	\EXP_{p} [L_T] = \EXP_{p}[N_k(T)] \kl (p_k, p_1 + \varepsilon)
	& \stackrel{(a)}{\ge} \kl \left(\frac{\EXP_p[N_k(T)]}{T}, \frac{\EXP_{p'} [N_k(T)]}{T}\right)
	\\
	& \stackrel{(b)}{\ge} \left(1 - \frac{\EXP_p[N_k(T)]}{T} \right)\log \left(\frac{T}{T - \EXP_{p'} [N_k(T)]}\right) - \log 2.
\end{align*}
where $(a)$ stems from the {\it data processing inequality}, see \cite{garivier2018explore}, and $(b)$ is from the fact that for all $(x, y) \in [0, 1]^2$, $\kl(x, y ) \ge (1 - x) \log \frac{1}{1 - y} - \log 2$. By the uniform goodness assumption, with some constant $C>0$, we have:
\begin{align*}
	\EXP_p[N_k(T)] \lesssim C \sqrt{T} \quad \text{and} \quad T - \EXP_{p'} [N_k(T)] \lesssim C \sqrt{T}.
\end{align*}
Hence:
\begin{align*}
	 \frac{1}{\log(T)} \log\frac{T}{T - \EXP_{p'} [N_k(T)]} \gtrsim \frac{1}{\log (T)} \log \frac{T}{C\sqrt{T}} \gtrsim \frac{1}{2}.
\end{align*}
The inequality $\EXP_{p}[N_k(T)] \kl (p_k, p_1 + \varepsilon) \gtrsim \frac{1}{2}$ holds for any $\varepsilon>0$. Therefore, we have:
\begin{align*}
	\EXP_{p}[N_k(T)] \kl (p_k, p_1) \gtrsim \frac{1}{2}.
\end{align*}
Thus, we get the regret lower bound:
\begin{align*}
	R^\pi(T)  = \sum_{k \neq 1}  \Delta_k \EXP_p[N_k(T)]  \gtrsim \left(\sum_{k \neq 1} \frac{\Delta_k}{2 \kl(p_k, p_1)}\right) \log (T).
\end{align*}
This concludes the proof of Theorem~\ref{th:lowA}. \ep

\subsection{Proof of Lemma \ref{lem:2arms}}

To establish the lemma, we build, from the optimal algorithm $\tau$, an algorithm $\zeta$ that can be applied to a 2-armed bandit problem with known expected rewards $p_1$ and $p_k$. We then provide a connection between the regret of $\zeta$ in the 2-armed bandit problem and  $\mathbb{E}^\tau[N_k]$. We conclude the proof by establishing a regret lower bound for $\zeta$, using similar techniques as in \cite{Bubeck13}.  

{\bf 2-armed bandit problem with known rewards and the algorithm $\zeta$.} Consider a 2-armed bandit problem with Bernoulli arms $1$ and $k$ of means $p_1$ and $p_k$. The means are known but the arm with the highest mean is unknown. That is to say that the expected reward of arm $1$ can be either $p_1$ and $p_k$. For this bandit problem, we build an algorithm, denoted by ${\zeta}$, based on the algorithm $\tau$. 
\begin{enumerate}
	\item Pick $i_1$ and $i_k$ uniformly at random in the clusters ${\cal I}_1$ and ${\cal I}_k$. We run $\tau$ for $T$ rounds. When $\tau$ selects item $i_1$ (resp. $i_k$), then $\zeta$ also selects arm $1$ (resp. $k$).
	\item We repeat the above Step 1, to determine the arm selections made by $\zeta$. At the beginning of the successive episodes of $T$ rounds, the items $i_1$ and $i_k$ are again chosen uniformly at random in the clusters ${\cal I}_1$ and ${\cal I}_k$, independently of the choices made in earlier episodes.
\end{enumerate}

{\bf Regret of $\zeta$ and its connection to $\mathbb{E}^\tau[N_k]$.} Consider an episode of $T$ rounds for $\tau$. Let $i_k$ denote the item selected from ${\cal I}_k$ in the design of $\zeta$, the expected regret accumulated by $\zeta$ in this episode is $\Delta_k\mathbb{E}^\tau[n_{i_k}]$, where   
$n_{i_k}$ is the number of times $\tau$ selects $i_k$ in the episode. Since $i_k$ is chosen uniformly at random, and by definition of $\mathbb{E}^\tau[N_k]$, we actually have $\mathbb{E}^\tau[n_{i_k}]=\mathbb{E}^\tau[N_k]$, which connects the regret of $\zeta$ and $\mathbb{E}^\tau[N_k]$.

Next, assume that we stop the algorithm $\zeta$ after $\kappa$ episodes of $T$ 
rounds, where $\kappa$ is the first episode where $\zeta$ has 
made more than $m$ selections. $\kappa$ is a random variable, 
and Wald's first lemma implies that the expected regret $R^\zeta$  accumulated 
by $\zeta$ before we stop playing is:
$$
R^\zeta = \Delta_k \mathbb{E}[\kappa]\mathbb{E}^\tau[N_k].
$$
Since $\max\{\mathbb{E}^\tau[N_1],\mathbb{E}^\tau[N_k]\}\le m$, we have $\mathbb{E}[\kappa]\mathbb{E}^\tau[N_{1}]\le 2m$. Hence:
\begin{equation}\label{eq:reg1}
R^\zeta = \Delta_k \mathbb{E}[\kappa]\mathbb{E}^\tau[N_k]\le 2m \Delta_k {\mathbb{E}^\tau[N_k]\over \mathbb{E}^\tau[N_{1}]}.
\end{equation}
By construction, $R^\zeta$ corresponds to the expected regret of our algorithm $\zeta$ for a number of rounds larger than $m$ in the 2-arm bandit problem with known average rewards. The proof of Lemma \ref{lem:2arms} is completed by establishing a lower bound on $R^\zeta$

{\bf Regret lower bound of $R^\zeta$.} We prove the following lemma, which combined with (\ref{eq:reg1}) yields Lemma \ref{lem:2arms}. 

\begin{lemma}\label{lem:reg-2arm} We have: $R^\zeta \ge 2 {\Delta_k}\phi(k,m,p)$.
\end{lemma}

{\bf Proof of Lemma~\ref{lem:reg-2arm}. }
The proof is similar to that of Theorem~6 in \cite{Bubeck13}.  The following lemma by \cite{tsybakov2008introduction, Bubeck13} is the essential ingredient of the proof:
\begin{lemma}
	Let $P_0$ and $P_1$ be two probability measures on a measurable space $(\Omega, \set{F})$, with $P_0$ is absolutely continuous with respect to $P_1$. Then, for any $\set{F}$-measurable function $\Psi: \Omega \to \{0, 1\}$, we have:
	\begin{align*}
	\Pr_{P_0}(\Psi(\omega) =1) + \Pr_{P_1}(\Psi(\omega) =0) \geq \frac{1}{2} \exp(- \KL(P_0, P_1)).
	\end{align*} 
	\label{lm:Bretagnolle-Huber}
\end{lemma}

Consider $a_t\in \{1,k\}$, defined as the $t$-th arm selection made by $\zeta$. This selection happens in the $m_t$-th round of an episode of $T$ rounds for the algorithm $\tau$. At the beginning of this episode, in the design of $\zeta$, items $i_1$ and $i_k$ have been selected, and in this $m_t$-th round, $\tau$ selects either $i_1$ or $i_k$. The decisions made under $\tau$ in this episode depend on the observations made in this episode only, and this remark holds for $\zeta$ as well. We define by ${\cal F}$ the $\sigma$-algebra generated by the observations made before the $m_t$-th round in the episode. To build ${\cal F}$, we assume that each time $i_1$ or $i_k$ is selected, then a sample of the reward of both items is observed. 

With the above definitions, we have $\{ a_t=1\}\in {\cal F}$, and $\{ a_t=k\}\in {\cal F}$. Next consider the the following two probability measures on ${\cal F}$: $P_0$ corresponds to the observations made in the original model (with the true item clusters), and $P_1$ to the observations made assuming that items $i_1$ and $i_k$ are swapped: the average reward of $i_1$ is $p_k$ and that of $i_k$ is $p_1$. $P_0$ and $P_1$ differ only when it comes to observations made in rounds where items $i_1$ and $i_k$ are selected. At round $m_t$, we know that we have had at most $t-1$ such rounds. We deduce that:
$$
\KL(P_0, P_1) \le (t-1)(\kl(p_1,p_k) + \kl(p_k,p_1)).
$$
Applying Lemma \ref{lm:Bretagnolle-Huber}, we get:
\begin{equation}
\Pr_{P_0}(a_t =  k) + \Pr_{P_1} (a_t = 1) \ge \frac{1}{2}\exp \left(- (t-1)(\kl(p_1,p_k) + \kl(p_k,p_1)) \right) =  \frac{1}{2}e^{- (t-1)\gamma(p_1, p_k)}. \label{eq:BH_applied}
\end{equation}
Observe that $\Pr_{P_0}(a_t =  k) + \Pr_{P_1} (a_t = 1)$ is the expected instantaneous regret of $\zeta$ for its $t$-th arm selection. Hence, we have:
\begin{align*}
R^\zeta & \ge \frac{\Delta_k }{4}\sum_{t=1}^m e^{- (t-1)\gamma(p_1, p_k)}
\\
&= \frac{\Delta_k }{4} \frac{1 - e^{-m\gamma(p_1, p_k)}}{1 - e^{ - \gamma(p_1, p_k)}}
\\
& = 2 \Delta_k \phi(k,m,p),
\end{align*}
where the first inequality stems from the fact that $R^\zeta$ is the regret accumulated over more than $m$ rounds, and the second inequality is from (\ref{eq:BH_applied}). \ep

\newpage
\section{Fundamental limits for Model B: Proof of Theorem \ref{th:lowB}}\label{app:thlowB}

We apply the same strategy as in the case of clustered items. Let $\pi$ be an arbitrary algorithm. 

\underline{Proof of $R^\pi(T) \ge R_{\mathrm{nr}}(T)$:} Using the same reasoning as for Model A, the algorithm needs to sample at least $\max_uN_u(T)$, and a new item generated a satisficing regret equal to $\int_0^{\mu_{1-\varepsilon}} (\mu_{1-\varepsilon}-\mu)\zeta(\mu)d\mu$. Hence, we get: 
$$
R^\pi(T)\ge \mathbb{E}[\max_uN_u(T)]  \int_0^{\mu_{1-\varepsilon}} (\mu_{1-\varepsilon}-\mu)\zeta(\mu)d\mu = R_{\mathrm{nr}}(T).
$$

\underline{Proof of $R^\pi(T) \ge R_{\mathrm{i}}(T)$:} For the term $R_{\mathrm{i}}(T)$, assume that $\zeta$ is known. With this knowledge, we denote $\tau$ an optimal algorithm. We denote by $\mathbb{E}^\tau[N_\mu]$ the expected number of rounds an item with success rate $\mu$ is selected under $\tau$. Formally, the algorithm $\tau$ induced the two following random counting measures on the interval $[0,1]$: (i) $\Xi$ counts the number of the items whose parameter is $\mu\in [0,1]$ seen by the algorithm $\tau$ ('seen' means selected at least once), (ii) $\Upsilon$ counts the number of times the algorithm $\tau$ selects items whose parameter is $\mu\in [0,1]$. Now the intensity measures \cite{kallenberg2017random} $\gamma$ and $\omega$ of $\Xi$ and $\Upsilon$ are absolutely continuous w.r.t. $\zeta$, and in addition, $\omega$ is absolutely continuous w.r.t. $\gamma$. Denote by $d_{\gamma}$ and $d_{\omega}$ the densities of $\gamma$ and $\omega$ w.r.t. $\zeta$. Then, $\mathbb{E}^\tau[N_\mu]$ is defined by $d_{\omega}(\mu)/d_{\gamma}(\mu)$. In the remark at the end of this proof, we make these definitions and the expression of the regret of $\tau$ explicit in the case where $\zeta$ is constant over intervals of $[0,1]$. Our proof could actually directly use a sequence of such discretizations, and then concludes by monotone limits. 

Now the regret of an algorithm $\pi$ satisfies:
\begin{align*}
R_\varepsilon^\pi(T)&  \ge T {\int_0^{\mu_{1-\varepsilon}} \mathbb{E}^\tau[N_\mu](\mu_{1-\varepsilon}-\mu)\zeta(\mu)d\mu  \over \int_0^1 \mathbb{E}^\tau[N_\mu]\zeta(\mu)d\mu } 
\\
 & \ge T {\int_0^{\mu_{1-\varepsilon}} \mathbb{E}^\tau[N_\mu](\mu_{1-\varepsilon}-\mu)\zeta(\mu)d\mu  \over \min\{ 1, (1+C)\varepsilon \} m + \int_0^{\mu_{1-\varepsilon} -\varepsilon} \mathbb{E}^\tau[N_\mu]\zeta(\mu)d\mu } \\
& \ge {T\over m} { \frac{(1-\varepsilon)^2}{2C}\left( 1 - \frac{\varepsilon C}{1 - \varepsilon}\right)^2 \over \min\{ 1, (1+C)\varepsilon \}  + 1/m},
\end {align*}
where we use the fact that $\varepsilon \ge \frac{1}{\sqrt{m}}$ and  $1\le \mathbb{E}^\tau[N_\mu] \le m$ for all $\mu$. 

To complete the proof of the theorem, we just establish the following inquality:

\begin{align*}
T {\int_0^{\mu_{1-\varepsilon}} \mathbb{E}^\tau[N_\mu](\mu_{1-\varepsilon}-\mu)\zeta(\mu)d\mu  \over \min\{ 1, (1+C)\varepsilon \} m + \int_0^{\mu_{1-\varepsilon} -\varepsilon} \mathbb{E}^\tau[N_\mu]\zeta(\mu)d\mu } 
& \ge T  { \frac{(1 - \varepsilon)^2}{2 C} \left(1 - \frac{\varepsilon C}{1 - \varepsilon}\right)^2  \over \min\{ 1, (1+C)\varepsilon \} m + 1}.
\end{align*}

Let $\psi_i = \int_{\mu_{1-\varepsilon} -(i+1)\varepsilon}^{\mu_{1-\varepsilon}-i\varepsilon} \mathbb{E}^\tau[N_\mu]\zeta(\mu)d\mu$. Then, 
$$\psi_i \in [\int_{\mu_{1-\varepsilon} -(i+1)\varepsilon}^{\mu_{1-\varepsilon}-i\varepsilon} \zeta(\mu)d\mu, \infty),$$
since $\mathbb{E}^\tau[N_\mu] \ge 1$ for all $\mu$. We have:
\begin{align}
T {\int_0^{\mu_{1-\varepsilon}} \mathbb{E}^\tau[N_\mu](\mu_{1-\varepsilon}-\mu)\zeta(\mu)d\mu  \over \min\{ 1, (1+C)\varepsilon \} m + \int_0^{\mu_{1-\varepsilon} -\varepsilon} \mathbb{E}^\tau[N_\mu]\zeta(\mu)d\mu }
& \ge T {\sum_{i=1}^{\lfloor \mu_{1-\varepsilon}/\varepsilon \rfloor} i\varepsilon \psi_i   \over \min\{ 1, (1+C)\varepsilon \} m + \sum_{i=1}^{\lfloor \mu_{1-\varepsilon}/\varepsilon \rfloor}  \psi_i }. \label{eq:Thm2-1}
\end{align}

As the derivate of $\frac{a + bx}{c + dx}$ is $\frac{bc - da}{(c + dx)^2}$, ${\sum_{i=1}^{\lfloor \mu_{1-\varepsilon}/\varepsilon \rfloor} i\varepsilon \psi_i   \over \min\{ 1, (1+C)\varepsilon \} m + \sum_{i=1}^{\lfloor \mu_{1-\varepsilon}/\varepsilon \rfloor}  \psi_i }$ is either an increasing function or a decreasing function of $\psi_i$ when all other $\psi_j$'s are fixed. Therefore, the r.h.s. of \eqref{eq:Thm2-1} can be optimized only when the $\psi_i$'s are at extreme points, either $\int_{\mu_{1-\varepsilon} -(i+1)\varepsilon}^{\mu_{1-\varepsilon}-i\varepsilon} \zeta(\mu)d\mu$ or $\infty$.

When $\psi_i = \int_{\mu_{1-\varepsilon} -(i+1)\varepsilon}^{\mu_{1-\varepsilon}-i\varepsilon} \zeta(\mu)d\mu$ for all $i$,
\begin{align*}
T {\sum_{i=1}^{\lfloor \mu_{1-\varepsilon}/\varepsilon \rfloor} i\varepsilon \psi_i   \over \min\{ 1, (1+C)\varepsilon \} m + \sum_{i=1}^{\lfloor \mu_{1-\varepsilon}/\varepsilon \rfloor}  \psi_i }
\ge &T {\sum_{i=1}^{\lfloor \mu_{1-\varepsilon}/\varepsilon \rfloor} i\varepsilon \int_{\mu_{1-\varepsilon} -(i+1)\varepsilon}^{\mu_{1-\varepsilon}-i\varepsilon} \zeta(\mu)d\mu  \over \min\{ 1, (1+C)\varepsilon \} m + 1}\cr
\ge &T { \frac{(1 - \varepsilon)^2}{2 C} \left(1 - \frac{\varepsilon C}{1 - \varepsilon}\right)^2  \over \min\{ 1, (1+C)\varepsilon \} m + 1}.
\end{align*}

When $\psi_i = \infty$ for some $i \ge 1$, we have 
$$T {\sum_{i=1}^{\lfloor \mu_{1-\varepsilon}/\varepsilon \rfloor} i\varepsilon \psi_i   \over \min\{ 1, (1+C)\varepsilon \} m + \sum_{i=1}^{\lfloor \mu_{1-\varepsilon}/\varepsilon \rfloor}  \psi_i } \ge \varepsilon T. $$

Thus, we have 
\begin{align*}
T {\int_0^{\mu_{1-\varepsilon}} \mathbb{E}^\tau[N_\mu](\mu_{1-\varepsilon}-\mu)\zeta(\mu)d\mu  \over \min\{ 1, (1+C)\varepsilon \} m + \int_0^{\mu_{1-\varepsilon} -\varepsilon} \mathbb{E}^\tau[N_\mu]\zeta(\mu)d\mu }
& \ge T\min \left\{ { \frac{(1 - \varepsilon)^2}{2 C} \left(1 - \frac{\varepsilon C}{1 - \varepsilon}\right)^2  \over \min\{ 1, (1+C)\varepsilon \} m + 1} , \varepsilon \right\} \cr
& = T { \frac{(1 - \varepsilon)^2}{2 C} \left(1 - \frac{\varepsilon C}{1 - \varepsilon}\right)^2  \over \min\{ 1, (1+C)\varepsilon \} m + 1} ,
\end{align*}
where the last equation stems from the assumption $m \ge \frac{c}{\varepsilon^2}$. This concludes the proof. \ep

{\bf Remark.} Assume that there exists a finite set of non-overlapping intervals of $[0,1]$ and covering $[0,1]$ such that the density of $\zeta$ is constant over each of these intervals. We denote by $\zeta_i$ the probability that when a new item is selected, its parameter lies in the $i$-th interval. Further assume that the satisficing regret of an item with parameter in the $i$-th interval does depend on the parameter, and is equal to $\Delta_i$. Under the algorithm $\tau$, let $\mathbb{E}^\tau[P_i]$ denote the expected number of items seen (selected at least once) by $\tau$ and whose parameter is in the $i$-th interval, and let $\mathbb{E}^\tau[N_i]$ the expected number of times $\tau$ selects an item with parameter in the $i$-th interval. Then, the equivalent of $\mathbb{E}^\tau[N_\mu]$ is $\eta_i$, the expected number of times an item with parameter in the $i$-th interval is selected. $\eta_i$ is defined as $\eta_i= \mathbb{E}^\tau[N_i]/\mathbb{E}^\tau[P_i]$. Observe that since the item parameters of newly selected items are i.i.d. with distribution $\zeta_i$, $\mathbb{E}^\tau[P_i]$ is proportional to $\zeta_i$. This is just a consequence of the general Wald lemma: indeed, if $X_{ki}$ is the binary r.v. indicating whether the $k$-th item seen by the algorithm has a parameter in the $i$-th interval, and if $\kappa$ denotes the random number of items seen by the algorithm within the time horizon $T$, then Wald's equation holds if $\mathbb{E}^\tau[ X_{ki} 1_{\kappa\ge k} ] = \mathbb{E}[ X_{ki} ] \mathbb{P}^\tau[\kappa\ge k]$. This is true in our case since the event $\{\kappa \ge k\}$ corresponds to the fact that $\tau$ decides to sample the $k$-th item, and this decision is solely based on observations made on the $(k-1)$ first items. Finally, the regret of $\tau$ is:
$$
R^\tau(T) = \sum_i \Delta_i \mathbb{E}^\tau[N_i] = T{\sum_i\Delta_i\eta_i\zeta_i\over \sum_i \eta_i\zeta_i }.
$$
In the above formula we used the fact that $T=\sum_i \mathbb{E}^\tau[N_i]$. Note that we obtained a discrete version of $T {\int_0^{\mu_{1-\varepsilon}} \mathbb{E}^\tau[N_\mu](\mu_{1-\varepsilon}-\mu)\zeta(\mu)d\mu  \over \int_0^1 \mathbb{E}^\tau[N_\mu]\zeta(\mu)d\mu }$.

\newpage
\section{Fundamental limits for Model C: Proof of Theorem \ref{th:lowC}}\label{app:thlowC}

We start this section by illustrating the various terms involved in the regret lower bound in Theorem \ref{th:lowC}. We then prove the theorem.

\subsection{Examples}

Let $\pi$ be a uniformly good algorithm. Then under the conditions of Theorem \ref{th:lowC}, we have: $R^\pi(T) \ge \max\{ R_{\mathrm{nr}}(T), R_{\mathrm{ic}}(T),R_{\mathrm{uc}}(T) \}$ and $R^\pi (T) \gtrsim R_{\mathrm{uc}}'(T)$. We exemplify the scalings of these terms below, with a particular emphasis on those due to the need of learning user clusters, $R_{\mathrm{uc}}(T)$ and $R_{\mathrm{uc}}'(T)$. 

{\bf Case 1.}
We consider the case $L=2$ and $K=2$ with a following parameter set in Table~\ref{tb:modelC2arm_example_3}.
\begin{table}[H]
	\begin{center}
		\begin{tabular}{|c|c|c|}
			\hline
			& $k=1$ & $k=2$ \\
			\hline
			$\ell =1$ & $0.8$ & $0.6$ \\
			\hline
			$\ell = 2$ &  $0.8$ & $0.9$ \\
			\hline
		\end{tabular}
	\end{center}
	\caption{Values of $(p_{k \ell})$.}
	\label{tb:modelC2arm_example_3}
\end{table}
For this parameter, $\set{L}_\perp = \{(1,2)\}$ and $\set{L}^\perp(1) = \{2\}$. We have: $R_\mathrm{nr} = \Omega(\overline{n})$, $ R_\mathrm{ic} = \Omega(\frac{T}{m})$ and $R_\mathrm{uc} = \Omega(m)$. Furthermore, when $T = \omega(m)$,
\begin{align*}
R^\pi (T) \gtrsim \frac{\beta_1(0.8 - 0.6)}{\kl(0.6, 0.9)} m \log(\frac{T}{m}) = \Omega( m \log(\frac{T}{m})).
\end{align*}

{\bf Case 2.}
We consider the case $L=2$ and $K=2$ with a following parameter set in Table~\ref{tb:modelC2arm_example_4}.
\begin{table}[H]
	\begin{center}
		\begin{tabular}{|c|c|c|}
			\hline
			& $k=1$ & $k=2$ \\
			\hline
			$\ell =1$ & $0.8$ & $0.6$ \\
			\hline
			$\ell = 2$ &  $0.8$ & $0.7$ \\
			\hline
		\end{tabular}
	\end{center}
	\caption{Values of $(p_{k \ell})$.}
	\label{tb:modelC2arm_example_4}
\end{table}
For this parameter, $\set{L}_\perp = \emptyset$ and $\forall \ell \in [L], \; \set{L}^\perp(\ell)= \emptyset$. We have: $R_\mathrm{nr} = \Omega(\overline{n})$, $ R_\mathrm{ic} = \Omega(\frac{T}{m})$, $R_\mathrm{uc} = 0$ and $c(\beta, p) = 0$.

{\bf Case 3.}
We consider the case $L=2$ and $K=2$ with a following parameter set in Table~\ref{tb:modelC2arm_example_5}.
\begin{table}[H]
	\begin{center}
		\begin{tabular}{|c|c|c|}
			\hline
			& $k=1$ & $k=2$ \\
			\hline
			$\ell =1$ & $0.8$ & $0.9$ \\
			\hline
			$\ell = 2$ &  $0.8$ & $0.85$ \\
			\hline
		\end{tabular}
	\end{center}
	\caption{Values of $(p_{k \ell})$.}
	\label{tb:modelC2arm_example_5}
\end{table}

In this case, $\set{L}_\perp = \emptyset$ and $\forall \ell \in [L], \; \set{L}^\perp(\ell)= \emptyset$. We have: $R_\mathrm{nr} = \Omega(\overline{n})$, $ R_\mathrm{ic} = \Omega(\frac{T}{m})$, $R_\mathrm{uc} = 0$ and $c(\beta, p) = 0$.

{\bf Case 4.} We consider the case $L=2$ and $K=3$ with a following parameter set in Table~\ref{tb:modelC2arm_example_6}.
\begin{table}[H]
	\begin{center}
		\begin{tabular}{|c|c|c|c|}
			\hline
			& $k=1$ & $k=2$ & $k=3$ \\
			\hline
			$\ell =1$ & $0.9$ & $0.8$ & $0.7$\\
			\hline
			$\ell = 2$ &  $0.7$ & $0.8$  & $0.9$\\
			\hline
		\end{tabular}
	\end{center}
	\caption{Values of $(p_{k \ell})$.}
	\label{tb:modelC2arm_example_6}
\end{table}
In this case,  $\set{L}_\perp = \{ (1, 2) \}$ and $\forall \ell \in [L], \; \set{L}^\perp(\ell)= \emptyset$. We have: $R_\mathrm{nr} = \Omega(\overline{n})$, $ R_\mathrm{ic} = \Omega(\frac{T}{m})$, $R_\mathrm{uc} = \Omega(m)$ and $c(\beta, p) = 0$.

\subsection{Proof} \label{subsec:append_proof_lowC}

\underline{Proof of $R^\pi(T) \ge \max\{ R_{\mathrm{nr}}(T), R_{\mathrm{ic}}(T)\}$.}  The proof is the same as in that of Theorem \ref{th:lowA}. 

\underline{Proof of $R^\pi(T) \ge R_{\mathrm{uc}}(T)$.} We first give a simple proof that learning user clusters induces a regret scaling as $m$, i.e., that $R^\pi(T) =\Omega (m)$.
When a user first arrives, we do not know her cluster, and hence we have to recommend an item from a cluster picked randomly. This selection induces an average regret at least equal to $ \min_\ell\{\beta_\ell\} \Delta$ when ${\cal L}_\perp \neq \emptyset$. Since the number of users that arrive at least once is in expectation larger than $m(1- (1-{1\over m})^T)\ge m(1- e^{-T/m})$, we get that: $R^\pi(T) \ge   \min_\ell\{\beta_\ell\} \Delta m(1- e^{-T/m}) $. 

To get the right constant in the regret lower bound, we need to develop a more involved argument. Assume that the item clusters $\set{I}_k$'s and the success rates $p$ are known. With this knowledge, we denote by $\tau$ an optimal algorithm. We derive a regret lower bound for $\tau$.  Define $N_k^u : = \sum_{t=1}^T \indicator_{\{i_t^\pi \in \set{I}_k, u_t =u\}}$ as the number of times user $u$ is presented items in cluster ${\cal I}_k$ (under $\pi$).  Fix user clusters $\set{U}_1, \ldots, \set{U}_L$. Similar to the proof of Lemma~\ref{lem:2arms}, we will prove that:
\begin{lemma}\label{lm:low2arm_modelC}
	We have for all $T \ge 2 m$, for any $\ell \in [L]$, for any $k \in [K]$ such that $ \exists r \in \set{R}_\ell$ such that $k = k_r^\star$:
	$
		\frac{\psi(\ell, k, T, m, p)}{(T/m)} \le \frac{\EXP^\tau[N_k^u \mid u \in \set{U}_\ell]}{\EXP^\tau[N_{k_\ell^\star}^u \mid u \in \set{U}_\ell]} \le 1,
	$
	where $\psi(\ell, k, T, m, p) = \frac{1 - e^{ - \frac{T}{m}\gamma(p_{k_\ell^\star \ell}, p_{k \ell}) }}{8\left( 1 - e^{ - \gamma(p_{k_\ell^\star, \ell}, p_{k \ell}) }\right)}.$
\end{lemma}
For fixed user clusters, the expected conditional regret $R^\tau(T| \set{U}_1, \ldots, \set{U}_L)$ is: $R^\tau(T| \set{U}_1, \ldots, \set{U}_L) = \sum_{\ell \in [L]}|\set{U}_\ell| \sum_{k \in \set{R}_\ell} \EXP^\tau[N_{k}^u \mid u \in \set{U}_\ell]\Delta_{k \ell}$. Therefore, we have:
\begin{align*}
	R^\tau(T| \set{U}_1, \ldots, \set{U}_L) & = \sum_{\ell \in [L]}|\set{U}_\ell| \frac{\frac{T}{m} \sum_{k \in \set{R}_\ell}  \EXP^\tau[N_{k}^u \mid u \in \set{U}_\ell]\Delta_{k \ell}}{ \sum_{k \in [K]}  \EXP^\tau[N_{k}^u \mid u \in \set{U}_\ell]}
	\\
	& \stackrel{(a)}{\ge} \sum_{\ell \in [L]}|\set{U}_\ell| \frac{\frac{T}{m} \sum_{k \in \set{R}_\ell}  \EXP^\tau[N_{k_\ell^\star}^u \mid u \in \set{U}_\ell]\Delta_{k \ell} \psi(\ell, k, T, m, p)}{ \sum_{k \in [K]}  \EXP^\tau[N_{k_\ell^\star}^u \mid u \in \set{U}_\ell] \frac{T}{m} }
	\\
	& = \sum_{\ell \in [L]}|\set{U}_\ell| \frac{\sum_{k \in \set{R}_\ell} \Delta_{k \ell} \psi(\ell, k, T, m, p)}{ K},
\end{align*}
where for $(a)$ we used Lemma~\ref{lm:low2arm_modelC} (${\psi(\ell, k, T, m, p) \EXP^\tau[N_{k_\ell^\star}^u]}/{(T/m)} \le \EXP^\tau[N_k^u \mid u \in \set{U}_\ell]$ in the numerator, and $ \EXP^\tau[N_k^u \mid u \in \set{U}_\ell] \le  \EXP^\tau[N_{k_\ell^\star}^u \mid u \in \set{U}_\ell]$ in the denominator).
Taking the expectation over $\set{U}_1, \ldots, \set{U}_L$, we have (since $ \EXP[\set{U}_\ell] = \beta_\ell m$):
\begin{align*}
	R^\tau(T) \ge m \sum_{\ell \in [L]} \beta_\ell \frac{\sum_{k \in \set{R}_\ell} \Delta_{k \ell} \psi(\ell, k, T, m, p)}{ K}.
\end{align*}
{\bf Proof of Lemma~\ref{lm:low2arm_modelC}.}  Consider a $2$-armed bandit problem with Bernoulli arms  $ k_\ell^\star$  and $k$ of means $p_{k_\ell^\star}$ and $p_k$. The means are known but the arm with the highest mean is unknown. For this bandit problem, we build an algorithm $\zeta$ based on the decisions by $\tau$. 

{\bf A valid algorithm $\zeta$ based on $\tau$.} 
\begin{enumerate}
	\item Pick a user $u$ uniformly at random in the cluster ${\cal U}_\ell$. We run $\tau$ for $T$ rounds. When user $u$ comes to the system and $\tau$ selects the item in $\set{I}_{ k_\ell^\star}$ (resp. $\set{I}_{ k}$), then $\zeta$ also selects the arm $k_\ell^\star$ (resp. $k$). We call this procedure as an episode. 
	\item We repeat the above Step 1, to determine the arm selections made by $\zeta$. At the beginning of the successive episodes of $T$ rounds, the user $u$ are again chosen uniformly at random in the clusters ${\cal I}_\ell$, independently of the choices made in earlier episodes.
\end{enumerate}
$\zeta$ is a valid algorithm as the decision by $\tau$ is based on past observations. We stop $\zeta$ after $\kappa$ episodes of $T$ rounds, where $\kappa$ is the first episode where $\zeta$ has made more than $\frac{T}{m}$ selections.  By Wald's first lemma, the expected regret accumulated by $\zeta$ before we stop playing is:
\begin{align*}
	R^\zeta = \Delta_{k \ell} \EXP^\tau[\kappa] \EXP^\tau[N^u_k \mid u \in \set{U}_\ell].
\end{align*}
Since $ \EXP^\tau[N^u_k \mid u \in \set{U}_\ell] +  \EXP^\tau[N^u_{k_\ell^\star} \mid u \in \set{U}_\ell] \le \frac{T}{m}$, We have:
\begin{align*}
	\EXP[\kappa]\EXP^\tau[N^u_{k_\ell^\star} \mid u \in \set{U}_\ell] & \le \underbrace{\frac{T}{m}}_{\text{Stopping criteria of } \zeta} + \underbrace{\EXP[\sum_k N_k^u \mid u \in \set{U}_\ell]}_{\text{expected number of drawing the user $u$ in a single episode}}
	\\
	& = \frac{2T}{m}
\end{align*}
We will also prove a lower bound on $R^\zeta$:
\begin{lemma}\label{lem:two_arm_modelC}
	We have for all $T \ge 2 m$: $R^\zeta \ge 2 \Delta_{k \ell} \psi(\ell, k, T, m, p)$.
\end{lemma}
Combining this lemma with Lemma~\ref{lm:low2arm_modelC} concludes the proof of $R^\pi(T) \ge R_{\mathrm{uc}}(T)$.

{\bf Proof of Lemma~\ref{lem:two_arm_modelC}.} 
When the algorithm decides which arm to choose at time $t$, we assume that the algorithm has access to the rewards of both arms up to time $t-1$. This is a simpler problem than the original $2$-armed bandit problem. Hence the regret in the original problem is higher than that in the simpler problem. Let $\theta$ denote the distribution of the rewards of both arms, when the average reward of the first arm is $p_{k_\ell^\star}$ and that of the second arm is $p_k$. Let $\theta'$ denote the distribution of the rewards of both arms, when arms are swapped: the average reward of the first arm is $p_k$ and that of the second arm is $p_{k_\ell^\star}$. Let $\theta^{\otimes (t-1)}$ be a product measure for the reward observations up to time $t-1$ under the measure $\theta$. Let $k_t$ be the arm selected at time $t$. From Lemma \ref{lm:Bretagnolle-Huber}, we have, for each $t\ge 2$,
\begin{align} 
\Pr_{\theta}(k_t =  k) + \Pr_{\theta'} (k_t = k_{\ell}^\star)  
& \geq \frac{1}{2} \exp(- \KL(\theta^{\otimes (t-1)}, \theta'^{\otimes (t-1)})) \nonumber
\\
& = \frac{1}{2}\exp \left(- (t-1)(\kl(p_{k_\ell^\star \ell},p_{k\ell}) + \kl(p_{k \ell},p_{k_\ell^\star \ell})) \right) =  \frac{1}{2}e^{- (t-1)\gamma(p_{k^\star_\ell \ell}, p_{k \ell})}. \label{eq:BH_applied2}
\end{align}
Note that (\ref{eq:BH_applied2}) also holds for $t=1$ by the symmetry. 
We have,
\begin{align*}
R^\zeta & \ge \frac{\Delta_{k \ell} }{2}\sum_{t=1}^{T/m} \left(\Pr_{\theta}(k_t =  k) + \Pr_{\theta'} (k_t = k_\ell^\star)  \right)
\\
& \ge \frac{\Delta_{k \ell} }{4}\sum_{t=1}^{T/m} e^{- (t-1)\gamma(p_{k_\ell^\star \ell}, p_{k \ell})}
\\
& = \frac{\Delta_{k \ell} }{4} \frac{1 - e^{-\frac{T}{m}\gamma(p_{k_\ell^\star \ell}, p_{k \ell})}}{1 - e^{ - \gamma(p_{k_\ell^\star \ell}, p_{k \ell})}}
\\
& = 2 \Delta_{k \ell} \psi(\ell, k,T, m,p),
\end{align*}
where the first inequality stems from the fact that $R^\zeta$ is the regret accumulated over more than $\frac{T}{m}$ rounds, and the second inequality is from (\ref{eq:BH_applied2}). \ep

\medskip

\underline{Proof of $R^\pi(T)\gtrsim c(\beta,p)m\log(T/m)$.} We define:
\begin{align*}
\Theta: = \{ p \in [0,1]^{K \times L} : 
\forall \ell \in [L], \exists \Gamma \; (\text{permutation of } [K]) 
 \text{ s.t. } p_{\Gamma(1) \ell} > p_{\Gamma(2) \ell} \ge \ldots \ge p_{\Gamma(K) \ell} \}
\end{align*}
as a set of all possible problems. We denote $k_\ell^\star : = \arg \max_{k} p_{k \ell}$ as the index of the best item cluster for users in the cluster $\set{U}_\ell$ and $\Delta_{k \ell} := p_{k_\ell^\star \ell} - p_{k \ell} \; \forall k \in [K], \forall \ell \in [L]$. Consider an arbitrary algorithm $\pi$. We define the regret of a single user $u$ as:
\begin{align*}
R_u^\pi(T)& := \frac{T}{m} \sum_{\ell} \beta_\ell p_{k_\ell^\star \ell} - \sum_{t=1}^T \EXP\left[\sum_{k, \ell} \indicator_{\{u \in \set{U}_\ell, u_t = u, i_t^\pi \in \set{I}_k\}}p_{k \ell}\right]
\\
& = \sum_{\ell} \beta_\ell  \sum_{k \neq k_\ell^\star}\Delta_{k \ell} \EXP[N^u_k | u\in {\cal U}_\ell],
\end{align*}
where $N_k^u : = \sum_{t=1}^T \indicator_{\{i_t^\pi \in \set{I}_k, u_t =u\}}$ is the number of times user $u$ is presented an item of cluster ${\cal I}_k$ (under $\pi$). Remember that $u$ is random and belongs to ${\cal U}_\ell$ with probability $\beta_\ell$. We further define the conditional regret of a single user $u\in \set{U}_\ell$ given $N_u(T) = N$ and $u\in {\cal U}_\ell$ as:
\begin{align*}
R_u^\pi(T)_{N, \ell} & : = N p_{k_\ell^\star \ell} - \sum_{t=1}^{T }\EXP\left[\sum_{k} \indicator_{\{u_t = u, i_t^\pi \in \set{I}_k\}}p_{k \ell} \,\middle\vert\, N_u(T) = N, u \in \set{U}_\ell \right]
\\
& =  \sum_{k \neq k_\ell^\star} \Delta_{k \ell} \EXP[N^u_k \mid u \in \set{U}_\ell, N_u(T) = N].
\end{align*}
Note that we have $R^\pi(T) = \sum_{u \in \set{U}} R_u^\pi(T) = m R_u^\pi(T)$ and  $R_u^\pi(T)  = \sum_{N \ge 1}^T \Pr(N_u(T) =N) \sum_{\ell \in [L]} \beta_\ell R_u^\pi (T)_{N, \ell}$.

Assume that $\pi$ is uniformly good. This means that if for all problem $p \in \Theta$, as $ N \to \infty$, the conditional regret $R_u^\pi(T)_{N, \ell}$ satisfies:
\begin{align}\label{eq:cond_unif_good}
\forall \ell \in [L], \;\forall \; 0 < \alpha < 1, \quad R_u^\pi(T)_{N, \ell}  = o \left( \left( N \right)^\alpha \right).
\end{align}
The existence of uniformly good algorithms is guaranteed because applying the classical algorithms (e.g, UCB1) to each user satisfy indeed is uniformly good (this is proved for the ECB algorithm using Theorem \ref{thm:spec} presented in Appendix \ref{app:upperECB}). 

We state our claim in the following theorem, providing a lower bound of the regret of a single user:
\begin{theorem}
	\label{thm:lowerbound_solo_user}
	For any uniformly good algorithm $\pi \in \Pi$, for any $p \in \Theta$, when $T = \omega( m)$, we have: for any $u \in \set{U}$,
	\begin{align*}
	\liminf_{T \to \infty} \frac{R_u^\pi(T)}{\log(T/m)} \ge c(\beta, p),
	\end{align*}
	where $c(\beta, p)$ is the value of the following optimization problem:
	\begin{align}
	\inf_{n=(n_{k \ell}) \ge 0} &  \sum_{\ell \in [L]} \beta_\ell \sum_{k \neq k_\ell^\star} \Delta_{k \ell} n_{k \ell} \label{eq:lowerbound_opt}
	\\
	& \text{s.t. } \quad \forall \ell \in [L], \quad \forall \ell' \in \set{L}^\perp(\ell),  \quad\sum_{k \neq k_\ell^\star}\kl(p_{k \ell}, p_{k \ell'}) n_{k \ell}  \ge 1. \nonumber
	\end{align}
\end{theorem}
In the above theorem, we can interpret $n_{k \ell}$ as 
$$
n_{k \ell} \stackrel{T \to \infty}{\sim} \frac{\EXP[N_k^u \mid u \in \set{U}_\ell, N_u(T) = T/m]}{\log (T/m)}.
$$

{\bf Proof of Theorem~\ref{thm:lowerbound_solo_user}: Case $K=2$, $L=2$.} To illustrate the idea behind the proof, we address the simple case with two item and user clusters. We define the values of $(p_{k \ell})$ as in Table~\ref{tab:modelC2arm}, where $\mu_1 > \mu_2$ and  $\mu_1 < \mu_2'$, so that $\set{L}^\perp(1) = \{2\}$.
\begin{table}[htb]
	\label{tab:modelC2arm}
	\begin{center}
		\begin{tabular}{|c|c|c|}
			\hline
			& $k=1$ & $k=2$ \\
			\hline
			$\ell =1$ & $\mu_1$ & $\mu_2$ \\
			\hline
			$\ell = 2$ &  $\mu_1$ & $\mu_2'$ \\
			\hline
		\end{tabular}
	\end{center}
	\caption{The values of $(p_{k \ell})$. $\mu_1 > \mu_2$ and  $\mu_1 < \mu_2'$.}
\end{table}

The proof is in two steps. In the first step we derive a lower bound of the conditional regret, and in the second step, we de-condition using properties of the user arrival process.

{{\it Step 1.}} In this step, we condition on $N_u(T) = N$ and $u\in \set{U}_1$. All expectations and probabilities are conditioned with respect by these events. We apply a classical change-of-measure argument. Let $p$ denote the original model. We build a perturbed model $p'$ obtained from $p$ by just swapping the ides of the user clusters. Let $\Pr_p$ and $\EXP_{p}$ (resp. $\Pr_{p'}$ and $\EXP_{p'}$) be the probability measure and the expectation under $p$ (resp. $p'$), respectively. We compute the log-likelihood ratio of the observations for user $u$ generated under $p$ and $p'$ as:
\begin{align*}
L_T : = &  \sum_{t=1}^T \indicator_{\{u_t = u\}} \left(  \indicator_{\{u \in \set{U}_1, i_t^\pi \in \set{I}_2\}} \left( \indicator_{\{ X_{ i_t^\pi u } = +1\} } \log \frac{\mu_2}{\mu_2'} + \indicator_{\{ X_{ i_t^\pi u } = 0\} } \log \frac{1- \mu_2}{1- \mu_2'}  \right) \right. 
\\
& \left. + \indicator_{ \{ u \in \set{U}_2, i_t^\pi \in \set{I}_2 \}} \left( \indicator_{\{ X_{ i_t^\pi u } = +1\} } \log \frac{\mu_2'}{\mu_2} + \indicator_{\{ X_{ i_t^\pi u } = 0\} } \log \frac{1- \mu_2'}{1- \mu_2}  \right) \right).
\end{align*}
For any measurable random variable $Z \in [0, 1]$, we have:
\begin{align*}
\EXP_p[L_T] = \EXP_p [N_2^u] \kl (\mu_2, \mu_2') \stackrel{(a)}{\ge} \kl (\EXP_p [Z], \EXP_{p'}[Z]),
\end{align*}
where $(a)$ stems from the data-processing inequality (cf. \cite{garivier2018explore}). Taking $Z = N_2^u / N$, we have:
\begin{align*}
\kl (\EXP_p [Z], \EXP_{p'}[Z]) & =  \kl \left( \frac{\EXP_p [N_2^u]}{N}, \frac{\EXP_{p'}[N_2^u]}{N} \right)
\\
& \ge \left(1 - \frac{\EXP_p [N_2^u]}{N}\right) \log \left(\frac{N}{N - \EXP_{p'}[N_2^u]}\right) - \log 2,
\end{align*}
where for the last inequality, we used that for all $(x, y ) \in [0, 1]^2$, 
\begin{align*}
\kl(x,y) \ge (1 - x) \log \frac{1}{1 - y} - \log 2.
\end{align*}
As the algorithm $\pi$ is uniformly good, $N- \EXP_{p'}(N_2^u) = o(N^\alpha)$ for all $\alpha \in (0, 1)$ as $N \to \infty$. Therefore, for all $\alpha \in (0,1)$,
\begin{align*}
\liminf_{N \to \infty} \frac{1}{\log N} \log \frac{N}{N - \EXP_{p'}(N_2^u)} & \ge \liminf_{N \to \infty} \frac{1}{\log N} \log \frac{N}{N^\alpha}
\\
& = 1 - \alpha.
\end{align*}
Furthermore, $\lim_{N\to \infty} \frac{\EXP_p [N_2^u]}{N} = 0 $ as $\pi$ is uniformly good. Therefore, we have:
\begin{align}
\label{eq:modelC_lbeq}
\liminf_{N \to \infty} \frac{\EXP_p [N_2^u]}{\log N} \ge \frac{1}{\kl (\mu_2, \mu_2')}.
\end{align}

{\it Step 2. De-conditioning.} In view of Lemma \ref{lem:lower_tail}, we have:
\begin{align*}
\Pr\left(N_u(T) \le \frac{T}{2m}\right) \le \exp \left(  - \frac{T}{2m}\left( 1 - \log(2) \right)\right).
\end{align*}

From the above inequality and (\ref{eq:modelC_lbeq}), we deduce:
\begin{align*}
\liminf_{T \to \infty} \frac{R_u^\pi (T)}{\log (T/m)} & \ge  \liminf_{T \to \infty}  \frac{ \sum_{N=1}^{T} \Pr(N_u(T) = N) \beta_1 R_u^\pi (T)_{N, 1}}{\log(T/m)}
\\
& \ge   \liminf_{T \to \infty}  \frac{ \sum_{N=\frac{T}{2m}}^{T} \Pr(N_u(T) = N) \beta_1 R_u^\pi(T)_{\frac{T}{2m}, 1}}{\log(T/m)}
\\
& \ge \liminf_{T \to \infty} \beta_1 (\mu_1 - \mu_2) \sum_{N=\frac{T}{2m}}^{T} \Pr(N_u(T) = N) \frac{ \EXP_p [N_2^u | u \in \set{U}_1, N_u(T) = T/(2m)]}{\log (T/m)} 
\\
& \stackrel{(a)}{\ge} \liminf_{T \to \infty}\beta_1 (\mu_1 - \mu_2) \left(1 - \exp \left(- \frac{T}{2m}(1 - \log 2)\right)\right) \frac{\log(T/(2m))}{\kl(\mu_2, \mu_2') \log(T/m)}
\\
& \stackrel{(b)}{=} \frac{\beta_1 (\mu_1 - \mu_2) }{\kl(\mu_2, \mu_2')},
\end{align*}
where for $(a)$ we used Lemma~\ref{lem:lower_tail} with \eqref{eq:modelC_lbeq} and for $(b)$ we used $T=\omega(m)$. This concludes the proof of the case $K=2$ and $L=2$.

\ep

{\bf Proof of Theorem~\ref*{thm:lowerbound_solo_user}: General case.}
We consider a simpler problem: the algorithm knows the values of $(p_{k \ell})$ and $\beta_\ell$. Take $(\ell, \ell') \in [L]^2$ such that  $\ell' \in \set{L}^\perp(\ell)$. As in the case of two user and item clusters, we will prove that (this is done at the end of this proof):
\begin{align}\label{eq:lowC_general}
\liminf_{N \to \infty} \frac{\sum_{k \neq k_\ell^\star}\EXP_p [N_k^u \mid u \in \set{U}_\ell, N_u(T)=N ] \kl(p_{k \ell}, p_{k \ell'})}{\log N} \ge 1.
\end{align}

This inequality holds for any possible $(\ell, \ell')$ such that $\ell' \in \set{L}^\perp(\ell)$. Therefore, for all $\ell \in [L]$, for all $\ell' \in \set{L}^\perp(\ell)$
\begin{align*}
\liminf_{N \to \infty} \frac{ \sum_{k \neq k_\ell^\star}\EXP_p [N_k^u  \mid u \in \set{U}_\ell, N_u(T)=N ] \kl(p_{k \ell}, p_{k \ell'})}{\log N} \ge 1.
\end{align*}

Then, we have:
\begin{align*}
& \liminf_{T \to \infty} \frac{R_u^\pi(T)}{\log(T/m)} 
\\
& \ge \liminf_{T \to \infty} \sum_{\ell \in [L]} \beta_\ell \sum_{N=1}^T \Pr(N_u(T) = N) \frac{R_u^\pi(T)_{N, \ell}}{\log (T/m)}
\\
& \ge  \liminf_{T \to \infty} \sum_{\ell \in [L]} \beta_\ell \sum_{N=\frac{T}{2 m}}^{T}\Pr(N_u(T) = N) \frac{\sum_{k \neq k_\ell^\star} \Delta_{k \ell} \EXP[N_k^u \mid u \in \set{U}_\ell, N_u(T) = T/(2m)]}{\log(T/m)}
\\
& \stackrel{(a)}{\ge} \liminf_{T \to \infty} \left(1 - \exp \left(- \frac{T}{2m}(1 - \log 2)\right)\right) \frac{ \sum_{\ell \in [L]} \beta_\ell \sum_{k \neq k_\ell^\star} \Delta_{k \ell} \EXP[N_k^u \mid u \in \set{U}_\ell, N_u(T) = T/(2m)]}{\log(T/m)}
\\
& = \liminf_{T \to \infty} \left(1 - \exp \left(- \frac{T}{2m}(1 - \log 2)\right)\right) \frac{ \sum_{\ell \in [L]} \beta_\ell \sum_{k \neq k_\ell^\star} \Delta_{k \ell} \EXP[N_k^u \mid u \in \set{U}_\ell, N_u(T) = T/(2m)]}{\log(T/(2m)) + \log 2}
\\
& \stackrel{(b)}{=}   \liminf_{T \to \infty} \frac{ \sum_{\ell \in [L]} \beta_\ell \sum_{k \neq k_\ell^\star} \Delta_{k \ell} \EXP[N_k^u \mid u \in \set{U}_\ell, N_u(T) = T/(2m)]}{\log(T/(2m))},
\end{align*}
where $(a)$ is from Lemma~\ref{lem:lower_tail} and $(b)$ is from $T=\omega(m)$.
Thus, we have:
\begin{align*}
\liminf_{T \to \infty} \frac{R_u^\pi(T)}{\log(T/m)} \ge c(\beta, p),
\end{align*}
where $c(\beta, p)$ is the value of the following optimization problem:
\begin{align*}
\inf_{n = (n_{k \ell}) \ge 0} &  \sum_{\ell \in [L]} \beta_\ell \sum_{k \neq k_\ell^\star} \Delta_{k \ell} n_{k \ell}
\\
& \text{s.t. } \quad \forall \ell \in [L], \quad \forall \ell' \in \set{L}^\perp(\ell),  \quad\sum_{k \neq k_\ell^\star}\kl(p_{k \ell}, p_{k \ell'}) n_{k \ell}  \ge 1.
\end{align*}

\underline{Proof of the inequality \eqref{eq:lowC_general}.} Again, we use a change-of-measure argument. Let $p$ and $p'$ be a original model and a model with the indices of user clusters $(\ell, \ell')$ are swapped from the original model, respectively. Let $\Pr_p$ and $\EXP_{p}$ (resp. $\Pr_{p'}$ and $\EXP_{p'}$) be the probability measure and the expectation under $p$ (resp. $p'$), respectively. We define our log-likelihood ratio as:
\begin{align*}
L_T : = &  \sum_{t=1}^T \sum_{k \neq k^\star_\ell} \indicator_{\{u_t = u\}} \left(  \indicator_{\{u \in \set{U}_\ell, i_t^\pi \in \set{I}_{k}\}} \left( \indicator_{\{ X_{ i_t^\pi u}  = +1\} } \log \frac{p_{k \ell}}{p_{k \ell'}} 
+ \indicator_{\{ X_{ i_t^\pi u}  = 0 \} } \log \frac{1- p_{k \ell}}{1-  p_{k \ell'}}  \right) \right. 
\\
& \left. +  \indicator_{\{u \in \set{U}_{\ell'}, i_t^\pi \in \set{I}_{k}\}} \left( \indicator_{\{ X_{ i_t^\pi u } = +1\} } \log \frac{p_{k \ell'}}{p_{k \ell}} 
+ \indicator_{\{ X_{ i_t^\pi } u = 0\} } \log \frac{1- p_{k \ell'}}{1-  p_{k \ell}}  \right) \right).
\end{align*}
Taking the conditional expectation $\EXP_p[\cdot \mid u \in \set{U}_{\ell}, N_u(T) = N]$, we have:
\begin{align*}
\EXP_p[L_T \mid u \in \set{U}_{\ell}, N_u(T) = N] & = \sum_{k \neq k^\star_\ell}\EXP_p[N_{k}^u \mid u \in \set{U}_{\ell}, N_u(T) = N]\kl(p_{k \ell}, p_{k\ell'}) 
\\
& \stackrel{(a)}{\ge} \kl \left( \frac{\EXP_p [N_{k_{\ell'}^\star}^u]}{N}, \frac{\EXP_{p'}[N_{k_{\ell'}^\star}^u]}{N} \right)
\\
&\stackrel{(b)}{\ge} \left(1 - \frac{\EXP_p [N_{k_{\ell'}^\star}^u]}{N}\right) \log \left(\frac{N}{N - \EXP_{p'}[N_{k_{\ell'}^\star}^u]}\right) - \log 2,
\end{align*}
where for $(a)$, we used the data processing inequality by  \cite{garivier2018explore} and for $(b)$, we used that for all $(x, y ) \in [0, 1]^2$, 
\begin{align*}
\kl(x,y) \ge (1 - x) \log \frac{1}{1 - y} - \log 2.
\end{align*}
As the algorithm $\pi$ is uniformly good, $ \EXP_{p}[N_{k_{\ell'}^\star}^u] = o(N^\alpha)$ and $N- \EXP_{p'}(N_{k_{\ell'}^\star}^u) = o(N^\alpha)$ for all $\alpha \in (0, 1)$ as $N \to \infty$. Therefore, for all $\alpha \in (0,1)$, we have:
\begin{align*}
\liminf_{N \to \infty} \frac{1}{\log N} \log \frac{N}{N - \EXP_{p'}[N_{k_{\ell'}^\star}^u]} & \ge \liminf_{N \to \infty} \frac{1}{\log N} \log \frac{N}{N^\alpha}
\\
& = 1 - \alpha
\end{align*}
and 
\begin{align*}
\liminf_{N \to \infty}\frac{\EXP_p [N_{k_{\ell'}^\star}^u]}{N}  = 0.
\end{align*}
Therefore, we have:
\begin{align*}
\liminf_{N \to \infty}  \frac{ \sum_{k \neq k^\star_\ell}\EXP_p[N_{k}^u \mid u \in \set{U}_{\ell}, N_u(T) = N]\kl(p_{k \ell}, p_{k \ell'})}{\log(N)}  \ge 1.
\end{align*}
This concludes the proof of Theorem~\ref{thm:lowerbound_solo_user}. \ep

\newpage
\section{Performance guarantees of ECT: Proof of Theorem \ref{thm:algorithmA}}\label{app:upperA}

The proof consists in several parts. First we study the initial sampling procedure (at the beginning of the exploration phase). We then upper bound the regret induced by the exploration phase. We analyze the performance of the clustering part of the algorithm, and finally upper bound the regret generated during the test phase.

{\bf Item sampling procedure.}
Let $\tilde{\mathcal{I}}_k = {\cal S}\cap \mathcal{I}_k$ be the set of items from ${\cal I}_k$ that are sampled. \\
Then, for $\epsilon_1 = \sqrt{\frac{\log TK}{2(\min_k \alpha_k)^2(\log T)^2}}$,
\begin{align*}
 \Pr\left(|\tilde{\mathcal{I}}_k|\leq\alpha_k(1-\epsilon_1) (\log T)^2\right) 
& \stackrel{(a)}\leq \exp\left(-(\log T)^2 \kl((1-\epsilon_1)\alpha_k, \alpha_k)\right)\\
& \stackrel{(b)}\leq \exp\left(-(\log T)^2 2\epsilon_1^2\alpha_k^2\right)\\
& \leq \frac{1}{TK}.
\end{align*}
where (a) is from Chernoff-Hoeffding bound (b) is from Pinsker’s inequality.

Hence, the event $\mathcal{A}_1=\{|\tilde{\mathcal{I}}_k|\geq\alpha_k(1-\epsilon_1) (\log T)^2 \textrm{ for } 1\leq k \leq K\}$ holds with probability at least $1-\frac{1}{T}$. As a consequence, the expected regret due the event $\mathcal{A}_1^c$ is $O(1)$. Thus, we can assume that the event $\mathcal{A}_1$ holds throughout the remaining of the proof.

{\bf Exploration phase.} In this phase, we wish to recommend each item in ${\cal S}$ for $\log T$ times. We prove that this exploration phase takes around $(\log T)^3$ rounds (and this is the regret it generates). Let us consider a user $u$. This user can make the exploration phase longer if it arrives more than $(\log T)^2$ during the $(\log T)^3$ first rounds. We have:
\begin{align}
\Pr\left(N_u ((\log T)^3)\geq (\log T)^2\right) &\leq \exp\left(-(\log T)^3 \kl \left( \frac{1}{\log T}  ,\frac{1}{m}\right)\right)\cr
&\leq \exp\left(-2(\log T)^3 \left( \frac{1}{\log T}  -\frac{1}{m}\right)^2\right)\cr
&\stackrel{(a)}\leq\exp\left(-2\left(\log T -\frac{2}{\log T}\right)\right)
\label{eqn:upperuserarrival}
\end{align}
where (a) is obtained from $m>(\log T)^3$ (remember that $T=o(m^2)$).

We deduce the probability that the duration of exploration phase $T_{\textrm{exp}}$ exceeds $(\log T)^3$,
\begin{align*}
\mathbb{P} (T_{\textrm{exp}}\ge (\log T)^3) &=\mathbb{P}(\exists u \textrm{ s.t. } N_u ((\log T)^3)\geq (\log T)^2)\cr
& \stackrel{(a)}\le \frac{(\log T)^3}{T^2} \exp\left(\frac{4}{\log T}\right)
\end{align*}
where (a) is obtained from the union bound and (\ref{eqn:upperuserarrival}).

Now the expected time taken in the exploration phase is,
\begin{align*}
\mathbb{E}[T_{\textrm{exp}}] 
&= (\log T)^3 + \mathbb{E}[T_{\textrm{exp}}|T_{\textrm{exp}}\ge (\log T)^3]\mathbb{P} (T_{\textrm{exp}}\ge (\log T)^3)\cr
&\le (\log T)^3 + \frac{(\log T)^3}{T} \exp\left(\frac{4}{\log T}\right)\cr
&= \set{O}((\log T)^3).
\end{align*}
Therefore, we can conclude that the expected regret that occurs in the exploration phase is $\set{O}((\log T)^3)$.

{\bf Clustering phase.} The performance of the clustering phase can be analyzed using the same arguments as in the proof of Theorem 6 in \cite{yun2016optimal}. To simplify the notation, let $\epsilon = (\log {T})^{-\frac{1}{4}}$. Recall that $Q_i = \{j\in {\cal S}: |\hat{\rho}_i-\hat{\rho}_j| \leq \epsilon \}$ for all $i \in {\cal S}$. We also define a set $\mathcal{B}_k$ for $1\leq k \leq K$ as:
 \begin{align*}
\mathcal{B}_k & = \{i\in {\cal S}: |p_k-\hat{\rho}_i| \leq \frac{1}{2}\epsilon \}.
 \end{align*}
This set has the following properties:
\begin{enumerate}[label=(\roman*), itemsep=-0.5ex]
\item $|\mathcal{B}_k|=\Omega((\log T)^2)$ with probability at least $1-\frac{1}{T}$. This follows from the following argument.
\begin{align*}
\Pr(i\in\mathcal{B}_k)&\geq \Pr(i\in\mathcal{B}_k|i\in\tilde{\mathcal{I}}_k)\Pr(i\in\tilde{\mathcal{I}}_k)\cr
&\stackrel{(a)}\geq \alpha_k(1-\epsilon_1) \Pr\left(|p_k-\hat{\rho}_i | \leq \frac{1}{2}\epsilon \Big|i\in\tilde{\mathcal{I}}_k\right)\cr
&\stackrel{(b)}{\geq} \alpha_k(1-\epsilon_1)\left(1-2\exp\left(-\frac{\sqrt{\log T}}{2}\right)\right),
\end{align*}%
where (a) follows from the assumption that $\mathcal{A}_1$ holds and (b) stems from Chernoff-Hoeffding's bound.
Let $r=\alpha_k(1-\epsilon_1)\left(1-2\exp\left(-\frac{\sqrt{\log T}}{2}\right)\right)$. Then,
\begin{align*}
\Pr\left(|\mathcal{B}_k|<\left(r-\frac{1}{\sqrt{2\log T}}\right)(\log T)^2\right) 
\leq& \exp\left(-(\log T)^2 \kl\left(r-\frac{1}{\sqrt{2\log T}},r\right)\right)\cr
\leq& \exp\left(-2(\log T)^2\left(\frac{1}{\sqrt{2\log T}}\right)^2\right)\cr
\leq& \frac{1}{T}.
\end{align*}
Therefore, $|\mathcal{B}_k|=\Omega((\log T)^2)$ with probability at least $1-\frac{1}{T}$.

\item $\big| {\cal S} \setminus (\cup_{k=1}^{K}\mathcal{B}_k)\big| = \set{O}(\log T)$ with probability at least $1-\frac{1}{T}$. To show this, we use a similar argument as in (i):
{\abovedisplayskip=5pt
\begin{align*}
\Pr(i \in S \setminus (\cup_{k=1}^{K}\mathcal{B}_k)) 
&\leq \sum_{k=1}^{K}\Pr(i\in \mathcal{I}_k)\Pr(i \in S \setminus (\cup_{k=1}^{K}\mathcal{B}_k)|i\in \mathcal{I}_k)\cr
&\leq \sum_{k=1}^{K}\Pr(i\in \mathcal{I}_k)\Pr\left(|p_k-\hat{\rho}_i| > \frac{1}{2} \epsilon\Big|i\in\mathcal{I}_k\right)\cr
&\leq \sum_{k=1}^{K}\Pr(i\in \mathcal{I}_k)2\exp\left(-\frac{\sqrt{\log T}}{2}\right)\cr
& = 2\exp\left(-\frac{\sqrt{\log T}}{2}\right)
\end{align*}}
Then, the probability that the size of $\big| {\cal S} \setminus (\cup_{k=1}^{K}\mathcal{B}_k)\big|$ is greater than $\log T$ is,
{\abovedisplayskip=5pt
\begin{align*}
& \Pr\left(\big| {\cal S} \setminus (\cup_{k=1}^{K}\mathcal{B}_k)\big| \geq \log T\right)
\\
&\leq \exp\left(-(\log T)^2 \kl \left(\frac{1}{\log T}, 2\exp\left(-\frac{\sqrt{\log T}}{2}\right)\right)\right) \cr
&  \stackrel{(a)} \leq \exp\bigg(-(\log T)\bigg(  \frac{\sqrt{\log T}}{4} +2\exp\left(-\frac{\sqrt{\log T}}{2}\right)-\log 2 -1\bigg)\bigg)\cr
& \leq \frac{1}{T},
\end{align*}}
where (a) is obtained from Lemma \ref{lem:kllower} when $\log T \ge 2^4$.
\item If $|\mathcal{B}_k\cap Q_i| \geq 1$, then $|\mathcal{B}_j\cap Q_i| =0$ for all $j$, $k$ such that $|p_k-p_j|=\Theta(1)$. Because $|\hat{\rho}_i-\hat{\rho}_l|\geq |p_k-p_j|-|p_k-\hat{\rho}_i|-|p_j-\hat{\rho}_l| \geq |p_k-p_j| - \epsilon$ for $i\in\mathcal{B}_k$ and $l\in\mathcal{B}_j$.

\item $\mathcal{B}_k \subset Q_i$ for all $i \in \mathcal{B}_k$, since $|\hat{\rho}_i-\hat{\rho}_j|\leq |\hat{\rho}_i-p_k|+|\hat{\rho}_j-p_k|\leq \epsilon$ for all $j \in\mathcal{B}_k $.
\end{enumerate}

From properties (iii) and (iv), there exists an item $i\in (\cup_{k=1}^{K}\mathcal{B}_k) \setminus (\cup_{\ell =1}^{k-1} Q_{i_\ell})$ such that $|Q_i \setminus (\cup_{\ell =1}^{k-1} Q_{i_\ell})| \geq m_k$ where $m_k$ is the $k$-th largest value among $\{|\mathcal{B}_1|,...,|\mathcal{B}_{K_0}|\}$ for $K_0=|\{p_k: 1\leq k \leq K\}|$.  Here, $m_k = \Omega ( (\log T^2)^2)$ from property (i).

We also have $|Q_v| = \set{O}(\log T)$ for $v$ such that $|Q_v \cap (\cup_{k=1}^{K}\mathcal{B}_k)) | = 0$ from property (ii). Thus, the item $v$ cannot be chosen as $i_k$.

We can conclude that $|{p}_k - \hat{p}_k|\leq \epsilon$ for $k=1,2$ with probability $1-2/T$, since $|\hat{\rho}_i - p_k| \le \epsilon$ when $|Q_i \cap \mathcal{B}_k| \ge 1$. Hence as before for event ${\cal A}_1$, we can assume that $\mathcal{A}_2=\{|{p}_k - \hat{p}_k|\leq \epsilon\ \textrm{for }k=1,2\}$ holds in the remaining of the proof.

{\bf Test phase.} After $n$ recommendations of an item $i$ from $\mathcal{I}_k\neq \mathcal{I}_1$, the probability that $i$ passes the test is,
\begin{align*}
\Pr\left(\hat{\rho}_i>\hat{p}_1-\Delta_0 /2\right) & = \Pr\left(\hat{\rho}_i>\frac{1}{2}(\hat{p}_1+\hat{p}_2)\right)\\
& \stackrel{(a)}\leq \Pr\left(\hat{\rho}_i>\frac{1}{2}({p}_1+{p}_2)-\epsilon\right)\\
& \leq \exp\left(-n\kl\left(\frac{1}{2}({p}_1+{p}_2)-\epsilon,p_k\right)\right)\\
& \leq \exp\left(-2n\left(\frac{1}{2}(p_1+p_2)-p_k-\epsilon\right)^2\right)
\end{align*}
where (a) is obtained from the assumption that $\mathcal{A}_2$ holds.

To simplify the notation, let $x=\lfloor \frac{2\log 3}{\Delta_0^2} \rfloor$. Since we test the item after every $x$ recommendations, we have at most $m/x$ tests for each item. Therefore, the expected number of times a sub-optimal item $j$ is recommended is:
\begin{align}
\EXP\Big[N(j)\Big] & = x + \sum_{i=1}^{m/x}x\Pr\left(\hat{\rho}_k>\hat{p}_1-\Delta_0 /2 \textrm{ for } i \textrm{-th test}\right)\cr
& \leq x + \sum_{i=1}^{m/x}x\exp\left(-2ix\left(\frac{1}{2}(p_1+p_2)-p_k-\epsilon\right)^2\right)\cr
& \leq x + \frac{2}{({p}_1+{p}_2-2p_k-2\epsilon)^2}\cr
& \leq \frac{4\log 3}{({p}_1-{p}_2)^2} + \frac{2}{({p}_1+{p}_2-2p_k-2\epsilon)^2}. \label{eq:numsubopt}
\end{align}

Furthermore, the probability that item $i\in I_1$ is not removed until the last test is,
\begin{align}
\Pr\left(\bigcap_{t=1}^{m/x}\{\hat{\rho}_i>\hat{p}_1-\Delta_0 /2 \textrm{ for } t\textrm{-th test} \}\right) &\geq 1 - \sum_{t=1}^{m/x} \Pr\left(\hat{\rho}_i \leq\hat{p}_1-\Delta_0 /2 \textrm{ for } t\textrm{-th test}\right)\cr
& = 1 - \sum_{t=1}^{m/x} \Pr\left(\hat{\rho}_i<\frac{1}{2}({p}_1+{p}_2)-\epsilon\right)\cr
& \geq 1- \sum_{t=1}^{m/x} \exp\left(-tx\kl\left(\frac{1}{2}({p}_1+{p}_2)-\epsilon,p_1\right)\right)\cr
& \geq 1- \sum_{t=1}^{m/x} \exp\left(-\frac{1}{2}tx\left(p_1-p_2+2\epsilon\right)^2\right)\cr
& \geq 1- \frac{\exp\left(-\frac{1}{2}x\left(p_1-p_2+2\epsilon\right)^2\right)}{1-\exp\left(-\frac{1}{2}x\left(p_1-p_2+2\epsilon\right)^2\right)}\cr
& \geq \frac{1}{2}. \label{eq:optimallower}
\end{align}

By \eqref{eq:optimallower}, if we assume that user arrives for $\overline{N}$ times at most, we need at most $2\overline{N}$ optimal items in $\mathcal{V}$ in expectation. Thus, the required number of new samples from $\mathcal{I}\setminus \mathcal{V}_0$ is less than $\frac{2\overline{N}}{\alpha_1 }$. Therefore, from \eqref{eq:numsubopt}, the expected regret that occurs in the test phase under the assumption that every user arrives for less than $\overline{N}$ times is,
\begin{align*}
	\set{O}\Bigg(\frac{2\overline{N}}{\alpha_1 }\sum_{k=2}^{K}\alpha_k(p_1-p_k)\Bigg(\frac{4\log 3}{({p}_1-{p}_2)^2}+ \frac{2}{({p}_1+{p}_2-2p_k)^2}\Bigg)\Bigg).
\end{align*}
On the other hand, the regret due to more than $\overline{N}$ arrivals of users is $\frac{m}{m(e-1)}=O\left(1\right)$ by Lemma \ref{lem:couponcollector}.

Finally, the expected regret of ECT satisfies:
\begin{align*}
&R^{ECT}(T) = \set{O}\Bigg((\log T)^3 +\frac{2\overline{N}}{\alpha_1 } \sum_{k=2}^{K} \Bigg[\alpha_k(p_1-p_k) \Bigg(\frac{4\log 3}{({p}_1-{p}_2)^2}
+ \frac{2}{({p}_1+{p}_2-2p_k)^2}\Bigg)\Bigg]\Bigg).
\end{align*}
\ep

\newpage
\section{Performance guarantees of ET: Proof of Theorem \ref{thm:modelBupper}}\label{app:upperB}

Recall that $\mu_x$ is the expected reward such that $\mathbb{P} (\rho_i \le \mu_x) = x$, and that we are interested in the satisficing regret defined by:
$$R_\varepsilon^\pi (T) =  \mathbb{E}^\pi \left( \sum_{t=1}^T \max \{ 0, \mu_{1-\varepsilon} - \rho_{i^\pi_t} \} \right). $$
We consider the case where $\varepsilon  \ge C \sqrt{\frac{\pi}{2\log T}}$.
Further recall that we assume that $\zeta (x) \le C$ for all $x \in [0,1]$.

To prove Theorem \ref{thm:modelBupper}, we first analyze the performance of the exploration phase, and in particular show that $\hat{\mu}_{1-\frac{\varepsilon}{2}}$ is very close to $\mu_{1-\varepsilon}$. We then study the regret generated during the test phase.

{\bf Exploration Phase.} We first derive an upper and a lower bound of $\hat{\mu}_{1-\frac{\varepsilon}{2}}$. Here, we use the fact that for all $i \in \mathcal{S}$, the $\hat{\rho}_i$'s are i.i.d. random variables.

From Chernoff bound and Pinsker's inequality,
\begin{align}
\mathbb{P}(\hat{\rho}_i \ge \mu_{1-\frac{\varepsilon}{4}}) \le & \frac{\varepsilon}{4} + \int_0^{\mu_{1-\frac{\varepsilon}{4}}} \exp \left( -2^4 \log (T) \kl (\mu_{1-\frac{\varepsilon}{4}} , \mu) \right) \zeta(\mu) d\mu \cr
\le & \frac{\varepsilon}{4} + \int_0^{\mu_{1-\frac{\varepsilon}{4}}} \exp \left( -2^5  (\mu_{1-\frac{\varepsilon}{4}} - \mu)^2 \log (T)  \right) \zeta (\mu) d\mu \cr
\le & \frac{\varepsilon}{4} + \int_0^{\infty}  C \exp \left( -2^5  x^2 \log (T)  \right)  dx \cr
\le & \frac{\varepsilon}{4} + \frac{\varepsilon}{8} = \frac{3\varepsilon}{8}, \label{eq:rholow}
\end{align}
where for the last inequality, we use the Gaussian integral $\int_{-\infty}^{\infty} e^{-x^2} dx  = \sqrt{\pi}$. When $\varepsilon  \ge C \sqrt{\frac{\pi}{2\log T}}$,
\begin{align*}
 \int_0^{\infty}  C \exp \left( -2^5  x^2 \log (T)  \right)  dx  = & \frac{1}{2}\int_{-\infty}^{\infty}  \frac{C}{\sqrt{2^5 \log (T)}} \exp \left( -  x^2   \right)  dx  \cr
 = & \frac{1}{8}  C \sqrt{\frac{\pi}{2\log T}} \le \frac{\varepsilon}{8}.
\end{align*}

Similarly,
\begin{align}
\mathbb{P}(\hat{\rho}_i \le \mu_{1-\frac{3\varepsilon}{4}}) \le &1- \frac{3\varepsilon}{4} + \int_{\mu_{1-\frac{3\varepsilon}{4}}}^1 e^{ -2^4 \log (T) \kl (\mu_{1-\frac{3\varepsilon}{4}} , \mu) } \zeta (\mu)d\mu \cr
\le &1- \frac{3\varepsilon}{4} + \int_{\mu_{1-\frac{3\varepsilon}{4}}}^1 e^{ -2^5 (\mu_{1-\frac{3\varepsilon}{4}} - \mu)^2 \log (T) } \zeta (\mu) d\mu \cr
\le & 1- \frac{3\varepsilon}{4} + \int_0^{\infty}  C \exp \left( -2^5  \mu^2 \log (T)  \right)  d\mu \cr
< &1-  \frac{3\varepsilon}{4} + \frac{\varepsilon}{8} \le 1-  \frac{5\varepsilon}{8}. \label{eq:rhoupper}
\end{align}

From the Chernoff-Hoeffding and \eqref{eq:rholow},
\begin{align}
\mathbb{P}\left( \left| \{ i\in \mathcal{S} : \hat{\rho}_i \ge \mu_{1-\frac{\varepsilon}{4}} \} \right| \ge \frac{\varepsilon}{2} |\mathcal{S}| \right) \le & \exp \left( -|\mathcal{S}| \kl (\frac{\varepsilon}{2} , \frac{3\varepsilon}{8}) \right)\cr
\le & \exp \left( -2 |\mathcal{S}| \frac{\varepsilon^2}{8^2} \right) \le \frac{1}{T^2}. \label{eq:Slower}
\end{align}

From the Chernoff-Hoeffding and \eqref{eq:rhoupper},
\begin{align}
\mathbb{P}\left( \left| \{ i\in \mathcal{S} : \hat{\rho}_i \le \mu_{1-\frac{3\varepsilon}{4}} \} \right| \le (1-\frac{\varepsilon}{2} )|\mathcal{S}| \right) \le & \exp \left( -|\mathcal{S}| \kl (1-\frac{\varepsilon}{2} , 1-\frac{5\varepsilon}{8}) \right)\cr
\le & \exp \left( -2 |\mathcal{S}| \frac{\varepsilon^2}{8^2} \right) \le \frac{1}{T^2}. \label{eq:Supper}
\end{align}

We conclude from \eqref{eq:Slower} and \eqref{eq:Supper} that with probability $1-\frac{2}{T^2}$, we have 
\begin{equation}
\mu_{1- \frac{3\varepsilon}{4}} \le \hat{\mu}_{1-\frac{\varepsilon}{2}} \le \mu_{1- \frac{\varepsilon}{4}}. \label{eq:muhalf}
\end{equation}

Further observe that using the same arguments as those used to upper bound the duration of the exploration phase of the ECT algorithm, the expected duration, and hence the expected regret, of the exploration phase in ET is $O({(\log T)^2\over \varepsilon^2})$.

{\bf Test Phase.} For convenience, let $\Delta =  \log\log_2( 2^e m^2)$. Then, ET runs at most $\tau = \lfloor \log_2 (\frac{m}{\Delta}) \rfloor$ tests for each item. We define the distance $\mathcal{D}$ between two Bernoulli distributions as follows:
$$\mathcal{D} (p,q) = \kl (s,p) =\kl (s,q) \quad\mbox{with}\quad s=\frac{\log\frac{1-p}{1-q}}{\log\frac{q(1-p)}{p(1-q)}} \quad\mbox{for}\quad p \neq q.$$

Let $m^{\pi}(\mu)$ be the expected number of users to whom a randomly selected item with parameter $\mu$ is recommended. Let $\rho^{(\ell)} (\mu)$ be the random value $\hat{\rho}_i$ of item $i$ having $\mu$ after $\lfloor 2^{\ell} \Delta \rfloor$ observations.

Consider items having $\mu$ such that $ \mu \le \hat{\mu}_{1-\frac{\varepsilon}{2}}$ and $2^{-\ell } \le \mathcal{D}( \mu , \hat{\mu}_{1-\frac{\varepsilon}{2}}) $. Then, 
 $\kl (\bar{\rho}^{(\ell)} , \mu) \ge 2^{-\ell} $ and we have 
\begin{align}
m^{\pi}(\mu) \le & 2^{\ell}\Delta + \sum_{r=\ell}^\tau 2^{r+1} \Delta \mathbb{P}\left(\hat{\rho}^{(r)} (\mu) \ge  \bar{\rho}^{(r)}  \right) \cr
\le & 2^{\ell}\Delta + \sum_{r=\ell}^\tau 2^{r+1} \Delta \mathbb{P}\left(\hat{\rho}^{(r)} (\mu) \ge  \bar{\rho}^{(\ell)}  \right) \cr
\le & 2^{\ell}\Delta + \sum_{r=\ell}^\tau 2^{r+1} \Delta \exp \left( -(2^r \Delta) \kl (\bar{\rho}^{(\ell)},\mu ) \right) \cr
\le & 2^{\ell}\Delta + \sum_{r=\ell}^\tau 2^{r+1} \Delta  \exp \left( -2^{r-\ell} \Delta  \right)\cr
\le & 2^{\ell+2}\Delta .\label{eq:mpimu}
\end{align}

From \eqref{eq:mpimu},
\begin{equation}
m^{\pi}(\mu) \le 
\begin{cases}
2^{3}\Delta & \mbox{for}~ 2^{-1 } \le \mathcal{D}( \mu , \hat{\mu}_{1-\frac{\varepsilon}{2}})  \\
\frac{2^{3}\Delta}{\mathcal{D}( \mu , \hat{\mu}_{1-\frac{\varepsilon}{2}})}& \mbox{for}~ 2^{-\ell } \le \mathcal{D}( \mu , \hat{\mu}_{1-\frac{\varepsilon}{2}}) \le 2^{-\ell +1}  \\
m & \mbox{for}~ \mathcal{D}( \mu , \hat{\mu}_{1-\frac{\varepsilon}{2}}) \le 2^{-\tau}
\end{cases}
\end{equation}

Next we study the expected regret generated by recommending a newly sampled item. From the regret definition and  \eqref{eq:mpimu},
\begin{align}
\frac{1}{8\Delta}\int_{0}^{\mu_{1-\varepsilon}} (\mu_{1-\varepsilon} - \mu) m^{\pi}(\mu) \zeta (\mu) d\mu \le & \int_{0}^{\mu_{1-\varepsilon}} (\mu_{1-\varepsilon} - \mu)\left(1+  \frac{1}{ \mathcal{D} ( \mu , \mu_{1-\frac{3\varepsilon}{4}}) } \right)\zeta (\mu) d \mu \cr
\le & \int_{0}^{\mu_{1-\varepsilon}} C(\mu_{1-\varepsilon} - \mu)\left(1+  \frac{2}{ ( \mu - \mu_{1-\frac{3\varepsilon}{4}})^2 } \right) d \mu \cr
\le & \int_{0}^{\mu_{1-\varepsilon}} C(\mu_{1-\varepsilon} - \mu)\left(1+  \frac{2}{ ( \frac{\varepsilon}{4C}  +\mu_{1-\varepsilon} - \mu)^2 } \right) d \mu \cr
\le & \frac{C}{2} + \int_{0}^{\mu_{1-\varepsilon}}  \frac{2C}{ \frac{\varepsilon}{4C}  +\mu_{1-\varepsilon} - \mu }  d \mu \cr
\le & \frac{C}{2} + \log (4C /\varepsilon),
 \label{eq:averegperitem}
\end{align}
where the second inequality stems from Pinsker's inequality $2(p-q)^2 \le \kl(p,q) $ and the definition of $\mathcal{D}$, and the third inequality uses the assumption $\zeta (\mu) \le C$.

If an item has a parameter $\mu \ge \hat{\mu}_{1-\frac{\varepsilon}{2}}$, we do not remove it from $\mathcal{V}$ with probability at least
\begin{align*}
\mathbb{P} \left( \bigcap_{\ell=1}^\tau \left\{ \hat{\rho}^{(\ell)}(\mu) > \bar{\rho}^{(\ell)} \right\} \right) 
\ge &  1-  \sum_{\ell=1}^\tau  \mathbb{P}\left(\hat{\rho}^{(\ell)}(\mu) \le \bar{\rho}^{(\ell)} \right) \cr
\ge& 1-  \sum_{i=1}^\tau  e^{ - \Delta  }  \cr
= & 1-  \frac{1}{\log_2(2^e m^2)} \frac{1 - e^{ -\tau \Delta}}{1 - \frac{1}{\log_2(2^e m^2)}} \cr
\ge&  1-  \frac{1 }{\log_2(2^e m^2) - 1} \cr
\ge&\frac{1}{2}.
\end{align*}
To recommend items $i$ with parameters $\mu_i \ge \mu_{1-\varepsilon}$ to the $\overline{N}$ arrivals of every user, we then need, on average, $\frac{2\overline{N}}{\varepsilon }$ sampled items. From \eqref{eq:averegperitem}, we conclude that the satisficing regret of ET satisfies:
$$R^\pi_\varepsilon (T) = \set{O}\left( \frac{\overline{N} \log(e/\varepsilon) \log\log(m)}{\varepsilon }   + \frac{(\log T)^2}{\varepsilon^2}  \right). $$
\ep

\newpage

\section{Performance guarantees of EC-UCS: Proof of Theorem \ref{thm:newalgC}}\label{app:upperC}
These two last sections \ref{app:upperC} and \ref{app:upperECB} of the appendix are devoted to the analysis of the regret of EC-UCS and ECB in systems with clustered items and users. The two algorithms share the same initial phase to cluster items. The next subsection is hence devoted to the analysis of this item clustering phase. Then, we present an analysis of the performance of the other phases of EC-UCS, and conclude this section with the statement and proof of lemmas used in the analysis of EC-UCS.

\addtocontents{toc}{\protect\setcounter{tocdepth}{1}}
\subsection{Clustering items in EC-UCS and ECB}
\addtocontents{toc}{\protect\setcounter{tocdepth}{2}}  

The exploration phase for item clustering is of duration $10m$, and hence induces a regret upper bounded by $10m$. In what follows, we just investigate the quality of the item clusters that result from this phase. 

Recall that the algorithm randomly selects a set  $\mathcal{S}$ of items to cluster. We denote by $V_1,\dots,V_K$ the true cluster $\mathcal{S} \cap \mathcal{I}_1, \dots, \mathcal{S} \cap \mathcal{I}_K$, respectively, and assume that $ m^2 \ge T (\log T)^3$ and $n = \omega(\log T)$. We let $n_0 := \min\{ n, \frac{m}{(\log T)^2} \}$ be the number of  sampled items. For each $k$, the size of $V_k$ concentrates around $\alpha_k n_0$. Indeed, from the Chernoff-Hoeffding's inequality, 
\begin{align}
&\mathbb{P}\left(\left| |V_k| - \alpha_k n_0 \right| \ge \sqrt{ n_0 \log T}  \right) \le \frac{2}{T^2}. \label{eq:clustersize}
\end{align}
Since $n_0 = \omega(\log T)$, we have 
\begin{equation}
|V_k| = \alpha_k n_0 (1+ o(1)) \quad \mbox{for all}\quad k \in [K].
\end{equation}
Then, $|V_k| \ge \overline{N}$  for all $1\le k \le K$ since $n_0 = \omega ( \log (m) + \frac{T}{m})$. Therefore, all users arriving after the exploration phase could be potentially recommended by items from a single cluster $V_k$ without repetition. 

Recall the procedure used by EC-UCS to cluster items in ${\cal S}$. For the $10m$ first user arrivals, it recommends items from $\mathcal{S}$ uniformly at random. These $10m$ recommendations and the corresponding user responses are recorded in the dataset $\mathcal{D}$. From the dataset $\mathcal{D}$, the item clusters are extracted using a spectral algorithm (see Algorithm \ref{alg:SIC_B}). This algorithm is taken from \cite{yun2014streaming}, and considers the {\it indirect edges} between items created by users. Specifically, when a user appears more than twice in ${\cal D}$, she creates an indirect edge between the items recommended to her for which she provided the same answer (1 or 0). Items with indirect edges are more likely to belong to the same cluster.

Algorithm~\ref{alg:SIC_B} builds an adjacency matrix $A$ from indirect edges. From Lemma ~\ref{lem:m2} (presented in Appendix \ref{app:preliminaries}), we know that at least $m/2$ users arrive twice in the first $10m$ arrivals with probability at least $1-{1\over T}$. We conclude that the construction of $A$ is equivalent to a stochastic block model with random sampling where the number of vertices is $n_0 $, the sampling budget is $s \ge m/2$. We establish in the next theorem that this budget is enough to reconstruct the clusters $V_1,\dots,V_K$ exactly using Algorithm~\ref{alg:SIC_B}. Theorem~\ref{thm:spec} is proved in Appendix \ref{app:itemcluster}.

\begin{theorem} Let $\hat{I}_1,\dots,\hat{I}_K$ be the output of Algorithm~\ref{alg:SIC_B}.
	With probability $1-\frac{1}{T}$, there exists permutation $\Gamma$ such that
	$$\left| \bigcup_{k=1}^K (\hat{I}_{\Gamma (k)} \setminus V_k) \right| = 0. $$ \label{thm:spec}
\end{theorem}

\addtocontents{toc}{\protect\setcounter{tocdepth}{1}}
\subsection{Regret of EC-UCS: Proof of Theorem \ref{thm:newalgC}}
\addtocontents{toc}{\protect\setcounter{tocdepth}{2}}  

The first component of the regret of EC-UCS is generated during the exploration phase for item clustering. This component is ${\cal O}(m)$. Then in view of Theorem \ref{thm:spec}, errors in item clustering cannot generate more than a ${\cal O}(1)$ regret. Hence, in what follows, we always assume that after the item clustering phase, we have:
$$\left| \bigcup_{k=1}^K (\hat{I}_{\Gamma (k)} \setminus V_k) \right| = 0. $$
Without loss of generality, we assume that $\Gamma(k)= k$ in the remaining of this section. After the item clustering phase, there are four sources of regret referred to as: 1. Exploration for user clustering, 2. Arrival of reference users, 3. User clustering, and 4. Optimistic assignments.

{\bf 1. Exploration for user clustering.} The regret induced by exploration of the users in ${\cal U}_0$ until $t\le (10+\log T)m$ is $\frac{m\log T}{\log T}=m$. Hence, the regret due to this step is:
\begin{align*}
{\cal O}\left(m \sum_{\ell} \beta_\ell (p_{\sigma_\ell(1)\ell}- p_{\sigma_\ell(K)\ell})\right).
\end{align*}

{\bf 2. Arrival of reference users.} If the users in ${\cal U}^*$ have not arrived enough times, the algorithm cannot cluster them as intended, and this generates regret. Let $n_u$ denote the number of times user $u$ has arrived (until a time that will always be specified).\\
We define the event $\mathcal{E}_{top}^{(i)} =\{ \exists u \in {\cal U}^* $ such that $n_u \le \frac{(9+2^i)}{2}$ at $t = (9+2^i)m\}$ for $0\le i \le \log_2\log T$. Then, by Lemma \ref{lem:algCtop}, the regret due to $\mathcal{E}_{top}^{(i)}$ until $t=\lfloor (10+\log T)m\rfloor$ is,
\begin{align}
&R_{ref}(\lfloor (10+\log T)m\rfloor)\cr
&\le \sum_{\ell}\beta_\ell (p_{\sigma_\ell(1)\ell}- p_{\sigma_\ell(K)\ell}) \sum_{i=0}^{\lfloor \log_2\log T \rfloor} 2^{i+1} m \left(\frac{me}{(\log T)^3}\right)^{(\log T)^2}  \exp\left(-\frac{(9+2^i)m}{16\log T}\right)\cr
& \le \sum_{\ell}\beta_\ell (p_{\sigma_\ell(1)\ell}- p_{\sigma_\ell(K)\ell}) m \left(\frac{me}{(\log T)^3}\right)^{(\log T)^2} \int_{-1}^{\log_2\log T} 2^{x+1} \exp\left(-\frac{(9+2^x)m}{16\log T}\right)dx \cr
&=\sum_{\ell}\beta_\ell (p_{\sigma_\ell(1)\ell}- p_{\sigma_\ell(K)\ell}) m \left(\frac{me}{(\log T)^3}\right)^{(\log T)^2} \int_{1/2}^{\log T} \frac{2}{\log 2} \exp\left(-\frac{(9+y)m}{16\log T}\right)dy\cr
&=\sum_{\ell} \beta_\ell (p_{\sigma_\ell(1)\ell}- p_{\sigma_\ell(K)\ell}) m \left(\frac{me}{(\log T)^3}\right)^{(\log T)^2}   \Bigg[-\frac{32\log T}{m\log 2}\exp\left(-\frac{(9+y)m}{16\log T}\right)\Bigg]_{1/2}^{\log T}\cr
&= \sum_{\ell} \beta_\ell (p_{\sigma_\ell(1)\ell}- p_{\sigma_\ell(K)\ell})\exp\left(-\Theta\left(\frac{m}{\log T}\right)\right).
\label{eqn:refregret}
\end{align}
where we have used the assumption $m^2 \ge T(\log T)^3$.\\
Also, $\Pr(\mathcal{E}_{top}^{(\lfloor \log_2\log T \rfloor)}) \le 2\left(\frac{me}{(\log T)^3}\right)^{(\log T)^2}  \exp\left(-\frac{(9+\frac{1}{2}\log T)m}{4\log T}\right) = \exp\left(-\Theta(m)\right)$. Hence, in view of  (\ref{eqn:refregret}), the regret due to $\mathcal{E}_{top}^{(i)}$ satisfies
\begin{align}
R_{ref}(T) \le \sum_{\ell} \beta_\ell (p_{\sigma_\ell(1)\ell}- p_{\sigma_\ell(K)\ell})\exp\left(-\Theta\left(\frac{m}{\log T}\right)\right).
\label{eqn:topregret}
\end{align}

Let $\mathcal{B}_1=\left(\bigcup_{i=0}^{\lfloor \log_2\log T \rfloor} \mathcal{E}_{top}^{(i)}\right)^c$, i.e., $\mathcal{B}_1$ correspond to the event where every $u\in {\cal U}^*$ has arrived $n_u > \frac{(9+2^i)}{2}$ times at $t = (9+2^i)m$ for all $i$. Then, from \eqref{eqn:topregret}, the regret because of $\mathcal{B}_1^c$ is
\begin{align}
R_{ref}(T) \le \exp\left(-\Theta\left(\frac{m}{\log T}\right)\right).
\label{eqn:topregret2}
\end{align}

{\bf 3. User clustering.} 
The size of ${\cal U}^*$ is sufficiently large, so that ${\cal U}^*$ consists of users from all clusters. More precisely, for a well-chosen $\epsilon_1>0$, Lemma \ref{lem:algCsize} states that with probability $1-1/T$,  the size of $\tilde{\mathcal{U}}_\ell = {\cal U}^*\cap \mathcal{U}_\ell$ is greater than  $\beta_\ell(1-\epsilon_1) (\log T)^2$ for all $\ell$. Let $\mathcal{B}_2=\{|\tilde{\mathcal{U}}_\ell|\geq\beta_\ell(1-\epsilon_1) (\log T)^2 \textrm{ for } 1\leq\ell\leq L\}$.
By Lemma \ref{lem:algCsize}, the expected regret due to the event 
$\mathcal{B}_2^c$ is ${\cal O}(1)$. 

We now assume that both $\mathcal{B}_1$ and $\mathcal{B}_2$ holds throughout the remaining of the proof.

Under ${\cal B}_1$, we have numerous observations for users in ${\cal U}^*$. Hence, most of users in ${\cal U}^*$ have their empirical average success rate vector concentrated around the true parameter vector when $t$ is large. Therefore, under $\mathcal{B}_2$, the clustering step can learn the hidden parameters very accurately (since there are clear user clusters). We formalize this observation below. Consider $t \ge T_0=\lceil Cm \rceil$ where
$$C=\max\left(\frac{512K^{3}}{\min(y_{\ell r},\delta)^2}\log\left(\frac{16K^{\frac{3}{2}}}{\min(y_{\ell r},\delta)}\right), \frac{2\sqrt{K}}{\min_\ell \beta_\ell}\right) .$$
Then, we have
\begin{itemize}
\item[\em{(C1)}] $K\sqrt{\frac{8Km}{T_0}\log\frac{T_0}{m}} < \frac{1}{4}\min(y_{\ell r} ,\delta)$,  
\item[\em{(C2)}] $(1-\epsilon_1)\left(1-2K\left(\frac{m}{T_0}\right)^2\right)\min_\ell \beta_\ell>\frac{m}{T_0}$.
\end{itemize}
Recall that in EC-UCS (see the pseudo-code), we use a parameter $\epsilon >0$ when clustering users. This parameter is fixed and equal to $\epsilon = K\sqrt{\frac{8Km}{t}\log\frac{t}{m}}$.
From Lemma~\ref{lem:algoCcluster}, under {\em (C1)}  and {\em (C2)}, we have  
$$ \|{p}_\ell - \hat{p}_\ell\| < \epsilon <  \frac{1}{4}\min(y_{\ell r} ,\delta) \quad \forall t \ge T_0 \quad \mbox{with probability}\quad 1-\frac{2}{T}. $$  

Hence after $T_0$ rounds, the algorithm has accurate estimates of the parameters. We include $T_0$ in the regret upper bound, but can then assume that $\|{p}_\ell - \hat{p}_\ell\| < \epsilon <  \frac{1}{4}\min(y_{\ell r} ,\delta)$ for all $t\ge T_0$ in the remaining of the proof. The expected regret generated by the complement of this event is ${\cal O}(1)$.

{\bf 4. Optimistic assignments.} When $\|{p}_\ell - \hat{p}_\ell\| < \epsilon$, the algorithm can exploit the learned parameters. Suppose $u_t \in \mathcal{U}_\ell$. Recall the notation: $\mathcal{L}(u_t) \gets \{\ell\in [L_0]: \sum_{k=1}^K n_{k u_t}x_{k \ell}^2 < 2K \log n_{u_t}\}$ used in the pseudo-code of EC-UCS. 

Since the probability of the event $\{\ell \not\in \mathcal{L}(u_t)\}$ decreases rapidly with the number of arrivals of $u_t$, the regret induced by this event is ${\cal O}(1)$. Since $\ell \in\mathcal{L}(u_t)$ holds most of the time, the algorithm recommend optimal items in item cluster $k_\ell^*$ at least $\frac{n_{u_t}}{2K}$ times.  If $r\not\in\mathcal{L}^{\perp}(\ell)$, we can distinguish $\mathcal{U}_\ell$ from $\mathcal{U}_r$ with the constant number of recommendations of optimal item $k_\ell^*$, since ${p}_{k_\ell^* \ell}\neq{p}_{k_\ell^* r}$. On the other hand, if $r\in\mathcal{L}^{\perp}(\ell)$, the algorithm cannot distinguish them unless it plays suboptimal items. Actually, suboptimal items should be played at most ${\cal O}(\log \overline{N})$ times in expectation. We make the above observations precise in Lemma \ref{lem:algCN}, from which we conclude that the regret generated in this phase is:
\begin{align*}
{\cal O}\left(m\sum_{\ell} \beta_\ell (p_{\sigma_\ell(1)\ell}- p_{\sigma_\ell(K)\ell})     \left(\sum_{r\in\mathcal{R}_\ell\setminus\mathcal{L}^{\perp}(\ell)} \frac{K^2\log K}{\phi(|{p}_{k_\ell^* r}-{p}_{k_\ell^* \ell}|^2)} +\sum_{k\in \mathcal{S}_{\ell r}}\sum_{r\in\mathcal{L}^{\perp}(\ell)} \frac{K\log \overline{N}}{|\mathcal{S}_{\ell r}||{p}_{k \ell}-{p}_{kr}|^2}\right)\right)
\end{align*}

Overall, accounting for all the regret sources, we have established that the regret of EC-UCS is:
\begin{align*}
&\set{O} \left( m \sum_{\ell} \beta_\ell (p_{\sigma_\ell(1)\ell}- p_{\sigma_\ell(K)\ell}) \left(\max \left(\frac{K^{3}\log K}{ \phi(\min(y_{\ell r},\delta)^2)}, \frac{\sqrt{K}}{\min_\ell \beta_\ell}\right) \right.\right. \cr
&\qquad\qquad\qquad\qquad + \left.\left. \sum_{r\in\mathcal{R}_\ell\setminus\mathcal{L}^{\perp}(\ell)} \frac{K^2\log K}{ \phi(|{p}_{k_\ell^* r}-{p}_{k_\ell^* \ell}|^2)}+ \sum_{k\in \mathcal{S}_{\ell r}} \sum_{r\in\mathcal{L}^{\perp}(\ell)} \frac{K\log \overline{N}}{|\mathcal{S}_{\ell r}||{p}_{k \ell}-{p}_{kr}|^2}\right)\right)
\end{align*}

\addtocontents{toc}{\protect\setcounter{tocdepth}{1}}
\subsection{Technical lemmas for the proof of Theorem \ref{thm:newalgC}}
\addtocontents{toc}{\protect\setcounter{tocdepth}{2}}

\begin{lemma}[Cramer's theorem]
 For any i.i.d. sequence $(X_{i})_{i\ge 1}$ of real r.v. and any closed set $F \subseteq {\rm I\!R}$,
\begin{align*}
\Pr\left(\frac{1}{n}\sum_{i=1}^{n}X_{i}\in F\right)\leq 2\exp\left(-n\inf_{x\in F}I(x)\right),
\end{align*}
where $I(a)=\sup_{\theta\in{\rm I\!R}}(\theta a-\log{E[e^{\theta X}]})$.
\label{lem:cramer}
\end{lemma}

\begin{lemma}
In the 'Exploitation' step (see EC-UCS pseudo-code), the expected regret generated by the exploration of $u_t \in {\cal U}_0$ until $t = \lfloor(10+m)\log T\rfloor$ is ${\cal O}(m)$.
\end{lemma}

\begin{proof}
Since the probability that a user from ${\cal U}_0$ arrives for each time $t$ is $\frac{1}{\log T}$, the regret induced when exploring for users in ${\cal U}_0$ is ${\cal O}(\frac{m\log T}{\log T}) = {\cal O}(m)$.
\end{proof}

\begin{lemma}
In the 'Exploitation' step, with probability $1-2\left(\frac{me}{(\log T)^3}\right)^{(\log T)^2}  \exp\left(-\frac{t}{16\log T}\right)$, at least $\lfloor(\log T)^2\rfloor$ users in ${\cal U}_0$ arrive at least $\lfloor\frac{t}{2m}\rfloor$ times within the first $t$ arrivals.
\label{lem:algCtop}
\end{lemma}

\begin{proof} We denote by $N_u(t)$ the number of times a user $u$ has arrived in the first $t$ arrivals. For any set $A\subset {\cal U}_0$, let $N_A(t)$ denote the total number of arrivals of users in $A$ among the first $t$ arrivals.

We write the probability that less than $(\log T)^2$ users in ${\cal U}_0$ arrive for $t/2m$ times in the $t$ first arrivals as:
\begin{align}
&\mathbb{P}[\sum_{u\in {\cal U}_0} \indicator{\{ N_u(t)\ge \frac{t}{2m}\} }< (\log T)^2]
 = \mathbb{P}[\sum_{u\in {\cal U}_0} \indicator{\{ N_u(t)< \frac{t}{2m}\} }\ge |{\cal U}_0| - (\log T)^2] \nonumber\\
&\le \mathbb{P}[\exists A \subset {\cal U}_0: |A|=|{\cal U}_0| - (\log T)^2, \forall u\in A, N_u(t)\le \frac{t}{2m}] \nonumber\\
&\le \mathbb{P}[\exists A\subset {\cal U}_0: |A|=|{\cal U}_0| - (\log T)^2, N_A(t)\le \frac{t(|{\cal U}_0| - (\log T)^2)}{2m}] \label{eqn:algCtop}\\
&\stackrel{(a)}\le 2{|{\cal U}_0| \choose  (\log T)^2} \exp\left(-\frac{t}{4}\log \left(1+\frac{p}{2-2p}\right)\right) \nonumber\\
&\stackrel{(b)}\le 2\left(\frac{me}{(\log T)^3}\right)^{(\log T)^2}  \exp\left(-\frac{t}{16\log T}\right),\nonumber
\end{align}
where (b) follows from $\log(1+x) >\frac{x}{2}$ for $0<x<1$ and (a) can be proved using lemma \ref{lem:cramer}. For simplicity, we define i.i.d. random variables $X_i \sim \textrm{Bern}(p)$ where $p=\frac{|{\cal U}_0|-(\log T)^2}{m}$. Then, $I(a)=a\log \frac{a(1-p)}{p(1-a)} - \log \frac{1-p}{1-a}$. Since $I(a)$ is a decreasing function in $(-\infty,p]$, $\inf_{a\le\frac{p}{2}} I(a)= I\left(\frac{p}{2}\right)$. Therefore, (\ref{eqn:algCtop}) can be rewritten as:
\begin{align*}
\Pr\left(\frac{1}{t}\sum_{i=1}^{t}X_{i}\le\frac{p}{2}\right)
&\leq 2\exp\left(-t I\left(\frac{p}{2}\right)\right)\cr
&= 2\exp\left(-t\left(\log \frac{2-p}{2-2p} + \frac{p}{2}\log\frac{1-p}{2-p}\right)\right)\cr
&\stackrel{(a)}\le 2\exp\left(-\frac{t}{4}\log \frac{2-p}{2-2p}\right).
\end{align*}
where (a) holds since $\log(1+x) \ge x-\frac{x^2}{2}$ for $0<x<1$.
\end{proof}

\begin{lemma}
Fix $\epsilon_1 = \sqrt{\frac{\log TK}{2(\min_\ell \beta_\ell)^2(\log T)^2}}$. Let $\tilde{\mathcal{U}}_\ell = {\cal U}^*\cap \mathcal{U}_\ell$ and  $\mathcal{B}_2=\{|\tilde{\mathcal{U}}_\ell|\geq\beta_\ell(1-\epsilon_1) (\log T)^2 \textrm{ for } 1\leq\ell\leq L\}$. Then, $\Pr(\mathcal{B}_2)\ge 1-\frac{1}{T}$.
\label{lem:algCsize}
\end{lemma}
\begin{proof}
 Since $\Pr(u\in \tilde{\mathcal{U}}_\ell|u \in {\cal U}^*) = \Pr(u\in {\mathcal{U}}_\ell|u \in {\cal U}^*)= \Pr(u\in {\mathcal{U}}_\ell)=\beta_\ell$,
\begin{align*}
\Pr\left(|\tilde{\mathcal{U}}_\ell|\leq\beta_\ell(1-\epsilon_1) (\log T)^2\right)
& \stackrel{(a)}\leq \exp\left(-(\log T)^2 \kl((1-\epsilon_1)\beta_\ell, \beta_\ell)\right)\\
& \stackrel{(b)}\leq \exp\left(-(\log T)^2 2\epsilon_1^2\beta_\ell^2\right) 
\\
& \leq \frac{1}{TK}.
\end{align*}
where (a) is from Chernoff-Hoeffding bound (b) is from Pinsker’s inequality.

Hence, the event $\mathcal{B}_2=\{|\tilde{\mathcal{U}}_\ell|\geq\beta_\ell(1-\epsilon_1) (\log T)^2 \textrm{ for } 1\leq\ell\leq L\}$ holds with probability at least $1-\frac{1}{T}$.
\end{proof}

In the remaining of this section, we fix $\epsilon_1 = \sqrt{\frac{\log TK}{2(\min_\ell \beta_\ell)^2(\log T)^2}}$ as chosen in the previous lemma.

\begin{lemma}
In the 'Exploitation' step, under $\mathcal{B}_1$ and $\mathcal{B}_2$, if t is large enough to satisfy the conditions $\epsilon < \frac{1}{2}\min_{\ell \neq \ell'} \|p_{\ell} - p_{\ell'}\|$ and $(1-\epsilon_1)\left(1-2K\left(\frac{m}{t}\right)^2\right)\min_\ell \beta_\ell>\frac{m}{t}$, then $\|{p}_\ell - \hat{p}_\ell\| < \epsilon$ with probability at least $1-\frac{2}{T}$.
\label{lem:algoCcluster}
\end{lemma}

\begin{proof}
Recall that $\epsilon = K\sqrt{\frac{8Km}{t}\log\frac{t}{m}}$ and $Q_u = \{v\in {\cal U}^*: \|\hat{\rho}_u-\hat{\rho}_v\| \leq \epsilon \}$ for all $u \in {\cal U}^*$. We define a set $\mathcal{C}_\ell$ for $1\leq\ell\leq L$ as:
$\mathcal{C}_\ell  = \{u\in {\cal U}^*: \|p_\ell-\hat{\rho}_u\| \leq \frac{\epsilon}{2} \}$. This set has the following properties:
\begin{enumerate}[label=(\roman*), itemsep=-0.5ex]
\item $|\mathcal{C}_\ell|=\Omega((\log T)^2)$  with probability at least $1-\frac{1}{T}$. This follows from the following argument.
\begin{align*}
\Pr(u\in\mathcal{C}_\ell)&\geq \Pr(u\in\mathcal{C}_\ell|u\in \mathcal{U}_\ell)\Pr(u\in\mathcal{U}_\ell)\cr
&\stackrel{(a)}\geq \beta_\ell(1-\epsilon_1) \Pr\left(\|p_\ell-\hat{\rho}_u \| \leq \frac{\epsilon}{2} \Big|u\in\mathcal{U}_\ell\right)\cr
&\stackrel{(b)}{\geq} \beta_\ell(1-\epsilon_1)\left(1-2\exp\left(-2\frac{t}{2Km}\left(\frac{\epsilon}{2K}\right)^2\right)\right)^K\cr
&\ge \beta_\ell(1-\epsilon_1)\left(1-2K\exp\left(-\frac{t\epsilon^2}{4K^3m}\right)\right)\cr
&\ge \beta_\ell(1-\epsilon_1)\left(1-2K\left(\frac{m}{t}\right)^2\right),
\end{align*}%

where (a) follows from the assumption that $\mathcal{B}_2$ holds and (b) stems from $\mathcal{B}_1$ and Chernoff-Hoeffding's bound.

Let $r=\beta_\ell(1-\epsilon_1)\left(1-2K\left(\frac{m}{t}\right)^2\right)$. Then,
\begin{align*}
\Pr\left(|\mathcal{C}_\ell|<\left(r-\frac{1}{\sqrt{2\log T}}\right)(\log T)^2\right)
\leq& \exp\left(-(\log T)^2 \kl\left(r-\frac{1}{\sqrt{2\log T}},r\right)\right)\cr
\leq& \exp\left(-2(\log T)^2\left(\frac{1}{\sqrt{2\log T}}\right)^2\right)\cr
\leq& \frac{1}{T}.
\end{align*}
Therefore, $|\mathcal{C}_\ell|=\Omega((\log T)^2)$ with probability at least $1-\frac{1}{T}$.

\item $\big| {\cal U}^* \setminus (\cup_{\ell=1}^{L}\mathcal{C}_\ell)\big| = \set{O}(\frac{m(\log T)^2}{t})$ with probability at least $1-\frac{1}{T}$. To show this, we use a similar argument as in (i):
{\abovedisplayskip=5pt
\begin{align*}
\Pr(u \in {\cal U}^* \setminus (\cup_{\ell=1}^{L}\mathcal{C}_\ell))
&\leq \sum_{\ell=1}^{L}\Pr(u\in \mathcal{U}_\ell)\Pr(u \in {\cal U}^* \setminus (\cup_{\ell=1}^{L}\mathcal{C}_\ell)|u\in \mathcal{U}_\ell)\cr
&\leq \sum_{\ell=1}^{L}\Pr(u\in \mathcal{U}_\ell)\Pr\left(\|p_\ell-\hat{\rho}_u\| > \frac{\epsilon}{2} \Big|u\in\mathcal{U}_\ell\right)\cr
&\leq \sum_{\ell=1}^{L}\Pr(u\in \mathcal{U}_\ell)2K\left(\frac{m}{t}\right)^2\cr
& = 2K\left(\frac{m}{t}\right)^2.
\end{align*}}
Then, the probability that the size of $\big| {\cal U}^* \setminus (\cup_{\ell=1}^{L}\mathcal{C}_\ell)\big|$ is greater than $\frac{m(\log T)^2}{t}$ is,
{\abovedisplayskip=5pt
\begin{align*}
\Pr\left(\big| {\cal U}^* \setminus (\cup_{\ell=1}^{L}\mathcal{C}_\ell)\big| \geq \frac{m(\log T)^2}{t}\right)
&\leq \exp\left(-(\log T)^2 \kl \left(\frac{m}{t}, 2K\left(\frac{m}{t}\right)^2\right)\right) \cr
& \stackrel{(a)} \leq \exp\left(-(\log T)^2\frac{m}{t}\log\frac{t}{m}\right)\cr
&\stackrel{(b)}\leq \frac{1}{T},
\end{align*}}
where (a) is obtained from Lemma \ref{lem:kllower} and $t \ge 2Km$ and (b) is from $ t \le m\log T$.
\item If $|\mathcal{C}_\ell\cap Q_u| \geq 1$, then $|\mathcal{C}_m\cap Q_u| =0$ for all $\ell\neq m$. Because $\|\hat{\rho}_u-\hat{\rho}_v\|\geq \|p_\ell-p_m\|-\|p_\ell-\hat{\rho}_u\|-\|p_m-\hat{\rho}_j\| \geq \|p_\ell-p_m\| - \epsilon > \epsilon$ for $u\in\mathcal{C}_\ell$ and $j\in\mathcal{C}_m$, where the last inequality follows from $2\epsilon<\min_{\ell \neq \ell'} \|p_{\ell} - p_{\ell'}\|$.

\item $\mathcal{C}_\ell \subset Q_u$ for all $u \in \mathcal{C}_\ell$, since $\|\hat{\rho}_u-\hat{\rho}_v\|\leq \|\hat{\rho}_u-p_\ell\|+\|\hat{\rho}_v-p_\ell\|\leq \epsilon$ for all $v \in\mathcal{C}_\ell $.

\end{enumerate}

From the properties (iii) and (iv), there exists an item $u\in (\cup_{\ell=1}^{L}\mathcal{C}_\ell) \setminus (\cup_{r =1}^{\ell-1} Q_{i_r})$ such that $|Q_u \setminus (\cup_{r =1}^{\ell-1} Q_{i_r})| \geq m_\ell$.  Here, $m_\ell = \Omega ( (\log T)^2)$ from property (i).

We also have $|Q_v| = \set{O}(\frac{m(\log T)^2}{t})$ for $v$ such that $|Q_v \cap (\cup_{k=1}^{K}\mathcal{C}_k)) | = 0$ from property (ii). Since we assume $(1-\epsilon_1)\left(1-2K\left(\frac{m}{t}\right)^2\right)\min_\ell \beta_\ell>\frac{m}{t}$, the item $v$ cannot be chosen as $i_k$.

We can conclude that $\|{p}_\ell - \hat{p}_\ell\|\leq \epsilon$ with probability $1-2/T$, since $\|\hat{\rho}_u - p_\ell\| \le \epsilon$ when $|Q_u \cap \mathcal{C}_\ell| \ge 1$.
\end{proof}

\begin{lemma}
If $\|{p}_\ell - \hat{p}_\ell\| < \frac{1}{4}\min(y_{\ell r} ,\delta)$ for all $r\neq\ell$, the regret due to recommendations based on optimistic user assignments is,
\begin{align*}
 {\cal O}\left(m\sum_{\ell}\beta_\ell (p_{\sigma_\ell(1)\ell}- p_{\sigma_\ell(K)\ell}) \left(\sum_{r\in\mathcal{R}_\ell\setminus\mathcal{L}^{\perp}(\ell)} \frac{K^2\log K}{ \phi(|{p}_{k_\ell^* r}-{p}_{k_\ell^*\ell}|^2)} + \sum_{k\in \mathcal{S}_{\ell r}}\sum_{r\in\mathcal{L}^{\perp}(\ell)} \frac{K\log \overline{N}}{|\mathcal{S}_{\ell r}||{p}_{k \ell}-{p}_{kr}|^2}\right)\right).
\end{align*}
\label{lem:algCN}
\end{lemma}

\begin{proof}
Recall that $x_{k \ell} = \max\{|\hat{p}_{k \ell}-\hat{\rho}_{k u_t}| - \epsilon, 0\}$ and $\mathcal{L}^{\perp}(\ell)=\{\ell'\neq\ell: k_\ell^*\neq k_{\ell'}^*, p_{k_\ell^*\ell}=p_{k_{\ell}^*\ell'}\}$. We take $\epsilon < \frac{1}{4}\min(y_{\ell r} ,\delta)$ to satisfy the condition $\|{p}_\ell - \hat{p}_\ell\| < \frac{1}{4}\min(y_{\ell r} ,\delta)$. When we make a recommendation to $u_t \in \mathcal{U}_\ell$ by referring to their neighbors from ${\cal U}^*$, if regret is generated, the following event holds.

\begin{align*}
\mathcal{E}_{N} &= \bigg\{\sum_{k=1}^K n_{k u_t}x_{k \ell}^2  > 2K \log n_{u_t} \bigg\} \cr
&\qquad\qquad\qquad\cup\left(\bigcup_{r\neq\ell}\bigg\{ \sum_{k=1}^K n_{k u_t}x_{kr}^2  < 2K \log n_{u_t} \bigg\}\cap\left\{\arg \max_k p_{k \ell}  \neq \arg \max_k p_{kr}\right\}\right)\cr
&= \bigg\{\sum_{k=1}^K n_{k u_t}x_{k \ell}^2  > 2K \log n_{ u_t} \bigg\}
\cup\left(\bigcup_{r\in \mathcal{R}_\ell}\bigg\{ \sum_{k=1}^K n_{k u_t}x_{kr}^2  < 2K \log n_{u_t} \bigg\}\right)\cr
&:= \mathcal{E}_{N_1}\cup\left(\bigcup_{r\in \mathcal{R}_\ell}\mathcal{E}_{N_2}^{(r)}\right)
\end{align*}

First, an upper bound of the probability of the event $\mathcal{E}_{N_1}$ is

\begin{align*}
\Pr(\mathcal{E}_{N_1})
&\le \Pr\left(\bigcup_{k=1}^K \{n_{k u_t}x_{k \ell}^2  >  2\log n_{u_t}\}\right)\cr
&\le \sum_{k=1}^K\Pr \left(n_{k u_t}x_{k \ell}^2  > {2\log n_{u_t}}\right)\cr
&\le \sum_{k=1}^K\Pr \left(|\hat{p}_{k \ell}-\hat{\rho}_{k u_t}| - \epsilon  > \sqrt{\frac{2\log n_{u_t}}{n_{k u_t}}}\right)\cr
&\le \sum_{k=1}^K\Pr \left(|\hat{p}_{k \ell}-\hat{\rho}_{k u_t}| - |{p}_{k \ell}-\hat{p}_{k \ell}|  > \sqrt{\frac{2\log n_{u_t}}{n_{k u_t}}}\right)\cr
&\le \sum_{k=1}^K\Pr \left(|{p}_{k \ell}-\hat{\rho}_{k u_t}|  > \sqrt{\frac{2\log n_{u_t}}{n_{u_t}}}\right)
\end{align*}

By Lemma \ref{lem:couponcollector}, we know that $u_t$ arrives at most $\overline{N}$ times in expectation. Therefore, the regret induced by the event $\mathcal{E}_{N_1}$ is
\begin{align}
R_{\mathcal{E}_{N_1}}(\overline{N})
& \le \sum_{k=1}^K\sum_{s=1}^{\overline{N}} \Pr \left(|{p}_{k \ell}-\hat{\rho}_{k u_t}|  > \sqrt{\frac{2\log s}{s}} \,\middle\vert\, n_{u_t}=s\right)\cr
& \le \sum_{k=1}^K \sum_{s=1}^{\overline{N}} 2\exp\left(-4\log s\right)\cr
&\le 2K\left(1 + \int_{1}^{\overline{N}} \frac{1}{x^4} dx\right) \cr
&\le 2K\left(1 + \left[-\frac{1}{3x^3}\right]_{1}^{\overline{N}}\right)
= \frac{2K}{3}\left(4-\frac{1}{\overline{N}^3}\right)
\label{eqn:algoCevent1}
\end{align}

Next, we evaluate the probability of the event $\mathcal{E}_{N_2}^{(r)}$. First, we assume that  $r\not\in\mathcal{L}^{\perp}(\ell)$. Then,
\begin{align}
\Pr\left(\max_{\{1\le n_{k u_t} \le n_{u_t}\}}n_{k u_t}x_{k \ell}^2  >  2\log \left(\frac{n_{u_t}}{2}\right) \right)
&\le \sum_{s=1}^{n_{u_t}}\Pr \left(|\hat{p}_{k \ell}-\hat{\rho}_{k u_t}| - \epsilon  > \sqrt{\frac{2\log (\frac{n_{u_t}}{2})}{n_{k u_t}}} \,\middle\vert\,  n_{k u_t}=s\right)\cr
&\le \sum_{s=1}^{n_{u_t}}\Pr \left(|{p}_{k \ell}-\hat{\rho}_{k u_t}| > \sqrt{\frac{2\log (\frac{n_{u_t}}{2})}{n_{k u_t}}} \,\middle\vert\, n_{k u_t}=s\right)\cr
&\le \sum_{s=1}^{n_{u_t}} 2\exp\left(-4\log\left(\frac{n_{u_t}}{2}\right)\right)
=\frac{32}{n_{u_t}^3}
\label{eqn:probassign}
\end{align}

If the event $\{\forall k, \max_{\{1\le n_{k u_t} \le n_{u_t}\}}n_{k u_t}x_{k \ell}^2  <  2\log \left(\frac{n_{u_t}}{2}\right)\}$ occurs, then the event $\{n_{k_\ell^* u_t}\ge\frac{n_{u_t}}{2K}\}$ occurs as well. Hence, we can deduce that $\Pr\left(n_{k_\ell^* u_t}<\frac{n_{u_t}}{2K}\right)\le \frac{32K}{n_{u_t}^3}$ by (\ref{eqn:probassign}). The probability of the event $\mathcal{E}_{N_2}^{(r)}$ satisfies:

\begin{align*}
\Pr(\mathcal{E}_{N_2}^{(r)})
&\le \Pr\left(n_{k_\ell^* u_t} x_{k_\ell^* r}^2  < 2K \log n_{u_t} \right)\cr
&\le \Pr\left( x_{k_\ell^* r}  < \sqrt{\frac{2K\log n_{ u_t}}{n_{k_\ell^* u_t}}}\right)\cr
&\le \Pr\left( |\hat{p}_{k_\ell^* r}-\hat{\rho}_{k_\ell^* u_t}|  < \sqrt{\frac{2K\log n_{u_t}}{n_{k_\ell^* u_t}}} + \epsilon\right)\cr
&\le \Pr\left( |{p}_{k_\ell^* r}-{p}_{k_\ell^*\ell}|-|{p}_{k_\ell^*\ell}-\hat{\rho}_{k_\ell^* u_t}| - |{p}_{k_\ell^* r}-\hat{p}_{k_\ell^* r}|  <\sqrt{\frac{2K\log n_{u_t}}{n_{k_\ell^*u_t}}} + \epsilon\right)\cr
&\le \Pr\left(|{p}_{k_\ell^*\ell}-\hat{\rho}_{k_\ell^* u_t}|  > |{p}_{k_\ell^* r}-{p}_{k_\ell^*\ell}| - \sqrt{\frac{2K\log n_{u_t}}{n_{k_\ell^* u_t}}} - 2\epsilon\right)\cr
&\le \Pr\left(|{p}_{k_\ell^*\ell}-\hat{\rho}_{k_\ell^* u_t}|  > \frac{1}{2} |{p}_{k_\ell^* r}-{p}_{k_\ell^*\ell}| - \sqrt{\frac{2K\log n_{u_t}}{n_{k_\ell^*u_t}}}\right)\cr
&\le \Pr\left(|{p}_{k_\ell^*\ell}-\hat{\rho}_{k_\ell^* u_t}|  > \frac{1}{2} |{p}_{k_\ell^* r}-{p}_{k_\ell^*\ell}| - \sqrt{\frac{4K^2\log n_{u_t}}{n_{u_t}}}\right) + \frac{32K}{n_{u_t}^3}.
\end{align*}

Therefore, the regret induced by event $\mathcal{E}_{N_2}^{(r)}$ is

\begin{align}
&R_{\{\mathcal{E}_{N_2}^{(r)}\setminus \mathcal{E}_{N_1}\}}  (\overline{N})\cr
&\le \sum_{s=1}^{\overline{N}}\Pr\left(|{p}_{k_\ell^*\ell}-\hat{\rho}_{k_\ell^* u_t}|  > \frac{1}{2} |{p}_{k_\ell^* r}-{p}_{k_\ell^*\ell}| - \sqrt{\frac{4K^2\log s}{s}}\right) + \frac{32K}{s^3}\cr
&\le \frac{64K^2}{|{p}_{k_\ell^* r}-{p}_{k_\ell^* \ell}|^2}\log\left(\frac{6K}{|{p}_{k_\ell^* r}-{p}_{k_\ell^* \ell}|}\right) \cr
&\qquad\qquad\qquad+ \sum_{s=\lceil\frac{64K^2}{|{p}_{k_\ell^* r}-{p}_{k_\ell^* \ell}|^2}\log\left(\frac{6K}{|{p}_{k_\ell^* r}-{p}_{k_\ell^* \ell}|}\right)\rceil}^{\overline{N}} 2\exp\left(-\frac{s|{p}_{k_\ell^* \ell}-{p}_{k_\ell^* r}|^2}{8}\right) + 48K\cr
&\le \frac{64K^2}{|{p}_{k_\ell^* r}-{p}_{k_\ell^* \ell}|^2}\log\left(\frac{6K}{|{p}_{k_\ell^* r}-{p}_{k_\ell^* \ell}|}\right) + \frac{16}{|{p}_{k_\ell^* \ell}-{p}_{k_\ell^* r}|^2} + 48K\cr
&\le \frac{128K^2}{|{p}_{k_\ell^* r}-{p}_{k_\ell^* \ell}|^2}\log\left(\frac{6K}{|{p}_{k_\ell^* r}-{p}_{k_\ell^* \ell}|}\right).
\label{eqn:algoCevent2}
\end{align}
\newpage

Next assume that $r\in\mathcal{L}^{\perp}(\ell)$. Then, $k_\ell^*\neq k_{r}^*$ and $p_{k_\ell^* \ell}=p_{k_{\ell}^* r}$, and
\begin{align*}
\Pr(\mathcal{E}_{N_2}^{(r)})
&\le \Pr\left(\sum_{k \in S_{\ell r}} n_{k u_t}x_{kr}^2  <2K \log n_{u_t} \right)\cr
&\le \Pr\left(\bigcup_{k \in S_{\ell r}}\{n_{k u_t}x_{kr}^2  < \frac{2K \log n_{u_t}}{|S_{\ell r}|}\}\right)\cr
&\le \sum_{k \in S_{\ell r}} \Pr\left( x_{kr}  < \sqrt{\frac{2K \log n_{u_t}}{|S_{\ell r}| n_{k u_t}}}\right)\cr
&\le \sum_{k \in S_{\ell r}} \Pr\left(|{p}_{k \ell}-\hat{\rho}_{k u_t}|  > \frac{1}{2} |{p}_{kr}-{p}_{k \ell}| - \sqrt{\frac{2K \log n_{u_t}}{|S_{\ell r}|n_{k u_t}}}\right)\cr
&\le \sum_{k \in S_{\ell r}} \Pr\left(|{p}_{k \ell}-\hat{\rho}_{k u_t}|  > \frac{1}{2} |{p}_{kr}-{p}_{k \ell}| - \sqrt{\frac{2K \log \overline{N}}{|S_{\ell r}|n_{k u_t}}}\right).
\end{align*}

\begin{align}
R_{\mathcal{E}_{N_2}^{(r)}}(\overline{N})
&\le \sum_{k\in\mathcal{S}_{\ell r}} \sum_{s=1}^{\overline{N}} \Pr\left(|{p}_{k \ell}-\hat{\rho}_{k u_t}|  > \frac{1}{2} |{p}_{kr}-{p}_{k \ell}| - \sqrt{\frac{2K\log \overline{N}}{|\mathcal{S}_{\ell r}|s}}|n_{k u_t}=s\right)\cr
&\stackrel{(a)}\le \sum_{k\in\mathcal{S}_{\ell r}} \left( \frac{32K \log\overline{N}}{|\mathcal{S}_{\ell r}||{p}_{k \ell}-{p}_{k r}|^2} + \sum_{s=\lceil \frac{32K \log\overline{N}}{|\mathcal{S}_{\ell r}||{p}_{k \ell}-{p}_{k r}|^2}\rceil}^{\overline{N}} 2\exp\left(-\frac{s|{p}_{k \ell}-{p}_{k r}|^2}{8}\right)\right)\cr
& \le\sum_{k\in\mathcal{S}_{\ell r}}\frac{32}{|{p}_{k \ell}-{p}_{k r}|^2}\left(\frac{K\log \overline{N}}{|\mathcal{S}_{\ell r}|} + 1\right).
\label{eqn:algoCevent3}
\end{align}

Combining (\ref{eqn:algoCevent1}), (\ref{eqn:algoCevent2}) and (\ref{eqn:algoCevent3}), the expected regret due to recommendations made by referring to the nearest neighbors in ${\cal U}^*$ is,

\begin{align*}
& R_N(T) \le m\sum_{\ell}\beta_\ell (p_{\sigma_\ell(1)\ell}- p_{\sigma_\ell(K)\ell}) \left(\frac{2K}{3}\left(4-\frac{1}{\overline{N}^3}\right)\right.\cr
& + \left. \sum_{r\in\mathcal{R}_\ell\setminus\mathcal{L}^{\perp}(\ell)}\frac{128K^2}{|{p}_{k_\ell^* r}-{p}_{k_\ell^* \ell}|^2}\log\left(\frac{6K}{|{p}_{k_\ell^* r}-{p}_{k_\ell^* \ell}|}\right) + \sum_{r\in\mathcal{L}^{\perp}(\ell)} \sum_{k\in\mathcal{S}_{\ell r}}\frac{32}{|{p}_{k \ell}-{p}_{k r}|^2}\left(\frac{K\log \overline{N}}{|\mathcal{S}_{\ell r}|} + 1\right)\right)\cr
& = \set{O}\left(m\sum_{\ell}\beta_\ell (p_{\sigma_\ell(1)\ell}- p_{\sigma_\ell(K)\ell})     \left(\sum_{r\in\mathcal{R}_\ell\setminus\mathcal{L}^{\perp}(\ell)} \frac{K^2\log K}{ \phi(|{p}_{k_\ell^* r}-{p}_{k_\ell^* \ell}|^2)}
+ \sum_{k\in \mathcal{S}_{\ell r}}\sum_{r\in\mathcal{L}^{\perp}(\ell)} \frac{K\log \overline{N}}{|\mathcal{S}_{\ell r}||{p}_{k \ell}-{p}_{kr}|^2}\right)\right).
\end{align*}
\end{proof}

\newpage
\section{Performance guarantees of ECB and Item Clustering: Proof of Theorems \ref{th:Cgen} and ~\ref{thm:spec}}\label{app:upperECB}

\addtocontents{toc}{\protect\setcounter{tocdepth}{1}}
\subsection{Performance guarantees of ECB: Proof of Theorem \ref{th:Cgen}}\label{subsec:ECBproof}
\addtocontents{toc}{\protect\setcounter{tocdepth}{2}}

The proof is straightforward from the results of Theorem \ref{thm:spec}. Indeed, the latter implies that we can assume that the item clusters estimated from the first phase of the algorithm are exact (the complement of this event happens with probability $1/T$, and hence generates an expected regret ${\cal O}(1)$).
 
Hence we can assume that we know the exact clusters of items in $\mathcal{S}$. Now observe that for each cluster $V_k$ (a subset of ${\cal S}$), we have $|V_k| \ge \overline{N}$ (refer to Appendix \ref{app:upperC} for a precise statement). As a consequence, all users can be served using items from $V_1,\dots, V_K$, except for a few users arriving more than $\overline{N}$. Actually, from Lemma~\ref{lem:couponcollector}, these exceptional arrivals induce an average regret ${\cal O}(1)$. ECB applies, for each user, a UCB1 algorithm ~\cite{auer2002finite} to select the cluster from which an item is recommended. For this exploitation period, each user in $\mathcal{U}_\ell$ will then induce a regret  ${\cal O}(\sum_{k\neq k^\star_\ell} \frac{  \log ( \overline{N} ) }{p_{k^\star_\ell \ell}-p_{k\ell}})$. This completes the proof. \ep

\addtocontents{toc}{\protect\setcounter{tocdepth}{1}}
\subsection{Item Clustering Phase: Proof of Theorem~\ref{thm:spec}} \label{app:itemcluster}
\addtocontents{toc}{\protect\setcounter{tocdepth}{2}}

We let $e(v,S)\coloneqq\sum_{x\in S} A_{vx}$ and $e(A,B)\coloneqq\sum_{v\in A} \sum_{w\in B} A_{vw}$. 

Let
$$p(a,b) \coloneqq  (|V_b| - \indicator\{a=b \})\frac{2}{n_0(n_0-1)} \sum_{\ell = 1}^L \beta_\ell p_{a\ell} p_{b \ell}$$
and $p(a,0) \coloneqq  1-\sum_{k=1}^K p(a,k)$.

We also define $e(v, V_0)$ and $e(v,S^{(t)}_0)$ as follows:
\begin{align*}
e(v, V_0)  \coloneqq  & s- \sum_{k=1}^K e(v,V_k) \quad \mbox{and} \cr
e(v,S^{(t)}_0) \coloneqq & s- \sum_{k=1}^K e(v,S^{(t)}_k),
\end{align*}
where $s$ is the number of users who have received recommendations at least twice until $t=10m$.

{\bf Proof of the theorem.} The proof of Theorem~\ref{thm:spec} relies on the following random matrix concentration inequality. Specifically, from the matrix Bernstein inequality, we can bound $\|A - \mathbb{E}[A] \|$ as follows.
\begin{lemma}
	Assume $T>10m$. Let A be the adjacency matrix obtained in Algorithm \ref{alg:SIC_B}. Let $\|\cdot\|$ denote the spectral norm. Then,
\begin{align*}
	\Pr\left(\| A-\mathbb{E}[A]\|>5\sqrt{\frac{m\log m}{n_0}}\right) \leq \frac{1}{m^2}.
\end{align*}
	\label{lem:Bernstein}
\end{lemma}

From Lemma~\ref{lem:Bernstein}, we deduce that $\hat{A}$ (the rank-$K$ approximation of $A$) is approximately the same as $\mathbb{E}[A]$. Indeed, since both $\hat{A}$ and $\mathbb{E}[A]$ are of rank $K$, 
\begin{align*}
\|\hat{A} - \mathbb{E}[A] \|^2_F \le& 2K \|\hat{A} - \mathbb{E}[A] \|^2 \cr
\le & 4K (\|\hat{A} - A \|^2 +\|A - \mathbb{E}[A] \|^2) \cr
\le & 8K \|A - \mathbb{E}[A] \|^2,
\end{align*} 
where $\|A - \mathbb{E}[A] \|$ is negligible compared to $\|\mathbb{E}[A] \| = \Omega(\frac{m}{n_0})$.

From the columns of $\hat{A}$, we can classify items. Here, $\|\mathbb{E}[A]_v - \mathbb{E}[A]_w \| = \Omega( \frac{m}{n_0^{3/2}}   )$ when $v$ and $w$ belong to different clusters and $\mathbb{E}[A]_v = \mathbb{E}[A]_w$ when $v$ and $w$ belong to the same cluster. Therefore, the columns of $\hat{A}$ are concentrated around the correct cluster column unless $\|\hat{A}_v - \mathbb{E}[A]_v \|= \Omega( \frac{m}{n_0^{3/2}}   )$. From this argument, 
 the spectral decomposition used in the algorithms satisfies
\begin{equation}\label{eq:err_of_SC}
\big|\cup^K_{k=1}( S_k^{(0)} \setminus V_k )\big| = O\left(\frac{n_0^2\log m}{m}\right),
\end{equation}
since  $\sum_{v\in \mathcal{S}} \|\hat{A}_v - \mathbb{E}[A]_v \|^2= \| A-\mathbb{E}[A]\|_F^2 =O( \frac{m\log m}{n_0})$ from Lemma~\ref{lem:Bernstein} (cf. \cite{ok2017collaborative}).

In the improvement step of Algorithm \ref{alg:SP_plus_B}, the algorithm refines the result of Spectral Decomposition iteratively. We denote the set of misclassified items after $t$-th iteration by $\mathcal{E}^{(t)}$. We also introduce $\mathcal{E}_{k\ell}^{(t)}=S_k^{(t)} \cap V_\ell$ so that $\mathcal{E}^{(t)} = \bigcup_{k=1}^K \bigcup_{\ell: \neq k} \mathcal{E}_{k\ell}^{(t)} $.

Since the items move to more likely cluster with respect to $\hat{p}(i,j)$ at each step,

\begin{align}
&0 \le \sum_{k,\ell : k\neq \ell} \sum_{v\in\mathcal{E}_{k\ell}^{(t+1)}} \sum_{j=0}^{K} e\left(v,S_j^{(t)}\right) \log \frac{\hat{p}(k,j)}{\hat{p}(\ell,j)} \cr
&\stackrel{(a)} \le \sum_{k,\ell : k\neq \ell} \sum_{v\in\mathcal{E}_{k\ell}^{(t+1)}} \sum_{j=0}^{K} e\left(v,S_j^{(t)}\right) \log \frac{p(k,j)}{p(\ell,j)}\cr
&\quad+C_1 |\mathcal{E}^{(t+1)}|  \sqrt{\frac{m\log m}{n_0}}\cr
&\stackrel{(b)} \le \sum_{k,\ell : k\neq \ell} \sum_{v\in\mathcal{E}_{k\ell}^{(t+1)}} \sum_{j=0}^{K} e\left(v,V_j\right) \log \frac{p(k,j)}{p(\ell,j)}\cr
&\quad+C_1 |\mathcal{E}^{(t+1)}|  \sqrt{\frac{m\log m}{n_0}}+ C_2 e(\mathcal{E}^{(t)},\mathcal{E}^{(t+1)})\cr
&\stackrel{(c)} \le \sum_{k,\ell : k\neq \ell} \sum_{v\in\mathcal{E}_{k\ell}^{(t+1)}} \sum_{j=0}^{K} e\left(v,V_j\right) \log \frac{p(k,j)}{p(\ell,j)}\cr
&\quad+C_1 |\mathcal{E}^{(t+1)}|  \sqrt{\frac{m\log m}{n_0}}+ C_3 \sqrt{|\mathcal{E}^{(t)}||\mathcal{E}^{(t+1)}|\frac{m\log m}{n_0}}\cr
&\stackrel{(d)}\le -C_4\frac{m}{n_0}|\mathcal{E}^{(t+1)}| + C_3 \sqrt{|\mathcal{E}^{(t)}||\mathcal{E}^{(t+1)}|\frac{m\log m}{n_0}}\label{eqn:likelihoodtest}
\end{align}
where (a) is obtained from Lemma \ref{lem:upperA}; (b) stems from the fact that $\frac{p(k,j)}{p(\ell,j)}$ is a positive constant for all $1\le j \le K$;  (c) follows from Lemma~\ref{lem:upperC}; and (d) is obtained from Lemma~\ref{lem:greedyKL}.

From (\ref{eqn:likelihoodtest}), we can conclude that:
\begin{align*}
\frac{|\mathcal{E}^{(t+1)}|}{|\mathcal{E}^{(t)}|} & \le \frac{C_3^2}{C_4^2}\frac{ n_0 \log m}{  m}
\\
& \stackrel{(a)}{\le} C_5 \frac{1}{\log T},
\end{align*}
where $(a)$ is from $n_0 \le m/(\log T)^2$ and $ 10 m <  T$.

Therefore, after $\log (n_0)$ iterations, we have recovered the perfect clusters. \ep

Nest we state the lemmas used in the proof above.

\begin{lemma} When $|\mathcal{E}^{(0)}|=O\left(\frac{n_0^2\log m}{m}\right)$, with probability $1-\frac{2}{m^2}$,
\begin{align*}
	&\sum_{i=0}^{K} e\left(v,S_i^{(t)}\right) \bigg|\log \frac{p(k,i)}{\hat{p}(k,i)} \bigg| =O\left( \sqrt{\frac{m\log m}{n_0}}\right).
\end{align*}
	\label{lem:upperA}
\end{lemma}

\begin{lemma} When $\| A-\mathbb{E}[A]\| = O\left( \sqrt{\frac{m\log m}{n_0}} \right)$, $|\mathcal{E}^{(t)}|=O\left(\frac{n_0^2\log m}{m}\right)$, and $|\mathcal{E}^{(t+1)}|=O\left(\frac{n_0^2\log m}{m}\right)$,
\begin{align*}
&\sum_{v\in \mathcal{E}^{(t+1)}}e(v,\mathcal{E}^{(t)})= O\left(\sqrt{|\mathcal{E}^{(t)}||\mathcal{E}^{(t+1)}|\frac{m\log m}{n_0}} \right).
\end{align*}
\label{lem:upperC}
\end{lemma}

\begin{lemma} With probability at least $1-\frac{1}{m^2}$, for all $k$, for all $v\in V_k$, 
	$$ \sum_{a=0}^K  e(v, V_a) \log \left(\frac{p(k,a)}{p(k',a)} \right) = \Omega\left( \frac{m}{n_0}\right) \quad \mbox{for all}\quad k'\neq k.$$ \label{lem:greedyKL}
\end{lemma}

\addtocontents{toc}{\protect\setcounter{tocdepth}{1}}
\subsection{Proof of the lemmas}
\addtocontents{toc}{\protect\setcounter{tocdepth}{2}}

{\bf Proof of Lemma~\ref{lem:Bernstein}.}

The adjacency matrix $A$ can be considered as the sum of $s$ samples of connected pairs for some $s \ge \frac{m}{2}$. We denote such samples as $X_\ell$ for $1\le \ell \le s$. Then, $A=\sum_{\ell=1}^{s} X_\ell$.

By the matrix Bernstein inequality (cf. \cite{tropp2015introduction}), we have:
\begin{align}
\Pr\left(\| A-\mathbb{E}[A]\|>t\right) &=\Pr\left(\|\sum_{l=1}^{s} (X_\ell-\mathbb{E}[X_\ell])\|>t\right)\cr
&\leq n_0 \exp\left(-\frac{t^2}{2\left(\sigma^2(A)+Dt/3\right)}\right),
\label{eqn:bernstein}
\end{align}
where $\sigma^2(A)=\|\mathbb{E}[(A-\mathbb{E}[A])^2]\|$ and $D$ is an upper bound of $\| X_\ell-\mathbb{E}[X_\ell]\|$ for all $\ell$.

The expectation of the elements $x_{ij}$ of matrix $X_\ell$ is
$$\mathbb{E}\big[x_{ij}| i\in V_k, j\in V_{k'}\big]=\frac{2}{n_0(n_0-1)} \sum_{\ell = 1}^L \beta_\ell p_{k\ell} p_{k' \ell}.$$

Since every elements in $\mathbb{E}[X_\ell]$ is less than $\frac{2}{n_0(n_0-1)}$, we have:
\begin{align}
\|X_\ell-\mathbb{E}[X_\ell]\|&\stackrel{(a)}\le\sqrt{\frac{4n_0^2}{n_0^2 (n_0-1)^2}+2}\cr
&\leq \frac{2}{n_0-1}+\sqrt{2},
\label{eqn:bernstein4}
\end{align}
where (a) is from the fact that Frobenius norm of the matrix is greater than its spectral norm.

From (\ref{eqn:bernstein4}), we deduce that we can choose $D=\frac{3}{2}$. \\
Moreover, the variance of the matrix A is
\begin{align}
\sigma^2(A)&=\|\mathbb{E}[(A-\mathbb{E}[A])^2]\|\cr
&\stackrel{(a)}=\|\sum_{\ell=1}^{s} \mathbb{E}[(X_\ell-\mathbb{E}[X_\ell])^2]\|\cr
&\le\sum_{\ell=1}^{s} \| \mathbb{E}[(X_\ell-\mathbb{E}[X_\ell])^2]\|\cr
&\le \sum_{\ell=1}^{s} \left(\| \mathbb{E}[X_\ell^2]\|+\| \mathbb{E}[X_\ell]^2\|\right),
\label{eqn:bernstein2}
\end{align}
where (a) is obtained from the independence of the $X_\ell$'s.

To get an upper bound of (\ref{eqn:bernstein2}), observe that the expectation of the $(i,j)$-th element of matrix $X_\ell^2$ is
\begin{align*}
\mathbb{E}[(X_\ell^2)_{ij}]&=\mathbb{E}\bigg[\sum_{k=1}^{n_0} x_{ik}x_{kj}\bigg]
=\sum_{k=1}^{n_0} \mathbb{E}[x_{ik}x_{kj}]=0.
\end{align*}
In addition, the expectation of the $(i,i)$-th elements of matrix $X_\ell^2$ is
\begin{align*}
\mathbb{E}[(X_\ell^2)_{ii}]&=\sum_{k=1}^{n_0} \mathbb{E}[x_{ik}^2]
=\sum_{k=1}^{n_0} \mathbb{E}[x_{ik}]
\le \frac{2}{n_0}.
\end{align*}

On the other hand, the elements of $\mathbb{E}[X_\ell]^2$ are ${\cal O}\left(\frac{1}{n_0^3}\right)$, which implies $\|\mathbb{E}[X_\ell]^2\| ={\cal O}\left(\frac{1}{n_0^2}\right)$.
Hence, using (\ref{eqn:bernstein2}), we deduce that $\sigma^2(A)\leq \frac{2s}{n_0} + {\cal O}\left(\frac{s}{n_0^2}\right)\leq \frac{2m}{n_0}$.

Now, an upper bound of (\ref{eqn:bernstein}) is
\begin{align}
n_0 \exp\left(-\frac{t^2}{\frac{4m}{n_0}+t}\right).
\label{eqn:bernstein3}
\end{align}

To conclude the proof of this lemma, we need to consider two cases: $n_0=\frac{m}{(\log T)^2}$ and $n_0=n$.

(i) When $n_0=\frac{m}{(\log T)^2}$, $t=5\sqrt{\log m}\log T$. So, (\ref{eqn:bernstein3}) becomes:
\begin{align*}
\frac{m}{(\log T)^2} \exp\left(-\frac{25(\log T)^2\log m}{4(\log T)^2 + 5\sqrt{\log m}\log T}\right)
&\stackrel{(a)}\leq \exp\left(-3\log m + \log\left(\frac{m}{(\log T)^2}\right)\right)\cr
&\leq \frac{1}{m^2}
\end{align*}
where (a) is obtained from the assumption $T>10m$.

(ii) If $n_0=n$, $t=5\sqrt{\frac{m\log m}{n}}$. Then, (\ref{eqn:bernstein3}) becomes:
\begin{align*}
n\exp\left(-\frac{\frac{25m\log m}{n}}{\frac{4m}{n}+5\sqrt{\frac{m\log m}{n}}}\right)&\leq \exp \left(-3\log m+\log n\right)\cr
&\stackrel{(a)}\leq \frac{1}{m^2}
\end{align*}
where (a) holds since $n\leq\frac{m}{(\log T)^2}$.\ep

\vspace{0.5cm}

{\bf Proof of Lemma~\ref{lem:upperA}.}

Recall that
$$p(a,b) =  (|V_b| - \indicator\{a=b \})\frac{2}{n_0(n_0-1)} \sum_{\ell = 1}^L \beta_\ell p_{a\ell} p_{b \ell}$$
and $p(a,0) = 1-\sum_{k=1}^K p(a,k)$. The estimations are
 $\hat{p}(i, j ) = \frac{\sum_{v \in \mathcal{S}_i} \sum_{v' \in \mathcal{S}_j} A_{v, v'}}{s |S_i^{(0)}| }$ for all $1\leq i, j \leq K$ and  $\hat{p}(i, 0)= 1 - \sum_{k=1}^{K} \hat{p}(i,k)$.
	
An upper bound of $|p(i,j)-\hat{p}(i,j)|$ for $1\le i,j \le K$ is
	\begin{align}
	|\hat{p}(i,j)-p(i,j)| &\le\frac{1}{s |S_i^{(0)}|}\bigg|e(S_i^{(0)},S_j^{(0)})-\mathbb{E}[e(S_i^{(0)},S_j^{(0)})]\bigg| \cr &\quad+\frac{1}{s |S_i^{(0)}|}\bigg|\mathbb{E}[e(S_i^{(0)},S_j^{(0)})]-s |S_i^{(0)}| p(i,j)\bigg|.
	\label{eqn:upperprob1}
	\end{align}
	
	Let $\mathcal{A}$ be the set of partitions $\{S_k\}_{1\le k \le K}$ of the set ${\cal S}$. Recall that ${\cal S}$ is of cardinality $n_0$. Then,
	\begin{align}
	|\mathcal{A}| &\leq K^{n_0}.
	\label{eqn:numpartition}
	\end{align}
	
	Now, we fix one partition $\{S_k\}\in \mathcal{A}$. Then, by Chernoff-Hoeffding bound,
	\begin{align}
	\Pr\left(\big|e(S_i,S_j)-\mathbb{E}[e(S_i,S_j)]\big|<\sqrt{mn_0\log m} \textrm{ for all } i,j\right)
	&\geq 1 - \exp\left(-\Theta\left(n_0\log m\right)\right).
	\label{eqn:diffnumedge}
	\end{align}
	
Combining (\ref{eqn:numpartition}) and (\ref{eqn:diffnumedge}), we deduce that the following event holds:
	\begin{align*}
	\big|e(S_i,S_j)-\mathbb{E}[e(S_i,S_j)]\big|<\sqrt{mn_0\log m}
	\end{align*}
	for all $i,j$ and $\{S_k\}\in \mathcal{A}$ with probability $1 - \exp\left(-\Theta\left(n_0\log m\right)\right)$ (just applying a union bound).
	
	Since $\{S_k^{(0)}\}\in \mathcal{A}$, with probability $1 - \exp\left(-\Theta\left(n_0\log m\right)\right)$,
	\begin{align}
	\big|e(S_i^{(0)},S_j^{(0)})-\mathbb{E}[e(S_i^{(0)},S_j^{(0)})]\big|<\sqrt{mn_0\log m}
	\label{eqn:uppernumedges1}
	\end{align}
	for all $i,j$.
	
	On the other hand, since $|\mathcal{E}^{(0)}|= \set{O}\left(\frac{n_0^2\log m}{m}\right)$ from the assumption, with probability $1 - \exp\left(-\Theta\left(n_0\log m\right)\right)$,
	\begin{align}
	\frac{1}{s |S_i^{(0)}|}\bigg|\mathbb{E}[e(S_i^{(0)},S_j^{(0)})]-s |S_i^{(0)}| p(i,j)\bigg| & = \set{O}\left(\frac{n_0 \log m}{m} p(i,j) \right),
	\label{eqn:uppernumedges2}
	\end{align}
		for all $i,j$.

	Then, conditioned on \eqref{eqn:uppernumedges2} for all $1\le i,j \le K$, we can derive an upper bound of $|p(i,j)-\hat{p}(i,j)|$ for $1\le i,j \le K$, using (\ref{eqn:upperprob1}), (\ref{eqn:uppernumedges1}) and (\ref{eqn:uppernumedges2}):
	\begin{align*}
	|p(i,j)-\hat{p}(i,j)| = \set{O}\left(\sqrt{\frac{n_0\log m}{m}} p(i,j) \right),
	\end{align*}
	which implies that for all $1\le i,j \le K$
	
	\begin{align}
	\bigg|\log \frac{\hat{p}(i,j)}{p(i,j)}\bigg| &\le \frac{|p(i,j) - \hat{p}(i,j)|}{p(i,j)}\cr
	&=\set{O}\left( \sqrt{\frac{n_0 \log m}{m}} \right).
	\label{eqn:upperlogprob1}
	\end{align}
	
	Furthermore, an upper bound of $\bigg|\log \frac{\hat{p}(i,0)}{p(i,0)}\bigg|$ is
	\begin{align}
	\bigg|\log \frac{\hat{p}(i,0)}{p(i,0)}\bigg| & = \set{O}\left(\frac{1}{n_0}\sqrt{\frac{n_0 \log m}{m}} \right),
	\label{eqn:upperlogprob2}
	\end{align}
	with probability at least $1 -  \exp\left(-\Theta\left(n_0\log m\right)\right)$.
	
	We also have $\EXP[e(v, \mathcal{S})]\le \frac{m}{n_0}$ and from Chernoff inequality,
	\begin{align}
	\Pr\left(|e(v,\mathcal{S})- \mathbb{E}[e(v,\mathcal{S})]| > \sqrt{4m\log m}\right)
	\le \frac{1}{m^2}.
	\label{eqn:upperlogprob3}
	\end{align}
	
	Finally, we obtain the following from (\ref{eqn:upperlogprob1}),  (\ref{eqn:upperlogprob2}) and (\ref{eqn:upperlogprob3}):
	
	\begin{align*}
	&\sum_{i=0}^{K} e\left(v,S_i^{(t)}\right) \bigg|\log \frac{p(k,i)}{\hat{p}(k,i)} \bigg| =\set{O}\left( \sqrt{\frac{m\log m}{n_0}}\right),
	\end{align*}
	with probability at least $1 - \frac{2}{m^2}$.
	This concludes the proof.
	\ep

\vspace{0.5cm}

{\bf Proof of Lemma~\ref{lem:upperC}.}
We have:
\begin{align*}
\sum_{v\in \mathcal{E}^{(t+1)}}\left(e(v,\mathcal{E}^{(t)})-\mathbb{E}[e(v,\mathcal{E}^{(t)})]\right)
&\le \indicator^T_{\mathcal{E}^{(t)}}(A-\EXP[A])\indicator_{\mathcal{E}^{(t+1)}},
\end{align*}
where $\indicator_S$ is the vector whose $i$-th component is equal to 1 if $i\in S$ and to 0 otherwise. Since $\mathbb{E}[e(v,\mathcal{E}^{(t)})]\le \frac{2m}{n_0^2}|\mathcal{E}^{(t)}|\le 2\log m$, 

\begin{align*}
\sum_{v\in \mathcal{E}^{(t+1)}}e(v,\mathcal{E}^{(t)}) 
&\le \sum_{v\in \mathcal{E}^{(t+1)}}\left(e(v,\mathcal{E}^{(t)})-\mathbb{E}[e(v,\mathcal{E}^{(t)})]\right) + 2|\mathcal{E}^{(t+1)}|\log m \cr
& \le\|\indicator^T_{\mathcal{E}^{(t)}}(A-\EXP[A])\indicator_{\mathcal{E}^{(t+1)}}\|+ 2|\mathcal{E}^{(t+1)}|\log m\cr
&\le \|\indicator^T_{\mathcal{E}^{(t)}}\|\|(A-\EXP[A])\|\|
\indicator_{\mathcal{E}^{(t+1)}}\|+ 2|\mathcal{E}^{(t+1)}|\log m\cr
& \stackrel{(a)}{=} \set{O}\left(\sqrt{|\mathcal{E}^{(t)}||\mathcal{E}^{(t+1)}|\frac{m\log m}{n_0}} \right),
\end{align*}
where for $(a)$, we used the assumption that $ \| A - \EXP[A]\| = \set{O}\left(\sqrt{\frac{m \log m}{n_0}}\right)$ and the definition of $\ell_2$ norm.
\ep

\vspace{0.5cm}

{\bf Proof of Lemma~\ref{lem:greedyKL}.}

From Chernoff-Hoeffding bound, there exists $C>0$ such that
\begin{align}
\mathbb{P} \left( \left| e(v,V_a)  - \mathbb{E}[e(v,V_a)] \right| > C \sqrt{\frac{m\log m}{n_0}}  \right) \le \frac{1}{m^3}. \label{eq:lemgreedyKL1}
\end{align}

We also have for all $v\in V_k$ and all $k$,
\begin{equation}
\sum_{a=0}^K  \mathbb{E}[e(v, V_a)] \log \left(\frac{p(k,a)}{p(k',a)} \right) = \Omega\left( \frac{m}{n_0}\right) \quad \mbox{for all}\quad k'\neq k. \label{eq:lemgreedyKL2}
\end{equation}
Since $\log \left(\frac{p(k,a)}{p(k',a)} \right)  = \Theta(1)$ and $\frac{m}{n_0} = \Omega ((\log T)^2)$, from \eqref{eq:lemgreedyKL1} and \eqref{eq:lemgreedyKL2}, we have 
Lemma~\ref{lem:greedyKL}. \ep

\end{document}